%% file: main.tex
\documentclass{article}

\PassOptionsToPackage{numbers,compress}{natbib}
\usepackage[preprint]{neurips_2026}

\usepackage[utf8]{inputenc}
\usepackage[T1]{fontenc}
\usepackage{hyperref}
\usepackage{url}
\usepackage{booktabs}
\usepackage{amsfonts}
\usepackage{amsmath,amssymb,amsthm}
\usepackage{nicefrac}
\usepackage{microtype}
\usepackage{xcolor}
\usepackage{graphicx}
\usepackage{tikz}
\usetikzlibrary{arrows.meta,positioning,shapes.geometric}
\usepackage{enumitem}
\usepackage{caption}
\usepackage{makecell}
\usepackage{float}
\usepackage{caption}

\newtheorem{theorem}{Theorem}
\newtheorem{proposition}{Proposition}
\newtheorem{lemma}{Lemma}
\newtheorem{corollary}{Corollary}
\newtheorem{remark}{Remark}

\newcommand{\ms}[2]{$#1_{\pm #2}$}

\title{Ergodic Trajectory Design by Learned \\ Pushforward Maps:\\Provable Coverage via Conditional Flow Matching}

\author{%
  Ehsan Aghazadeh \quad Masoud Malekzadeh \quad Ahmad Ghasemi \quad Hossein Pishro-Nik \\
  University of Massachusetts Amherst \\
  \texttt{\{eaghazadeh,mmalekzadeh,aghasemi,pishro\}@umass.edu}
}

\begin{document}

\maketitle

\begin{abstract}
Designing continuous trajectories whose time-averaged occupancy provably matches a prescribed spatial density (the \emph{ergodic coverage} problem) is central to UAV-assisted data collection and sensing, robotic exploration, and mobile monitoring. For flying agents in particular, this challenge is acute: trajectories must balance coverage fidelity against tight energy budgets, no-fly zones, and acceleration limits. Existing methods either re-optimize each trajectory online (with cost growing in the horizon and re-running for every target, agent, and realization) or rely on bespoke analytical constructions that must be re-derived for each new constraint. We propose a \emph{pushforward} framework that decouples ergodicity from density matching: an analytic latent trajectory provides exact uniform ergodicity on a simple annular domain, and a single map, learned offline by optimal-transport conditional flow matching, transports this latent occupancy onto the prescribed target density. The composed trajectory is then asymptotically ergodic with respect to the learned pushforward distribution, with deviation from the target controlled by the flow-matching training loss. Once trained for a given target density and constraint set, the map serves an unbounded number of trajectories and a multi-agent fleet without per-agent retraining, and many differentiable operational constraints (no-fly zones, acceleration ceilings, or fairness penalties) enter as additive soft penalties in the training loss without re-deriving the design. We prove three results (an acceleration-energy bound, an $O(1/\sqrt{K})$ ergodic convergence rate in the number of trajectory cycles $K$, and an approximation-error bound) that combine into an end-to-end coverage bound estimable from CFM training diagnostics (certified given an architectural Lipschitz bound on $v_\theta$). Experiments on synthetic targets empirically support each theoretical envelope; on a real UAV-coverage dataset, the proposed method gives the strongest observed coverage--energy and coverage--constraint tradeoffs among the evaluated time-warping, optimization, and concurrent flow-matching baselines.
\end{abstract}

\section{Introduction}
\label{sec:intro}

A fundamental problem in spatial coverage is to design a continuous trajectory whose time-averaged occupancy provably matches a prescribed target density. This \emph{ergodic coverage} problem arises whenever a mobile agent (a UAV, a ground robot, or a sensor platform) must distribute its spatial presence according to a non-uniform importance map (see Appendix~\ref{app:applications} for concrete application domains including wireless communications, IoT data harvesting, and search-and-rescue). Formally, the goal is to construct $\mathbf{x}(t) \in \mathcal{A} \subset \mathbb{R}^2$ such that:
\begin{equation}\label{eq:ergodic_def}
    \frac{1}{T}\int_0^T \varphi(\mathbf{x}(t))\,dt \;\xrightarrow{T\to\infty}\; \int_{\mathcal{A}} \varphi(\mathbf{y})\, f_{\mathrm{target}}(\mathbf{y})\,d\mathbf{y}, \quad \forall\; \text{bounded continuous } \varphi.
\end{equation}
Achieving this is nontrivial because the agent must balance \emph{global redistribution} (reaching all regions with the correct frequency) against \emph{local continuity} (maintaining a smooth, physically realizable path), while satisfying operational constraints such as no-fly zones, speed limits, or energy budgets.

Existing methods face fundamental tradeoffs. Spectral Multiscale Coverage \citep{mathew2011metrics} is high quality but $O(K_{\mathrm{F}}^d)$ in the Fourier-basis cutoff $K_{\mathrm{F}}$ and requires per-scenario re-optimization. Grid-based and deep RL methods discretize and scale poorly, and their closed-loop variants typically require real-time environmental feedback impractical in privacy-preserving or communication-limited settings. Analytical time-warping \citep{malekzadeh2025nonuniform} is provably convergent and handles general 2D densities, but is bespoke per target/constraint, with no mechanism for additive penalties. \citet{sun2025flow} apply flow matching as per-trajectory control updates with direct dynamics accommodation, but each scenario requires fresh online optimization (non-amortized) and the proved guarantee is finite-horizon: no long-run convergence is established as $T$ grows. To our knowledge, no existing approach offers the combination of reusable amortized maps, modular constraint penalties, and \emph{asymptotic} ergodicity with an explicit $O(1/\sqrt{K})$ finite-cycle rate (Section~\ref{sec:related}).

\paragraph{Our approach: learn a pushforward map.} Instead of optimizing a trajectory directly, we \textbf{learn a pushforward map} $G_\theta: \mathbb{R}^d \to \mathbb{R}^d$ that transforms a simple, analytically constructed ergodic trajectory into one matching the prescribed target density (Figure~\ref{fig:pipeline}). The key insight is that \emph{ergodicity is preserved under any measurable pushforward}: whenever $G: D_\delta \to \mathbb{R}^2$ is Borel-measurable, the composed trajectory $\mathbf{x}(t)=G(\mathbf{z}(t))$ is exactly ergodic with respect to the pushforward $G_\#\pi_0^\delta$. For the learned map $G_\theta$, the deviation of $G_{\theta\#}\pi_0^\delta$ from $f_{\mathrm{target}}$ is quantified by our approximation bound (Theorem~\ref{thm:approx_bound}).

\paragraph{Contributions.}
\begin{enumerate}[nosep,leftmargin=*]
    \item We formalize a two-stage decomposition (latent ergodic trajectory + learned pushforward map) and prove that \emph{asymptotic} ergodicity transfers under the pushforward (Proposition~\ref{prop:ergodic}), in contrast to the finite-horizon spectral metric of prior methods (Section~\ref{sec:related}).
    \item We prove an \emph{acceleration energy bound} (Theorem~\ref{thm:acc_bound}): the energy cost is controlled by the map's Lipschitz constant $L$ with Hessian contributions vanishing as $O(\delta)$, which explains why OT-coupled training (empirically promoting low-Lipschitz, low-curvature maps) yields lower acceleration energy, with certified control of $L$ requiring spectral-normalized architectures.
    \item We prove an \emph{ergodic convergence rate} of $O(1/\sqrt{K})$ (Theorem~\ref{thm:convergence_rate}) via Cauchy--Schwarz, the Poincar\'{e} inequality on the annular domain~\citep{evans2010pde}, and i.i.d.\ cycle structure.
    \item We prove an \emph{approximation error bound} (Theorem~\ref{thm:approx_bound}, smooth-target regime) connecting the Wasserstein coverage error to the CFM training loss via Gr\"onwall ODE stability. An extension (Proposition~\ref{prop:approx_nonsmooth}) covers non-smooth or topology-mismatched targets at the cost of an additional $O(\eta_{\mathrm{reg}})+O(\delta)$ residual. Combined with the convergence rate, this yields an \emph{end-to-end coverage bound} (Corollary~\ref{cor:end_to_end}) evaluable from the CFM loss and a two-sided bracket on $L_v$.
    \item Three operational properties follow from the decomposition: \emph{modular constraint handling} via additive differentiable penalties in $\mathcal{L}_{\mathrm{total}}$, \emph{multi-agent reuse} of a single trained $G_\theta$, and \emph{open-loop, density-only operation} (no run-time observation of user or target locations), supporting deployment in bandwidth- or privacy-limited settings.
    \item We empirically support each theoretical envelope on three synthetic benchmarks and evaluate a real UAV-coverage dataset, where the method gives the strongest observed coverage--energy and coverage--constraint tradeoffs among the evaluated time-warping, optimization, and concurrent flow-matching baselines.
\end{enumerate}

\paragraph{Notation.} A complete notation guide covering all symbols, norms, and conventions used in the paper is collected in Appendix~\ref{app:notation}.

\section{Problem Setup}
\label{sec:setup}

\subsection{The Latent Ergodic Trajectory}
\label{sec:latent}

Fix $\delta > 0$ and let $D_\delta = \{\mathbf{z} \in \mathbb{R}^2 : \delta \leq \|\mathbf{z}\| \leq 1\}$ be the \emph{annular latent domain}. The excluded inner disc serves as a ``base station'' or an ``airport.'' The latent trajectory $\mathbf{z}(t)$ is a radial back-and-forth process: at the start of each cycle $k$, a fresh angle $\theta_k \sim \mathrm{Uniform}[0, 2\pi)$ is drawn independently. The agent moves outward along the ray at angle $\theta_k$ with radial position
\begin{equation}\label{eq:radial_profile}
    r(s) = \sqrt{\delta^2 + (1-\delta^2)\,s}, \quad s \in [0, 1],
\end{equation}
sweeping from $r(0) = \delta$ to $r(1) = 1$, then returns along the same ray. Writing $\mathbf{e}_\theta = (\cos\theta,\sin\theta)$, the outward leg of cycle $k$ is $\mathbf{z}_k(s) = r(s)\,\mathbf{e}_{\theta_k}$ for $s \in [0,1]$. The radial speed $\dot{r}(s) = (1-\delta^2)/(2r(s))$ is bounded by $(1-\delta^2)/(2\delta)$ everywhere, because the $\delta$-exclusion eliminates the origin singularity that would arise on a full disc. For stationary analysis we adopt the standard convention of an independent uniform phase offset $\Phi_0 \sim \mathrm{Uniform}[0,2)$ at which the observation window begins within the first cycle; this is needed only for pointwise distributional statements, not long-run averages. The inter-cycle reorientation from $\theta_{k-1}$ to $\theta_k$ takes place during a base-station rest inside $\{\|\mathbf z\|\leq\delta\}$ that is excluded from mission time (the rest duration is not part of the time-average; see Appendix~\ref{app:proof_thm1}, ``Airport turnaround'').

\emph{Regularity.} The trajectory is piecewise-$C^2$; Theorems~\ref{thm:acc_bound} and~\ref{thm:convergence_rate} integrate over open smooth interiors, and Appendix~\ref{app:boundary_smooth} gives a globally $C^1$ local-smoothing variant.

\begin{proposition}[Uniform ergodicity on $D_\delta$]\label{prop:uniform}
Under the phase convention above, the position $\mathbf{z}(t)$ is distributed as $\mathrm{Uniform}(D_\delta)$ with density $\pi_0^\delta = 1/(\pi(1-\delta^2))$ for every $t \geq 0$. Moreover, independently of this convention, the trajectory is ergodic with respect to $\pi_0^\delta$: by the i.i.d.\ cycle structure and the strong law of large numbers, $\frac{1}{T}\int_0^T \varphi(\mathbf{z}(t))\,dt \to \int_{D_\delta} \varphi\,d\pi_0^\delta$ a.s.\ for any integrable $\varphi$ (proof in Appendix~\ref{app:proofs_latent}).
\end{proposition}

The i.i.d.\ cycle structure is a deliberate design choice: it converts the ergodic convergence question into a standard problem about averages of independent random variables, avoiding mixing-time analysis entirely (Section~\ref{sec:convergence}). For typical values $\delta = 0.01$--$0.1$, the excluded area is $<3\%$ of the disc.

We present the 2D case throughout; extension to $\mathbb{R}^d$, with the resulting $\delta$ trade-off against the topology residual, is in Appendix~\ref{app:rd_extension} (Remark~\ref{rem:higher_dim}).

\subsection{Learning the Pushforward Map via Conditional Flow Matching}
\label{sec:flow}

We seek a map $G_\theta: D_\delta \to \mathbb{R}^2$ satisfying the pushforward condition $G_{\theta\#}\pi_0^\delta = f_{\mathrm{target}}$. The following result is the theoretical foundation:

\begin{proposition}[Ergodicity transfer]\label{prop:ergodic}
Let $\mathbf{z}(t)$ be ergodic with respect to $\pi_0^\delta$, and let $G: D_\delta \to \mathbb{R}^2$ be a Borel-measurable map with $G_\#\pi_0^\delta = f_{\mathrm{target}}$. Then $\mathbf{x}(t) = G(\mathbf{z}(t))$ is ergodic with respect to $f_{\mathrm{target}}$: for every $\varphi \in L^1(f_{\mathrm{target}})$, $\tfrac{1}{T}\int_0^T \varphi(\mathbf{x}(t))\,dt \to \int \varphi\,df_{\mathrm{target}}$ a.s.\ as $T\to\infty$.
\end{proposition}

The proof (Appendix~\ref{app:proofs_latent}) composes $\varphi \circ G$ with the ergodicity of $\mathbf{z}(t)$ and applies pushforward change-of-variables. Only measurability is required, so a topology mismatch between $D_\delta$ and a simply connected target is admissible at this abstract level; Remark~\ref{rem:invertibility} treats the learned diffeomorphism $G_\theta$ and the resulting $O(\delta)$ residual.

We learn $G_\theta$ using \textbf{Conditional Flow Matching with Optimal Transport coupling} (OT-CFM) \citep{lipman2023flow,tong2024improving,liu2023flow}, instantiating a dynamic-OT formulation with roots in the Benamou--Brenier fluid-mechanics view \citep{brenier1991polar,benamou2000computational,villani2009optimal,peyre2019computational}. A time-dependent velocity field $v_\theta(s, \mathbf{y})$, $s \in [0,1]$, is trained so that the ODE $\dot{\mathbf{y}}_s = v_\theta(s, \mathbf{y}_s)$ with $\mathbf{y}_0 = \mathbf{z}_0 \sim \pi_0^\delta$ produces $\mathbf{y}_1 \sim f_{\mathrm{target}}$; then $G_\theta(\mathbf{z}_0) = \mathbf{y}_1$. Training minimizes the CFM loss:
\begin{equation}\label{eq:cfm_loss}
    \mathcal{L}_{\mathrm{CFM}}(\theta) = \mathbb{E}_{s, (\mathbf{z}_0, \mathbf{x}_1)}\bigl\|v_\theta(s, \mathbf{y}_s) - (\mathbf{x}_1 - \mathbf{z}_0)\bigr\|^2,
\end{equation}
where $\mathbf{y}_s = (1-s)\mathbf{z}_0 + s\mathbf{x}_1$ and $(\mathbf{z}_0, \mathbf{x}_1)$ are OT-coupled pairs obtained via the Sinkhorn algorithm \citep{cuturi2013sinkhorn} with entropic regularization $\varepsilon_{\mathrm{sink}}$ within each mini-batch. The OT coupling produces nearly straight transport paths with low empirical curvature, which (while not a formal guarantee of small Lipschitz constant) is strongly associated with energy-efficient maps in practice (Section~\ref{sec:theory}).

\begin{remark}[Invertibility and the $\delta$ trade-off]\label{rem:invertibility}
Since $v_\theta$ is a smooth-activation MLP, Picard--Lindel\" of makes $G_\theta$ a diffeomorphism of $\mathbb{R}^2$; when the target support is simply connected, the inner boundary $\{\|\mathbf z\|=\delta\}$ cannot be exactly collapsed, contributing the topology residual $\eta_{\mathrm{top}}(\delta)=O(\delta)$ in Corollary~\ref{cor:end_to_end}. Shrinking $\delta$ reduces this residual but inflates latent energy ($E_{\mathrm{acc}}(\mathbf z)=\Theta(K/\delta^4)$ in $d{=}2$ with per-cycle scaling $\Theta(1/\delta^4)$, consistent with Remark~\ref{rem:higher_dim}'s $\Theta(K/\delta^{3d-2})$); $\delta\in[0.01,0.1]$ balances the two. Appendix~\ref{app:invertibility} covers Lipschitz certification (spectral normalization), Gr\"onwall flow bounds, higher-dimensional scaling, and topology-mismatch geometry.
\end{remark}

\begin{figure}[t]
\centering
\begin{tikzpicture}[
    box/.style={draw, rounded corners, minimum height=1cm, minimum width=2.2cm, align=center, fill=blue!8, font=\small},
    arrow/.style={-{Stealth[length=2.5mm]}, thick},
    note/.style={font=\scriptsize\itshape, text width=2.6cm, align=center}
]
\node[box] (latent) {Latent Trajectory\\$\mathbf{z}(t)$ on $D_\delta$};
\node[box, right=1.8cm of latent] (map) {Learned Map\\$G_\theta$ (OT-CFM)};
\node[box, right=2.8cm of map] (target) {Target Trajectory\\$\mathbf{x}(t) = G_\theta(\mathbf{z}(t))$};

\draw[arrow] (latent) -- node[above, font=\scriptsize]{pushforward} (map);
\draw[arrow] (map) -- node[above, font=\scriptsize]{$G_{\theta\#}\pi_0^\delta \approx f_{\mathrm{target}}$} (target);

\node[note, below=0.15cm of latent] {Structurally ergodic\\(Prop.~\ref{prop:uniform})\\Uniform on annulus};
\node[note, below=0.15cm of map] {Trained once\\Shared by all agents\\+ constraint penalties};
\node[note, below=0.15cm of target] {Ergodic w.r.t.\ $G_{\theta\#}\pi_0^\delta$\\gap to $f_{\mathrm{target}}$: Theorem~\ref{thm:approx_bound} $O(1)$ inference};

\end{tikzpicture}
\caption{Two-stage pipeline. The latent trajectory provides structural ergodicity (i.i.d.\ cycles, uniform density); the learned map transforms it to match the prescribed target density. Theoretical guarantees (Theorems~\ref{thm:acc_bound}--\ref{thm:approx_bound}) bound energy, convergence, and approximation error.}
\label{fig:pipeline}
\end{figure}
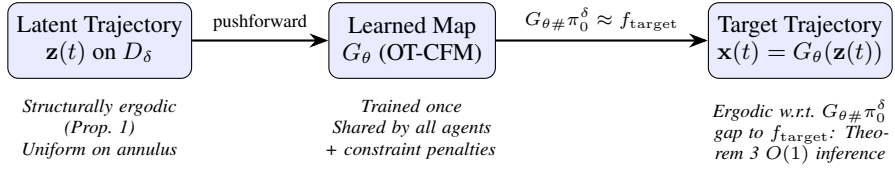
\subsection{Constraint Flexibility and Multi-Agent Reuse}
\label{sec:constraints}

A practical advantage is that many differentiable operational constraints can be incorporated as soft penalties in the training loss:
\begin{equation}\label{eq:total_loss}
    \mathcal{L}_{\mathrm{total}}(\theta) = \mathcal{L}_{\mathrm{CFM}}(\theta) + \lambda_{\mathrm{nfz}}\,\mathcal{R}_{\mathrm{nfz}}(\theta) + \lambda_{\mathrm{acc}}\,\mathcal{R}_{\mathrm{acc}}(\theta) + \cdots
\end{equation}
For instance, a no-fly zone (NFZ) centered at $\mathbf{c}$ with radius $r_{\mathrm{nfz}}$ is penalized via $\mathcal{R}_{\mathrm{nfz}} = \mathbb{E}_{\mathbf{z}}[\max(0, r_{\mathrm{nfz}} - \|G_\theta(\mathbf{z}) - \mathbf{c}\|)^2]$, and an acceleration penalty via $\mathcal{R}_{\mathrm{acc}} = \mathbb{E}_{\mathbf{z}}[\|H_{G_\theta}(\mathbf{z})\|_F^2]$, where $H_{G_\theta}$ is the Hessian tensor. Prior analytical methods instead require a separate mathematical rederivation for each new constraint. Exclusion constraints such as NFZs could alternatively be encoded by setting $f_{\mathrm{target}} = 0$ in the excluded region, so that an exact pushforward would then never enter it. However, the learned map $G_\theta$ is an approximation (Theorem~\ref{thm:approx_bound}), and the explicit penalty provides direct gradient signal at the constraint boundary to suppress residual violations that the density-based encoding alone cannot prevent.

\paragraph{Multi-agent reuse.}\label{sec:multi}
A single trained $G_\theta$ supports independent multi-agent deployment: $N$ agents running independent latent trajectories and sharing $G_\theta$ each remain individually ergodic with respect to the same achieved pushforward (gap to $f_{\mathrm{target}}$ controlled by Theorem~\ref{thm:approx_bound}), and the collective time-average converges at $O(1/\sqrt{NK})$ with no per-agent optimization. Coordinated extensions via joint or mean-field pushforward maps (with coupling objectives entering as additional penalty terms in~\eqref{eq:total_loss}, analogous to the NFZ and acceleration penalties above) are developed in Appendix~\ref{app:coordinated}.

\section{Theoretical Guarantees}
\label{sec:theory}

\subsection{Acceleration Energy Bound}
\label{sec:acc_bound}

We first quantify the energy cost of applying the learned map. For a $C^2$ map $G:\Omega\to\mathbb{R}^2$ on an open set $\Omega\subset\mathbb{R}^2$ and a compact $\mathcal{K}\subset\Omega$, define $L = \sup_{\mathbf{z}\in \mathcal{K}}\|J_G(\mathbf{z})\|_{\mathrm{op}}$ (Lipschitz constant on $\mathcal{K}$) and $M_H = \sup_{\mathbf{z}\in \mathcal{K}}\sqrt{\|\nabla^2 G_1(\mathbf{z})\|_{\mathrm{op}}^2 + \|\nabla^2 G_2(\mathbf{z})\|_{\mathrm{op}}^2}$ (Hessian tensor norm on $\mathcal{K}$); both are finite by continuity since $G\in C^2(\Omega)$ and $\mathcal{K}$ is compact. For the radial latent construction we take $\mathcal{K}=D_\delta$.

\begin{theorem}[Acceleration Energy Bound]\label{thm:acc_bound}
Let $\Omega \subset \mathbb{R}^2$ be open, let $G: \Omega \to \mathbb{R}^2$ be $C^2$, let $\mathcal{K}\subset\Omega$ be compact, and let $\mathbf{z}(t) \in \mathcal{K}$ be piecewise-$C^2$ on a finite partition $0 = t_0 < t_1 < \cdots < t_N = T$. Define the transformed trajectory $\mathbf{x}(t) = G(\mathbf{z}(t))$, with integrals $\int\|\ddot{\mathbf z}\|^2\,dt$ and $\int\|\dot{\mathbf z}\|^4\,dt$ taken over the union of the open smooth intervals $(t_{i-1},t_i)$. Then
\begin{equation}\label{eq:acc_bound}
    \sqrt{E_{\mathrm{acc}}(\mathbf{x})} \;\leq\; L\,\sqrt{E_{\mathrm{acc}}(\mathbf{z})} \;+\; M_H\,\sqrt{\Phi_4(\mathbf{z})},
\end{equation}
where $E_{\mathrm{acc}}(\cdot) = \int\|\ddot{\cdot}\|^2\,dt$ is the acceleration energy and $\Phi_4(\mathbf{z}) = \int\|\dot{\mathbf{z}}\|^4\,dt$.
\end{theorem}

The proof (Appendix~\ref{app:proof_thm1}) decomposes the acceleration $\ddot{\mathbf{x}} = J_G\ddot{\mathbf{z}} + \dot{\mathbf{z}}^\top H_G\dot{\mathbf{z}}$ and applies the Minkowski inequality in $L^2$.

For the radial latent trajectory on $D_\delta$, explicit computation (Lemma~\ref{lem:moments} in Appendix~\ref{app:proof_thm1}) shows $\Phi_4/E_{\mathrm{acc}} = 2\delta^2/(1+\delta^2) \to 0$ as $\delta \to 0$. This yields:

\begin{corollary}[Acceleration ratio]\label{cor:ratio}
For the radial latent trajectory on $D_\delta$:
\begin{equation}\label{eq:ratio}
    \frac{E_{\mathrm{acc}}(\mathbf{x})}{E_{\mathrm{acc}}(\mathbf{z})} \;\leq\; \bigl(L + M_H\sqrt{2\delta^2/(1+\delta^2)}\,\bigr)^2 \;\approx\; L^2 + 2\sqrt{2}\,L\,M_H\,\delta + O(\delta^2).
\end{equation}
\end{corollary}

\textbf{Interpretation.} The ratio is asymptotically dominated by $L^2$: the Hessian contribution vanishes as $O(\delta)$. OT-coupled training promotes (without certifying) small $L$ via near-rectilinear, low-curvature paths; certified control requires spectral-norm-constrained architectures. Throughout the paper $E_{\mathrm{acc}}$ denotes interior smooth-leg energy; turnaround/reorientation impulses at the outer boundary are budgeted separately as discrete maneuver costs (Appendix~\ref{app:boundary_smooth}).

\subsection{Ergodic Convergence Rate}
\label{sec:convergence}

\begin{theorem}[Ergodic Convergence Rate]\label{thm:convergence_rate}
Let $G: D_\delta \to \mathbb{R}^2$ be Lipschitz with constant $L$, and let $\varphi$ be a \emph{bounded Lipschitz} test function with $\|\varphi\|_\infty \leq B_\varphi$ and Lipschitz constant $L_\varphi$. Let $\mathbf{x}_k(s) = G(\mathbf{z}_k(s))$ be the transformed outward leg of cycle $k$, and define the per-cycle integral $X_k = \int_0^1 \varphi(\mathbf{x}_k(s))\,ds$. The $X_k$ are i.i.d.\ with common mean $\mu_G = \mathbb{E}[X_k] = \int \varphi\,d(G_\#\pi_0^\delta)$. Let $S_K = \tfrac{1}{K}\sum_{k=1}^K X_k$ denote the time average over $K$ cycles. Then:
\begin{enumerate}[label=(\alph*),nosep]
    \item $\mathrm{Var}(X_k) \leq C_{D_\delta}\,L^2\,L_\varphi^2$, where $C_{D_\delta}$ is the Neumann Poincar\'{e} constant of $D_\delta$~\citep{evans2010pde}. For small $\delta$, $C_{D_\delta} = C_D + O(\delta^2)$ with $C_D = 1/(j'_{1,1})^2 \approx 0.295$ the Poincar\'{e} constant of the unit disc ($j'_{1,1}$ being the first positive zero of the derivative of the Bessel function $J_1$ of the first kind)~\citep[Theorem~7.1.1]{henrot2006extremum}.
    \item $\Pr(|S_K - \mu_G| > \varepsilon) \leq 2\exp(-K\varepsilon^2/(2B_\varphi^2))$\quad (Hoeffding).
    \item $|S_K - \int\varphi\,f_{\mathrm{target}}| \leq \underbrace{|S_K - \mu_G|}_{\text{statistical } O(1/\sqrt{K})} + \underbrace{L_\varphi \cdot W_1(G_\#\pi_0^\delta, f_{\mathrm{target}})}_{\text{approximation}}$.
\end{enumerate}
The bounded-Lipschitz test-function class is consistent with the bounded-continuous class used in~\eqref{eq:ergodic_def}; the density argument recovering~\eqref{eq:ergodic_def} as $K\to\infty$ is given in Appendix~\ref{app:proof_thm2}.
\end{theorem}

\subsection{Approximation Error Bound}
\label{sec:approx}

The approximation term $W_1(G_{\theta\#}\pi_0^\delta, f_{\mathrm{target}})$ in Theorem~\ref{thm:convergence_rate}(c) is bounded via the velocity field quality (smooth-target regime; non-smooth and topology-mismatched cases via Proposition~\ref{prop:approx_nonsmooth}):

\begin{theorem}[Approximation Error Bound, smooth-target regime]\label{thm:approx_bound}
Suppose: \emph{(A1)} $\pi_0^\delta$ and $f_{\mathrm{target}}$ have finite second moments, and an OT velocity field $v^*$ with time-1 flow $G^*$ satisfying $G^*_\#\pi_0^\delta = f_{\mathrm{target}}$ exists; \emph{(A2)} $v^*(s,\cdot)$ is globally Lipschitz uniformly in $s\in[0,1]$; \emph{(A3)} $v_\theta(s,\cdot)$ is $L_v$-Lipschitz uniformly in $s$. Sufficient conditions for (A1)--(A3) and the admissibility of $(G_\theta,G^*)$ as a $W_2$-coupling (which follows from (A1)) are collected in Appendix~\ref{app:proof_thm3}. Define $\varepsilon_v^2 = \int_0^1 \mathbb{E}_{\mathbf{y}\sim p_s}\bigl[\|v_\theta(s,\mathbf{y}) - v^*(s,\mathbf{y})\|^2\bigr]\,ds$, where $p_s$ is the law of the OT displacement interpolant from~\eqref{eq:cfm_loss}. Then
\begin{equation}\label{eq:approx_bound}
    W_2(G_{\theta\#}\pi_0^\delta,\, f_{\mathrm{target}}) \;\leq\; e^{L_v}\cdot\varepsilon_v.
\end{equation}
\end{theorem}

\textbf{Scope of (A2).} (A2) is the operational smooth-regime hypothesis used by the Gr\"onwall step. Sufficient conditions where (A2) provably holds include Gaussian-to-Gaussian transport, other explicitly solvable transport families, and entropically regularized OT (whose population velocity field is smooth by construction). Gaussian-mixture targets (Exp.~1) are treated experimentally as smooth bounded-density targets in this regime, but do not by themselves guarantee a globally Lipschitz Brenier velocity without additional regularity. Density discontinuities (Exp.~2) and topology mismatch fall outside; in those regimes Proposition~\ref{prop:approx_nonsmooth} gives the weaker bound $W_2 \le e^{L_v}\cdot\varepsilon_v^{\eta,\delta} + O(\eta) + O(\delta)$ via a mollified/restricted surrogate (full statement and proof in Appendix~\ref{app:proof_thm3}).

\textbf{Connection to training loss.} A separate decomposition (Appendix~\ref{app:proof_thm3}, Eq.~\eqref{eq:eps_v_decomp_app}) bridges $\varepsilon_v$ to the empirical CFM loss via $\varepsilon_v \le \sqrt{\mathcal{L}_{\mathrm{CFM}}(\theta;\pi_\varepsilon) - \mathcal{L}_{\mathrm{CFM}}^*(\pi_\varepsilon)} + R_{\mathrm{sink}}(\varepsilon_{\mathrm{sink}})$, where the first term is the trainable (reducible) component and $R_{\mathrm{sink}}$ is an entropic-OT bias residual that vanishes in the exact-Brenier limit. This makes the Theorem~\ref{thm:approx_bound} envelope computable from training quantities up to the explicitly tracked Sinkhorn residual.

\subsection{End-to-End Coverage Bound}

Combining Theorem~\ref{thm:convergence_rate} with Theorem~\ref{thm:approx_bound} in the smooth regime or Proposition~\ref{prop:approx_nonsmooth} in the non-smooth/topology-mismatched regime:

\begin{corollary}[End-to-end bound]\label{cor:end_to_end}
Let $\varphi$ be bounded Lipschitz with $\|\varphi\|_\infty \le B_\varphi$ and Lipschitz constant $L_\varphi$. With probability at least $1-\alpha$, the time-averaged coverage error after $K$ i.i.d.\ cycles satisfies
\begin{equation}\label{eq:end_to_end}
    \left|S_K - \int\varphi\,f_{\mathrm{target}}\right| \;\leq\; \underbrace{B_\varphi\sqrt{\tfrac{2\ln(2/\alpha)}{K}}}_{\text{statistical: } O(K^{-1/2})} \;+\; \underbrace{L_\varphi\,e^{L_v}\cdot\varepsilon_v}_{\text{trainable approximation}} \;+\; \underbrace{L_\varphi\,\eta_{\mathrm{top}}(\delta)}_{\text{topology residual: } O(\delta)},
\end{equation}
where $\eta_{\mathrm{top}}(\delta) = O(\delta)$ is the Wasserstein cost of the annular-to-target rearrangement (Proposition~\ref{prop:approx_nonsmooth}) and vanishes when $f_{\mathrm{target}}$ is supported on a set homeomorphic to $D_\delta$ (so that (A2) of Theorem~\ref{thm:approx_bound} applies). Conditional on the precondition
\begin{equation}\label{eq:end_to_end_precond}
    \varepsilon \;>\; L_\varphi\bigl(e^{L_v}\cdot\varepsilon_v + \eta_{\mathrm{top}}(\delta)\bigr),
\end{equation}
the sample complexity to achieve error $\leq \varepsilon$ is
\begin{equation}\label{eq:end_to_end_samplecomplexity}
K \;\geq\; \frac{2 B_\varphi^2 \ln(2/\alpha)}{\bigl(\varepsilon - L_\varphi e^{L_v}\cdot\varepsilon_v - L_\varphi \eta_{\mathrm{top}}(\delta)\bigr)^2}.
\end{equation}
When~\eqref{eq:end_to_end_precond} fails, the approximation--topology floor dominates regardless of $K$ (monitoring in Appendix~\ref{app:floor}).

\emph{Non-smooth-target case.} When (A2) fails, the surrogate of Proposition~\ref{prop:approx_nonsmooth} replaces $\varepsilon_v$ by $\varepsilon_v^{\eta,\delta}$ and adds a mollification residual $L_\varphi\eta_{\mathrm{reg}}(\eta) = O(\eta)$:
\begin{equation}\label{eq:end_to_end_nonsmooth}
\left|S_K - \int\varphi\,f_{\mathrm{target}}\right| \;\leq\; B_\varphi\sqrt{\tfrac{2\ln(2/\alpha)}{K}} \;+\; L_\varphi e^{L_v}\cdot\varepsilon_v^{\eta,\delta} \;+\; L_\varphi\eta_{\mathrm{reg}}(\eta) \;+\; L_\varphi\eta_{\mathrm{top}}(\delta);
\end{equation}
the precondition and sample complexity follow~\eqref{eq:end_to_end_precond}--\eqref{eq:end_to_end_samplecomplexity} with the same substitution $\varepsilon_v\to\varepsilon_v^{\eta,\delta}$ and the additional floor term $L_\varphi\eta_{\mathrm{reg}}(\eta)$.
\end{corollary}

All terms in~\eqref{eq:end_to_end} are evaluable from training/inference quantities: $\varepsilon_v$ from the CFM loss via~\eqref{eq:eps_v_decomp_app} (modulo an empirically-tracked Sinkhorn-coupling residual $O(\varepsilon_{\mathrm{sink}})$), $L_v$ via the two-sided bracket $\hat L_v \leq L_v \leq L_v^{\mathrm{net}}$ (Appendix~\ref{app:Lv_protocol}; numerical floors use $\hat L_v$ for non-vacuity, while the architectural $L_v^{\mathrm{net}}$ is $\sim$30$\times$ looser), and $\eta_{\mathrm{top}}(\delta)=0$ when $\mathrm{supp}(f_{\mathrm{target}})$ is homotopy-equivalent to $D_\delta$ and $O(\delta)$ otherwise (per-experiment case analysis in Appendix~\ref{app:floor}). For Exp.~1, $\hat L_v\approx 3$ and $\varepsilon_v\approx 0.039$ give the numerical floor $e^{\hat L_v}\cdot\varepsilon_v\approx 0.78$, $\approx$10$\times$ above the observed $1-\rho_{\mathrm{iid}}\approx 0.08$ (Table~\ref{tab:Lv_calibration}).

\section{Experiments}
\label{sec:experiments}

\paragraph{Experimental setup.}
We evaluate on two benchmarks: three synthetic targets isolating individual
theoretical mechanisms, and a real-world dataset.
Both use an MLP for $v_\theta$ trained with
Adam and Sinkhorn OT coupling; the model architecture used for the synthetic targets is shallower than that for the real-world benchmark.
Full implementation details and metrics are provided in
Appendix~\ref{app:experiments_full}.





\subsection{Synthetic-target validation of theoretical bounds}
\label{sec:exp_synth}
The three synthetic targets each isolate one theoretical mechanism: a Gaussian mixture (density matching), a binary $75/25$ density (demand-proportional allocation, vs.\ the time-warping formula of \citet{malekzadeh2025nonuniform}), and a Gaussian mixture with an off-center no-fly zone (soft-constraint Pareto). Headline numbers: Experiment~1 achieves IID transport correlation $\rho_{\mathrm{iid}} = 0.92$ at $\varepsilon_v \approx 0.039$; Experiment~2 reaches an $L_1$ allocation deviation of $0.2$ percentage points on the $75/25$ binary target, a $\mathbf{10\times}$ tighter match than the time-warping formula ($2.0$ pp); Experiment~3 traces a Pareto frontier in which an NFZ-only penalty cuts violation $\mathbf{7\times}$ (from $0.7\%$ to $0.1\%$) and an added Hessian regularizer drives acceleration to a sub-unit $\mathbf{0.93\times}$ (smoother than the latent trajectory), at distinct $(\lambda_{\mathrm{nfz}},\lambda_{\mathrm{acc}})$ operating points. Per-experiment setups and full results are in Appendix~\ref{app:experiments_full}; figures are in Appendix~\ref{app:figures}; the Experiment~3 Pareto outcomes are tabulated in Table~\ref{tab:exp3}.
 
\textbf{Theorem-envelope diagnostics.}
The synthetic experiments quantitatively support the components of the
end-to-end coverage bound.
\emph{Convergence rate (Theorem~\ref{thm:convergence_rate}).}
The Hoeffding RMSE prediction $\sigma/\sqrt{K}\approx 0.085$ at $K=300$
envelopes the observed Experiment~1 IID-vs-trajectory residual $\approx 0.05$;
Figure~\ref{fig:convergence} (Appendix~\ref{app:figures}) shows empirical
log--log slope $-0.48$ vs.\ theoretical $-0.50$, with sensitivity diagnostics
in Appendix~\ref{app:slope_fit}.
\emph{Energy bound (Corollary~\ref{cor:ratio}).}
The Experiment~2 sub-unit acceleration ratio $0.71\times$ is consistent with
the worst-case envelope (sup-Lipschitz $\hat{L}\approx 2$,
Table~\ref{tab:Lv_calibration}) once the trajectory's angle-averaged effective
Jacobian is accounted for (Appendix~\ref{app:proof_thm1}, Remark on tightness).
\emph{Approximation bound (Theorem~\ref{thm:approx_bound}).}
The CFM-loss-derived $\varepsilon_v\approx 0.039$ produces a diagnostic plug-in
floor $e^{\hat{L}_v}\varepsilon_v$ that envelopes the empirical $\hat{W}_2$
across all five rows of Table~\ref{tab:Lv_calibration} at $\hat{L}_v\leq 3$;
the certified upper-bound floor uses the architectural $L_v^{\mathrm{net}}$
(Appendix~\ref{app:Lv_protocol}).

\subsection{Real-data evaluation: Milano UAV coverage}
\label{sec:exp_real}

Having validated the theoretical guarantees on synthetic targets, we now
evaluate practical effectiveness on the Milano benchmark~\citep{barlacchi2015multi} 
(dataset and preprocessing details in Appendix~\ref{app:milano}).
Two no-fly configurations test modular constraint handling: a
\textbf{single-disc} (1D) and a \textbf{multi-disc} (MD) with four disjoint
exclusion zones. Six baselines are evaluated: TOES~\citep{dong2023timeoptimal},
time-warping~\citep{malekzadeh2025nonuniform}, static
OT~\citep{peyre2019computational}, SAC~\citep{haarnoja2018soft}, and the
FMEC-Stein/FMEC-Sinkhorn variants of \citet{sun2025flow}.
Six variants of our proposed method differ only in the constraint penalty 
of~\eqref{eq:total_loss}: \textbf{OT-CFM} (no penalty), \textbf{+NFZ},
\textbf{+E} (Zeng rotary-wing power~\citep{zeng2019energy}), \textbf{+A}
(acceleration), \textbf{+NFZ+E}, and \textbf{+NEA} (NFZ$+$E$+$A combined).
The best $\lambda$ per variant is selected from a single-seed sweep
(Appendix~\ref{app:experiments_full}).

\begin{table}[t]
\centering
\caption{Headline metrics on Milano under two no-fly configurations: single-disc (1D) and multi-disc (MD). Mean$_{\pm\text{std}}$ over 3 seeds; \textbf{best} per column among coverage-competitive methods ($\rho \ge 0.5$); $^\ddagger$ achieves low violation at degraded coverage. Full coverage, energy, and NFZ tables in Appendix~\ref{app:full_metrics}.}
\label{tab:milano_headline}
\scriptsize
\setlength{\tabcolsep}{3pt}
\begin{tabular}{l cc cc cc}
\toprule
& \multicolumn{2}{c}{$\rho\uparrow$} & \multicolumn{2}{c}{$\bar P$ (W)$\downarrow$} & \multicolumn{2}{c}{NFZ frac (\%)$\downarrow$} \\
\cmidrule(lr){2-3}\cmidrule(lr){4-5}\cmidrule(lr){6-7}
Method & 1D & MD & 1D & MD & 1D & MD \\
\midrule
TOES~\citep{dong2023timeoptimal}$^\dagger$          & \ms{0.286}{0}        & \ms{0.265}{0}        & \ms{160}{0}                 & \ms{160}{0}                 & \ms{1.44}{0}              & \ms{0}{0}                 \\
Time-warp~\citep{malekzadeh2025nonuniform}          & \ms{0.690}{.04}      & \ms{0.862}{.01}      & \ms{532}{7}                 & \ms{555}{4}                 & \ms{9.11}{1.4}            & \ms{1.03}{.14}            \\
OT~\citep{peyre2019computational}                   & \ms{0.575}{.01}      & \ms{0.570}{.01}      & \ms{1.0{\cdot}10^{4}}{225}  & \ms{1.05{\cdot}10^{4}}{282} & \ms{3.25}{.10}            & \ms{4.20}{.36}            \\
SAC~\citep{haarnoja2018soft}                        & \ms{0.395}{.08}      & \ms{0.228}{.13}      & \ms{2.55{\cdot}10^{4}}{984} & \ms{1.98{\cdot}10^{4}}{2600}& \ms{5.23}{3.7}            & \ms{35.0}{9.3}            \\
FMEC-Stein~\citep{sun2025flow}$^\dagger$            & \ms{-.119}{.0001}    & \ms{-.094}{.0002}    & \ms{164}{0.6}               & \ms{164}{0.6}               & \ms{0}{0}                 & \ms{0}{0}                 \\
FMEC-Sinkhorn~\citep{sun2025flow}$^\dagger$         & \ms{-.080}{.003}     & \ms{-.065}{.0001}    & \ms{166}{0.03}              & \ms{166}{0.04}              & \ms{0}{0}                 & \ms{0}{0}                 \\
\midrule
OT-CFM                  & \ms{0.839}{.003}     & \ms{0.908}{.003}     & \ms{2.12{\cdot}10^{3}}{47}  & \ms{2.38{\cdot}10^{3}}{14}  & \ms{2.27}{.08}            & \ms{1.56}{.18}            \\
OT-CFM +NFZ             & \ms{\mathbf{0.865}}{.001} & \ms{\mathbf{0.911}}{.002} & \ms{2.23{\cdot}10^{3}}{9}   & \ms{2.24{\cdot}10^{3}}{29}  & \ms{\mathbf{0.011}}{.010} & \ms{\mathbf{0.189}}{.08}  \\
OT-CFM +E               & \ms{0.646}{.07}      & \ms{0.737}{.005}     & \ms{\mathbf{510}}{143}      & \ms{\mathbf{514}}{133}      & \ms{5.98}{1.3}            & \ms{13.2}{0.4}            \\
OT-CFM +A               & \ms{0.410}{.07}      & \ms{0.430}{.14}      & \ms{1.02{\cdot}10^{3}}{180} & \ms{1.04{\cdot}10^{3}}{314} & \ms{5.10}{3.0}            & \ms{16.4}{1.9}            \\
OT-CFM +NFZ+E           & \ms{0.572}{.02}      & \ms{0.230}{.05}      & \ms{819}{72}                & \ms{1.68{\cdot}10^{3}}{34}  & \ms{0.011}{.010}          & \ms{11.9}{2.8}            \\
OT-CFM +NEA$^\ddagger$  & \ms{0.104}{.13}      & \ms{0.188}{.18}      & \ms{645}{305}               & \ms{457}{289}               & \ms{0.102}{.15}           & \ms{0.032}{.04}           \\
\bottomrule
\multicolumn{7}{p{0.96\textwidth}}{\scriptsize $^\dagger$Low $\bar{P}$ and zero NFZ frac are artifacts of degenerate coverage ($\rho \leq 0.29$): these methods fail to redistribute mass across the domain and should not be read as genuinely energy-efficient or constraint-respecting. See Appendix~\ref{app:trajectories} for trajectory visualizations.}\\
\multicolumn{7}{l}{\scriptsize $^\ddagger$Achieves low violation at the cost of degraded coverage ($\rho = 0.10$/$0.19$).}\\
\end{tabular}
\end{table}

\begin{figure}[t]
\centering
\includegraphics[width=0.49\textwidth]{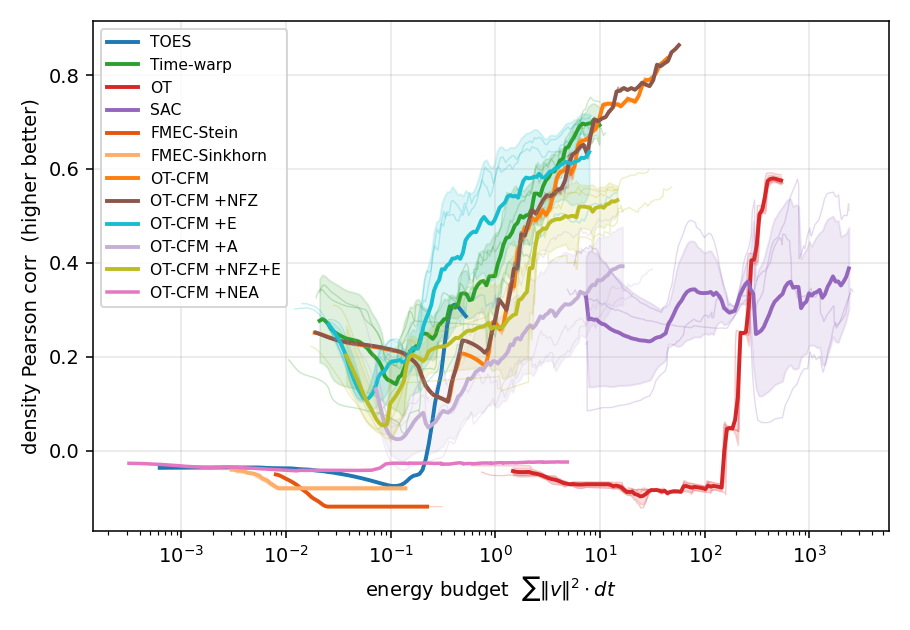}\hfill
\includegraphics[width=0.49\textwidth]{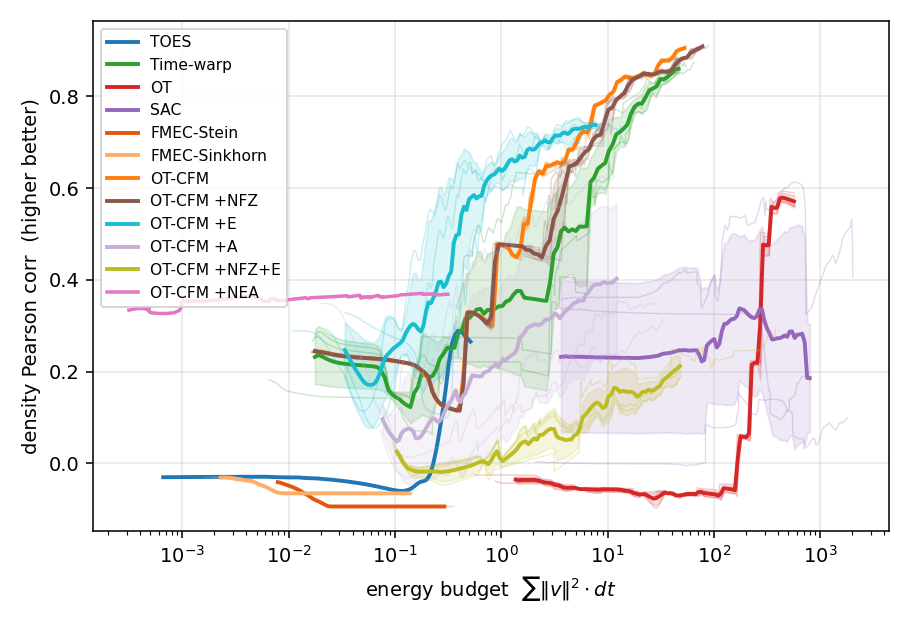}
\caption{Coverage--energy Pareto on Milano. \textbf{Left:} single-disc (1D) NFZ. \textbf{Right:} multi-disc (MD) NFZ. Pearson $\rho$ (higher better) plotted against the swept energy budget $\sum\|v\|^2\,dt$ (log scale; lower-right is best); shaded bands are $\pm 1$ std over three seeds. The OT-CFM family traces an empirical Pareto front above the evaluated baselines in both configurations; \textbf{+E} occupies the low-energy end and \textbf{+NFZ} the high-fidelity end. The constraint-axis Pareto pattern is summarized in Table~\ref{tab:milano_headline}; full coverage--NFZ figure in Appendix~\ref{app:milano_nfz_fig}.}
\label{fig:milano_pareto}
\end{figure}

\textbf{Empirical Pareto behavior.} On both NFZ configurations, OT-CFM~\textbf{+NFZ} jointly achieves the highest coverage ($\rho = 0.865 / 0.911$ for 1D / MD) and the lowest constraint violation among coverage-competitive methods (frac inside $0.011\% / 0.189\%$, a $\sim$$200\times / 8\times$ reduction over unconstrained OT-CFM); OT-CFM~\textbf{+E} attains the lowest mean Zeng power ($510 / 514$\,W, $\sim$$4\times$ below unconstrained). These gaps are predicted by theory: Theorem~\ref{thm:approx_bound} via OT-coupled CFM training, Theorem~\ref{thm:acc_bound} via the optional $\mathcal{R}_{\mathrm{acc}}$ regularizer. OT-CFM~+NFZ achieves higher coverage on MD ($\rho = 0.911$) than on 1D ($\rho = 0.865$) despite four exclusion zones rather than one, evidence that the modular soft-penalty design (\S\ref{sec:constraints}) absorbs constraint topology without re-derivation. Apparent zero violation and lowest $\bar{P}$ by TOES and FMEC are
artifacts of degenerate coverage ($\rho\leq 0.29$): these methods fail
to redistribute mass across the domain, making low energy and zero NFZ
violation accidental byproducts of near-stationary trajectories rather
than genuine efficiency gains (Appendix~\ref{app:trajectories}).
Full metrics are in Appendix~\ref{app:full_metrics}.

\textbf{Multi-agent reuse, amortization, and realizability.} A single trained $G_\theta$ supports an $N$-agent fleet with aggregate Fourier ergodic metric tracking the predicted $O(1/\sqrt{NK})$ rate (\S\ref{sec:multi}; Theorem~\ref{thm:convergence_rate}): on Milano, $\mathcal{E}_N$ drops $\sim$$2.5\times$ from $N{=}1$ to $N{=}20$ at flat per-agent power (Appendix~\ref{app:milano_amortization}). Training is paid once; every additional trajectory or agent costs a single forward pass through $v_\theta$. OT-CFM~\textbf{+E} is also the only OT-CFM variant whose median speed lies below the typical $30$\,m/s rotary-wing cap (Appendix~\ref{app:velocity}).

\subsection{Comparison with Prior Methods}

The combination of finite-cycle convergence, amortized inference, multi-agent reuse, $O(T_{\mathrm{NFE}})$ deployment cost (or $O(1)$ distilled, Appendix~\ref{app:inference_efficiency}), and modular soft-penalty constraints in Table~\ref{tab:comparison} is, to our knowledge, the first such combination in the ergodic-coverage literature; per-row qualifications and the per-trajectory cost gap with \citet{sun2025flow} are detailed in Appendix~\ref{app:comparison_footnotes}.

\begin{table}[t]
  \caption{Qualitative comparison with existing ergodic trajectory design methods. Quantitative real-data results: \S\ref{sec:exp_real}; column definitions, $T_{\mathrm{NFE}}$ semantics, and detailed row qualifications: Appendix~\ref{app:comparison_footnotes}.}
  \label{tab:comparison}
  \centering
  \small
  \begin{tabular}{lccccc}
    \toprule
    Method & \makecell{Conv.\\rate} & \makecell{Amortized\\inference} & \makecell{Multi-\\agent} & \makecell{Inference\\cost} & \makecell{Modular\\constraints} \\
    \midrule
    Time-warp \citep{malekzadeh2025nonuniform} & $O(1/\sqrt{K})$ & Yes & Yes & $O(1)$ & No$^\dagger$ \\
    OT \citep{peyre2019computational} & None & No & No & Per-target$^\flat$ & Cost-matrix only \\
    SAC \citep{haarnoja2018soft} & None & Yes & No & $O(1)$ & Reward shaping \\
    TOES \citep{dong2023timeoptimal}$^{\diamond,\natural}$ & None$^\diamond$ & Per-query$^\diamond$ & No & $O(\text{NLP})^\diamond$ & No$^\diamond$ \\
    FMEC \citep{sun2025flow}$^{\ddagger,\natural}$ & None$^\ddagger$ & Per-query$^\ddagger$ & No & $O(N_{\mathrm{iter}}T)^{\ddagger}$ & Soft only$^\ddagger$ \\
    \textbf{Ours} & $O(1/\sqrt{K})$ & \textbf{Yes} & \textbf{Yes} & $O(T_{\mathrm{NFE}})$ / $O(1)^\S$ & \textbf{Yes} \\
    \bottomrule
  \end{tabular}

  \vspace{2pt}
  {\scriptsize $^\dagger$Per-target/per-constraint analytical rederivation. $^\ddagger$Per-trajectory flow-time outer loop~\citep{sun2025flowcode}. $^\diamond$Per-query Pontryagin solve. $^\S$$O(1)$ via distilled evaluator (Appendix~\ref{app:inference_efficiency}). $^\flat$Per-target Sinkhorn solve. $^\natural$Quantitative results in Appendix~\ref{app:full_metrics}. Full qualifications: Appendix~\ref{app:comparison_footnotes}.}
\end{table}

\section{Related Work}
\label{sec:related}

\paragraph{Distinguishers.} Our work differs from prior ergodic-coverage and density-matching methods along three axes: (i) \emph{object of learning}: a reusable spatial pushforward map $G_\theta$ trained once, rather than per-instance trajectory/control optimization; (ii) \emph{theoretical surface}: asymptotic ergodicity by construction (Proposition~\ref{prop:ergodic}) together with an explicit $O(1/\sqrt K)$ finite-cycle rate (Theorem~\ref{thm:convergence_rate}) and an end-to-end coverage bound (Corollary~\ref{cor:end_to_end}) estimable from the training loss (certified given an architectural Lipschitz bound on $v_\theta$), rather than step-level optimality or finite-horizon spectral-metric minimization without long-horizon convergence guarantees; (iii) \emph{constraint handling}: arbitrary differentiable soft penalties composed additively in $\mathcal{L}_{\mathrm{total}}$, rather than bespoke analytical rederivation or inner-solver modification.

\paragraph{What ``ergodic'' means in this literature.} The term is used in three distinct senses: (i) strict dynamical-systems, (ii) asymptotic empirical, and (iii) finite-horizon coverage-metric. We adopt (i)+(ii), matching long-horizon UAV monitoring and wireless-coverage missions where horizons are open-ended; recent trajectory-optimization variants \citep{sun2025flow,dong2023timeoptimal,lee2024stein,sun2025kernel,hughes2025mmd} adopt only (iii). Citation-grounded disambiguation is in Appendix~\ref{app:ergodic_senses}.

\paragraph{Other ergodic-coverage paradigms.} The spectral lineage \citep{mathew2011metrics,miller2013trajectory,abraham2018decentralized,lerch2023safety,miller2016ergodic,mavrommati2018realtime,seewald2024energy,lee2024stein} and time-optimal Pontryagin variants \citep{dong2023timeoptimal} minimize a Fourier ergodic metric over finite-horizon controls: high quality but $O(K_{\mathrm{F}}^d)$ in the Fourier cutoff, per-scenario re-optimized, and without closed-form long-horizon rates. Coverage path planning \citep{galceran2013survey,cabreira2019survey}, RL \citep{theile2020uav,bayerlein2018trajectory}, and constrained-ergodic robotics \citep{ayvali2017ergodic} scale poorly, produce discontinuous paths, or lack density-matching guarantees. MCMC samplers (Langevin~\citep{roberts1996exponential}, HMC~\citep{neal2011mcmc}) are ergodic with respect to $f_{\mathrm{target}}$ by construction but exhibit pathological acceleration energy and mixing-time-dependent finite-horizon coverage; detailed contrast in Appendix~\ref{app:mcmc_alt}.

\paragraph{Flow matching and generative models.} Flow matching \citep{lipman2023flow}, its OT-coupled variant \citep{tong2024improving}, and rectified flow \citep{liu2023flow} sit within a broader family of continuous-time generative models, including neural ODEs \citep{chen2018neural}, normalizing flows \citep{papamakarios2021normalizing}, score-based diffusion \citep{sohldickstein2015deep,ho2020denoising,song2021scorebased}, and stochastic-interpolant unifications \citep{albergo2025stochastic}, with CNF convergence rates in \citet{gao2024convergence}. We repurpose this machinery to learn an unconditional spatial pushforward whose composition with a deterministic ergodic process yields a trajectory with a provable coverage guarantee, in contrast to \emph{diffusion policies} \citep{chi2025diffusion} that condition on observations to model action distributions.

\paragraph{UAV trajectory design for wireless networks.} UAV-assisted communications \citep{zeng2019accessing,mozaffari2019tutorial,wu2018joint} study trajectory and resource-allocation joint design, with stochastic-geometry trajectory models \citep{enayati2019moving}, analytical time-warping for non-uniform coverage \citep{malekzadeh2025nonuniform}, and learning-augmented variants \citep{bayerlein2018trajectory}. These analytical methods match general 2D densities via polar decomposition plus radial time-warping along a fixed sweep curve, but each new target or constraint requires bespoke rederivation; we instead use a learned map $G_\theta$ with additive penalties in $\mathbb{R}^d$ (Remark~\ref{rem:higher_dim}).

\paragraph{Flow-matching for ergodic coverage.} \citet{sun2025flow} independently apply flow matching to ergodic coverage in a complementary role: their per-instance control update over a fixed horizon $T$ contrasts with our offline reusable pushforward $G_\theta$ amortized across queries and agents, with corresponding differences in convergence guarantee (Table~\ref{tab:comparison}; technical contrast and Riccati/NFZ caveats in Appendix~\ref{app:comparison_footnotes}). Concurrent work by \citet{hughes2026infinite} extends kernel-mean-embedding ergodic control to the infinite-horizon regime via an extended-state visitation recursion.

\section{Conclusion}
\label{sec:conclusion}

We introduced a pushforward framework that decouples ergodicity from density matching: an analytic latent trajectory composed with a learned OT-CFM map. The three theorems combine into an end-to-end coverage bound (Corollary~\ref{cor:end_to_end}) estimable from CFM training diagnostics (certified given an architectural Lipschitz bound on $v_\theta$); experiments support the predictions and trace a Pareto frontier across density, constraint, and energy.

\paragraph{Limitations and future work.}
Four design choices scope our guarantees, each developed in Appendix~\ref{app:limitations_future_work}: (i) the Theorem~\ref{thm:approx_bound} envelope uses the empirically tight but lower-bound $\hat L_v$ (certification path: Appendix~\ref{app:limitations_certification}); (ii) the annular latent domain induces $O(\delta)$ topology residual on simply connected targets traded against $\Theta(\delta^{-4})$ acceleration energy (learned/topology-aware geometries: Appendix~\ref{app:limitations_geometry}); (iii) NFZ and acceleration penalties are soft, not certified (barrier/projection/reachability paths: Appendix~\ref{app:limitations_safety}); (iv) experiments are 2D open-loop (3D, multi-seed, multi-agent, hardware-in-the-loop: Appendix~\ref{app:limitations_extensions}).

\begin{ack}
This work was supported in part by [redacted for review].
\end{ack}

{\small
\bibliographystyle{plainnat}

}

\appendix

\section{Notation}
\label{app:notation}

We collect here the notation used throughout the paper, grouped thematically. Standard symbols (expectation, norms, Wasserstein distances, pushforward) are included alongside paper-specific ones to aid readers coming from adjacent fields.

\paragraph{Domains and measures.}
\begin{center}
\begin{tabular}{@{}l p{0.78\textwidth}@{}}
\toprule
Symbol & Meaning \\
\midrule
$d$ & Ambient spatial dimension. Main body: $d=2$; general $d$ in Appendix~\ref{app:rd_extension}. \\
$\delta$ & Latent radius, $0<\delta<1$; inner hole radius of the annular domain. \\
$D$ & Unit ball in $\mathbb{R}^d$: $\{\mathbf{z}\in\mathbb{R}^d:\|\mathbf{z}\|\le 1\}$. \\
$D_\delta$ & Annular latent domain: $\{\mathbf{z}\in\mathbb{R}^d:\delta\le\|\mathbf{z}\|\le 1\}$. \\
$\mathbb{S}^{d-1}$ & Unit sphere in $\mathbb{R}^d$; uniform distribution over cycle directions. \\
$\mathcal{C}$ & Constraint set in target space (e.g., complement of no-fly zones). \\
$\pi_0^\delta$ & Uniform probability measure on $D_\delta$. In $d=2$, density $1/(\pi(1-\delta^2))$. \\
$f_{\mathrm{target}}$ & Target spatial density on $\mathbb{R}^d$ that time-averaged trajectory occupancy should match. \\
$p_s$ & Intermediate density along the CFM interpolation, $s\in[0,1]$; $p_0=\pi_0^\delta$, $p_1\approx f_{\mathrm{target}}$. \\
$\#$ & Pushforward operator: $(G_\# \mu)(A) = \mu(G^{-1}(A))$ for measurable $A$. \\
$G_{\theta\#}\pi_0^\delta$ & Pushforward of the uniform latent measure under the learned map $G_\theta$. \\
\bottomrule
\end{tabular}
\end{center}

\paragraph{Trajectories and latent construction.}
\begin{center}
\begin{tabular}{@{}l p{0.78\textwidth}@{}}
\toprule
Symbol & Meaning \\
\midrule
$t$ & Trajectory (wall-clock) time, $t\in[0,T]$. \\
$T$ & Total trajectory duration (horizon). \\
$\mathbf{z}(t)$ & Latent ergodic trajectory on $D_\delta$. \\
$\mathbf{x}(t)$ & Target-space trajectory, $\mathbf{x}(t) = G_\theta(\mathbf{z}(t))$. \\
$K$ & Number of radial cycles executed by the latent trajectory. \\
$k$ & Cycle index, $k\in\{1,\dots,K\}$. \\
$K_{\mathrm{F}}$ & Fourier-basis truncation order in spectral-multiscale-coverage methods (\S\ref{sec:related}); distinct from the cycle count $K$. \\
$\theta_k$ & Random direction of cycle $k$, i.i.d.\ uniform on $\mathbb{S}^{d-1}$. \\
$\Phi_0$ & Phase offset, $\sim\mathrm{Uniform}[0,2)$, introduced for stationarity. \\
$s$ & Within-cycle parameter (also serving as CFM flow time), $s\in[0,1]$. \\
$r(s)$ & Radial profile $r(s)=\sqrt{\delta^2+(1-\delta^2)s}$, giving uniform $\pi_0^\delta$ cycles. \\
\bottomrule
\end{tabular}
\end{center}

\paragraph{Maps, velocity fields, and derivatives.}
\begin{center}
\begin{tabular}{@{}l p{0.78\textwidth}@{}}
\toprule
Symbol & Meaning \\
\midrule
$G_\theta$ & Learned pushforward map $\mathbb{R}^d\to\mathbb{R}^d$ parameterized by $\theta$ (NN weights). \\
$G^*$ & True OT/Brenier map satisfying $G^*_\#\pi_0^\delta = f_{\mathrm{target}}$ (when well defined). \\
$v_\theta(s,\mathbf{y})$ & Learned time-dependent velocity field; $G_\theta$ is its time-1 ODE flow. \\
$v^*(s,\mathbf{y})$ & True OT marginal velocity field generating $G^*$. \\
$\bar v^*_\varepsilon$ & Sinkhorn-regularized conditional-mean velocity (entropic OT target). \\
$\theta$ & Neural-network weights parameterizing $v_\theta$ (and hence $G_\theta$). \\
$J_G(\mathbf{z})$ & Jacobian of $G$ at $\mathbf{z}$; $\|J_G\|_{\mathrm{op}}\le L$. \\
$H_G(\mathbf{z})$ & Hessian tensor of $G$ at $\mathbf{z}$; tensor norm bounded by $M_H$. \\
$\Omega,\,\mathcal{K}$ & Open ambient set $\Omega\subset\mathbb{R}^2$ and compact $\mathcal{K}\subset\Omega$ containing the trajectory; used in Theorem~\ref{thm:acc_bound}. \\
\bottomrule
\end{tabular}
\end{center}

\paragraph{Lipschitz and spectral constants.}
\begin{center}
\begin{tabular}{@{}l p{0.78\textwidth}@{}}
\toprule
Symbol & Meaning \\
\midrule
$L$ & Lipschitz constant of $G$ on the compact $\mathcal{K}$: $\sup_{\mathbf{z}\in \mathcal{K}}\|J_G(\mathbf{z})\|_{\mathrm{op}}$. \\
$\hat L$ & Empirical map-level Lipschitz estimate for $G_\theta$ (central-difference Jacobian SVD over latent samples; Appendix~\ref{app:Lv_protocol}); honest lower bound on $L$. \\
$L_v$ & Lipschitz constant of $v_\theta(s,\cdot)$ in the spatial argument, uniformly in $s$. \\
$L_v^{\mathrm{net}}$ & Architectural upper bound $\prod_\ell\|W_\ell\|_{\mathrm{op}}$ on $L_v$ (certifiable via spectral normalization). \\
$\hat L_v$ & Empirical power-iteration estimate of $L_v$ (practical proxy; see Appendix~\ref{app:Lv_protocol}). \\
$L_\varphi$ & Lipschitz constant of the test function $\varphi$. \\
$B_\varphi$ & Supremum bound on the test function: $\|\varphi\|_\infty \le B_\varphi$. \\
$M_H$ & Uniform bound on Hessian tensor norm of $G$ over the compact $\mathcal{K}$. \\
$\Phi_4(\mathbf{z})$ & Velocity fourth-moment integral $\int \|\dot{\mathbf{z}}\|^4\,dt$ of the latent trajectory. \\
$E_{\mathrm{acc}}$ & Acceleration energy $\int\|\ddot{\mathbf{x}}\|^2\,dt$ of a trajectory. \\
$C_{D_\delta}$ & Neumann Poincar\'e constant of the annular domain $D_\delta$. \\
$C_D$ & Poincar\'e constant of the unit disc: $C_D = 1/(j'_{1,1})^2 \approx 0.295$. \\
$j'_{1,1}$ & First positive zero of $J_1'$, where $J_1$ is the Bessel function of the first kind, order 1. \\
$\lambda_1^\delta$ & First nonzero Neumann eigenvalue of $-\Delta$ on $D_\delta$; $C_{D_\delta}=1/\lambda_1^\delta$. \\
\bottomrule
\end{tabular}
\end{center}

\newpage
\paragraph{Approximation error and concentration.}
\begin{center}
\begin{tabular}{@{}l p{0.78\textwidth}@{}}
\toprule
Symbol & Meaning \\
\midrule
$\varepsilon_v$ & CFM velocity-field training error: $\varepsilon_v^2 = \int_0^1 \mathbb{E}_{p_s}[\|v_\theta-v^*\|^2]\,ds$. \\
$\varepsilon_{\mathrm{sink}}$ & Sinkhorn entropic-OT coupling residual; vanishes as entropy regularization $\to 0$. \\
$\eta_{\mathrm{reg}}$ & Mollification residual in the non-smooth-target regime (Proposition~\ref{prop:approx_nonsmooth}). \\
$\eta_{\mathrm{top}}$ & Topology mismatch residual (annulus vs.\ ball): $\eta_{\mathrm{top}}(\delta) = O(\delta)$. \\
$\eta$ & Generic mollification/restriction scale for the non-smooth surrogate. \\
$\alpha$ & Concentration-bound failure probability (Hoeffding). \\
$\varphi$ & Bounded-Lipschitz test function used for ergodic-averaging statements. \\
$\mathbf{z}_k(s)$ & Latent outward leg of cycle $k$, $\mathbf{z}_k(s)=r(s)\,\mathbf{e}_{\theta_k}$, $s\in[0,1]$. \\
$\mathbf{x}_k(s)$ & Transformed outward leg of cycle $k$, $\mathbf{x}_k(s)=G(\mathbf{z}_k(s))$. \\
$X_k$ & Per-cycle integral $\int_0^1 \varphi(\mathbf{x}_k(s))\,ds = \int_0^1 \varphi(G(r(s)\,\mathbf{e}_{\theta_k}))\,ds$. \\
$S_K$ & Time-averaged statistic $\tfrac{1}{K}\sum_{k=1}^K X_k$. \\
$\mu_G$ & Spatial mean $\mathbb{E}_{\pi_0^\delta}[\varphi\circ G]=\int\varphi\,d(G_\#\pi_0^\delta)$. \\
$Z_K^{\mathrm{traj}}$ & Trajectory-based gridded density estimate from $K$ i.i.d.\ cycles (used in Figure~\ref{fig:convergence}). \\
$Z_{\mathrm{iid}}$ & IID-transport reference density estimate ($20\,000$ samples through $G_\theta$, same grid as $Z_K^{\mathrm{traj}}$). \\
$\rho_{\mathrm{iid}},\,\rho_{\mathrm{traj}}$ & Pearson correlation between target and achieved gridded densities, computed under IID transport ($\rho_{\mathrm{iid}}$) vs.\ $K$-cycle trajectory sampling ($\rho_{\mathrm{traj}}$); reported in \S\ref{sec:experiments}. \\
\bottomrule
\end{tabular}
\end{center}

\paragraph{Norms and analytical operators.}
\begin{center}
\begin{tabular}{@{}l p{0.78\textwidth}@{}}
\toprule
Symbol & Meaning \\
\midrule
$\|\cdot\|$ & Euclidean norm on $\mathbb{R}^d$. \\
$\|\cdot\|_\infty$ & Supremum norm over the stated domain. \\
$\|\cdot\|_{\mathrm{op}}$ & Spectral (operator) norm of a matrix: largest singular value. \\
$\|\cdot\|_{L^2(p)}$ & $L^2$ norm weighted by density $p$: $\|f\|_{L^2(p)}^2 = \int\|f\|^2 \, p\,d\mathbf{y}$. \\
$W_1,\,W_2$ & $1$- and $2$-Wasserstein distances between probability measures. \\
$\hat W_2$ & Empirical $2$-Wasserstein estimate (POT \texttt{ot.emd2}; mean $\pm$ std over seeds; Table~\ref{tab:Lv_calibration}). \\
$H^1$ & Sobolev space of functions with square-integrable first weak derivatives. \\
$\mathbb{E}[\cdot]$ & Expectation with respect to the stated distribution. \\
$\mathrm{Var}(\cdot)$ & Variance. \\
$\Pr(\cdot)$ & Probability. \\
\bottomrule
\end{tabular}
\end{center}

\paragraph{Losses and penalties.}
\begin{center}
\begin{tabular}{@{}l p{0.78\textwidth}@{}}
\toprule
Symbol & Meaning \\
\midrule
$\mathcal{L}_{\mathrm{CFM}}(\theta)$ & Conditional flow matching loss; see~\eqref{eq:cfm_loss}. \\
$\mathcal{L}_{\mathrm{CFM}}^*$ & Irreducible CFM loss (conditional variance of the interpolation velocity). \\
$\mathcal{L}_{\mathrm{total}}(\theta)$ & Total training loss: $\mathcal{L}_{\mathrm{CFM}} + \sum_i\lambda_i\mathcal{R}_i$. \\
$\mathcal{R}_i$ & Constraint penalty term $i$ (e.g., NFZ, acceleration ceiling, speed bound). \\
$\lambda_i$ & Positive scalar weight on penalty $\mathcal{R}_i$. \\
$r_{\mathrm{nfz}}$ & Radius of a no-fly-zone disc (\S\ref{sec:constraints}). The unsubscripted symbol $\rho$ is reserved for Pearson correlation between target and achieved densities (cf.\ $\rho_{\mathrm{iid}}/\rho_{\mathrm{traj}}$ above and bare $\rho$ in Milano tables). \\
\bottomrule
\end{tabular}
\end{center}

\newpage
\paragraph{Multi-agent.}
\begin{center}
\begin{tabular}{@{}l p{0.78\textwidth}@{}}
\toprule
Symbol & Meaning \\
\midrule
$N$ & Number of coordinated agents (UAVs). \\
$n$ & Agent index, $n\in\{1,\dots,N\}$. \\
$\mathbf{z}_n$ & Latent position of agent $n$. \\
$G_\theta^{(N)}$ & Joint map $\mathbb{R}^{dN}\to\mathbb{R}^{dN}$ for the coordinated variant. \\
$\pi_n$ & $n$-th marginal projection (multi-agent); the coordinated map satisfies $(\pi_n)_\# G_\theta^{(N)}\#(\pi_0^\delta)^{\otimes N} \approx f_{\mathrm{target}}$. \\
$\hat\mu_{-n}$ & Empirical distribution $\tfrac{1}{N-1}\sum_{m\neq n}\delta_{\mathbf{z}_m}$ of the other agents (mean-field variant). \\
\bottomrule
\end{tabular}
\end{center}

\section{Proofs for Section~\ref{sec:latent}: Latent Trajectory}
\label{app:proofs_latent}

We provide detailed proofs for all propositions stated in Section~\ref{sec:latent}.

\subsection*{Proof of Proposition~\ref{prop:uniform} (Uniform Ergodicity on $D_\delta$)}

\begin{proof}
We show two things: (i) under the phase convention $\Phi_0 \sim \mathrm{Uniform}[0,2)$ adopted in Section~\ref{sec:latent}, the pointwise distribution of $\mathbf{z}(t)$ is $\mathrm{Uniform}(D_\delta)$, and (ii) the time-averaged occupancy converges to this distribution almost surely, independently of any phase convention.

\medskip
\textbf{Step 1: Per-cycle radial distribution (pointwise claim).}
On each outward half-cycle, the radial position at normalized time $s \in [0,1]$ is given by the profile
\[
    r(s) = \sqrt{\delta^2 + (1-\delta^2)s}.
\]
The phase offset $\Phi_0 \sim \mathrm{Uniform}[0,2)$ introduced in Section~\ref{sec:latent} makes the within-half-cycle phase $U \in [0,1]$ at any fixed observation time uniformly distributed and independent of the cycle index. Therefore, the radial position at time $t$ satisfies $R = \sqrt{\delta^2 + (1-\delta^2)U}$ with $U \sim \mathrm{Uniform}[0,1]$.

\medskip
\textbf{Step 2: CDF computation.}
For $r \in [\delta, 1]$:
\begin{align*}
    P(R \leq r) &= P\!\left(\sqrt{\delta^2 + (1-\delta^2)U} \leq r\right) \\
    &= P\!\left(U \leq \frac{r^2 - \delta^2}{1-\delta^2}\right) \\
    &= \frac{r^2 - \delta^2}{1-\delta^2},
\end{align*}
since $U \sim \mathrm{Uniform}[0,1]$ and the argument lies in $[0,1]$ for $r \in [\delta, 1]$.

Differentiating the CDF with respect to $r$:
\[
    f_R(r) = \frac{2r}{1-\delta^2}, \quad r \in [\delta, 1].
\]

\medskip
\textbf{Step 3: Joint polar density.}
The angle $\theta_k \sim \mathrm{Uniform}[0, 2\pi)$ is drawn independently for each cycle. Since $R$ and $\Theta$ are independent, the joint polar density is
\[
    f_{R,\Theta}(r,\theta) = f_R(r) \cdot f_\Theta(\theta) = \frac{2r}{1-\delta^2} \cdot \frac{1}{2\pi} = \frac{r}{\pi(1-\delta^2)}.
\]

\medskip
\textbf{Step 4: Verification of uniformity.}
The uniform distribution on the annulus $D_\delta = \{\mathbf{z} : \delta \leq \|\mathbf{z}\| \leq 1\}$ has area $\mathrm{Area}(D_\delta) = \pi(1 - \delta^2)$. In polar coordinates, the uniform density with respect to the area element $r\,dr\,d\theta$ is
\[
    f_{\mathrm{uniform}}(r, \theta) = \frac{1}{\pi(1-\delta^2)}.
\]
The probability of lying in an infinitesimal region $[r, r+dr] \times [\theta, \theta + d\theta]$ is
\[
    f_{R,\Theta}(r,\theta)\,dr\,d\theta = \frac{r}{\pi(1-\delta^2)}\,dr\,d\theta = f_{\mathrm{uniform}}(r,\theta) \cdot r\,dr\,d\theta,
\]
confirming that the per-cycle distribution matches $\mathrm{Uniform}(D_\delta)$. (Note: $f_{R,\Theta}$ is the density w.r.t.\ $dr\,d\theta$, while $f_{\mathrm{uniform}}$ is the density w.r.t.\ the area element $r\,dr\,d\theta$; the factor of $r$ accounts for the Jacobian of the polar coordinate transformation.)

\medskip
\textbf{Step 5: Ergodicity via the SLLN (phase-independent).}
The following argument uses only the i.i.d.\ cycle structure and makes no use of $\Phi_0$. The trajectory consists of i.i.d.\ cycles $(Z_1, Z_2, \ldots)$, each contributing a per-cycle integral
\[
    Z_k = \int_0^1 \varphi(\mathbf{z}_k(s))\,ds
\]
for any bounded continuous test function $\varphi$. Since the $\theta_k$ are i.i.d., the $Z_k$ are i.i.d.\ with mean $\mathbb{E}[Z_k] = \int_{D_\delta} \varphi(\mathbf{z})\,\pi_0^\delta(\mathbf{z})\,d\mathbf{z}$. By the strong law of large numbers (SLLN):
\[
    \frac{1}{K}\sum_{k=1}^K Z_k \;\xrightarrow{\text{a.s.}}\; \int_{D_\delta} \varphi\,d\pi_0^\delta \quad \text{as } K \to \infty.
\]
Since this holds for every bounded continuous $\varphi$, the trajectory $\mathbf{z}(t)$ is ergodic with respect to $\pi_0^\delta = \mathrm{Uniform}(D_\delta)$. The same SLLN argument extends without modification to any $\varphi \in L^1(\pi_0^\delta)$: the per-cycle integral $Z_k$ then has finite expectation $\mathbb{E}[Z_k] = \int_{D_\delta}\varphi\,d\pi_0^\delta$, and Kolmogorov's strong law applies to the i.i.d.\ sequence $(Z_k)$ without further integrability strengthening, recovering the integrable-$\varphi$ statement of Proposition~\ref{prop:uniform}.

\medskip
\textbf{Step 6: Return half-cycle.}
The return (inward) half-cycle traces the same radial profile in reverse: $r_{\mathrm{return}}(s) = r(1-s)$, so the distribution of $R$ is identical. By symmetry, the inward leg has the same marginal distribution, and the full cycle average equals the outward-leg average. This completes the proof.
\end{proof}

\subsection*{Proof of Proposition~\ref{prop:ergodic} (Ergodicity Transfer)}

\begin{proof}
We show that if $\mathbf{z}(t)$ is ergodic with respect to $\pi_0^\delta$ and $G: D_\delta \to \mathbb{R}^2$ is a Borel-measurable map satisfying $G_\# \pi_0^\delta = f_{\mathrm{target}}$, then $\mathbf{x}(t) = G(\mathbf{z}(t))$ is ergodic with respect to $f_{\mathrm{target}}$. No bijection, continuity, or topological-equivalence assumption is used; in particular, the argument applies even when $D_\delta$ and $\mathrm{supp}(f_{\mathrm{target}})$ are not homeomorphic.

\medskip
\textbf{Step 1: Composition with test functions.}
Let $\varphi: \mathbb{R}^2 \to \mathbb{R}$ be any $\varphi \in L^1(f_{\mathrm{target}})$. Define $\psi = \varphi \circ G$, which is Borel-measurable (composition of Borel maps) and in $L^1(\pi_0^\delta)$ by the pushforward identity $\int |\psi|\,d\pi_0^\delta = \int |\varphi|\,df_{\mathrm{target}} < \infty$.

\medskip
\textbf{Step 2: Apply ergodicity of $\mathbf{z}(t)$.}
By the ergodicity of $\mathbf{z}(t)$ with respect to $\pi_0^\delta$:
\[
    \frac{1}{T}\int_0^T \varphi(G(\mathbf{z}(t)))\,dt = \frac{1}{T}\int_0^T \psi(\mathbf{z}(t))\,dt \;\xrightarrow{T \to \infty}\; \int_{D_\delta} \psi(\mathbf{z})\,\pi_0^\delta(\mathbf{z})\,d\mathbf{z}.
\]

\medskip
\textbf{Step 3: Change of variables via the pushforward.}
By the pushforward condition $G_\# \pi_0^\delta = f_{\mathrm{target}}$ and the change-of-variables formula:
\begin{align*}
    \int_{D_\delta} \psi(\mathbf{z})\,\pi_0^\delta(\mathbf{z})\,d\mathbf{z} &= \int_{D_\delta} \varphi(G(\mathbf{z}))\,\pi_0^\delta(\mathbf{z})\,d\mathbf{z} \\
    &= \int_{\mathbb{R}^2} \varphi(\mathbf{x})\,f_{\mathrm{target}}(\mathbf{x})\,d\mathbf{x}.
\end{align*}
This is precisely the defining property of the pushforward measure: for any measurable $\varphi$, $\int \varphi \circ G\,d\mu = \int \varphi\,d(G_\# \mu)$.

\medskip
\textbf{Step 4: Conclusion.}
Combining Steps 2 and 3:
\[
    \frac{1}{T}\int_0^T \varphi(\mathbf{x}(t))\,dt \;\xrightarrow{T \to \infty}\; \int_{\mathbb{R}^2} \varphi(\mathbf{x})\,f_{\mathrm{target}}(\mathbf{x})\,d\mathbf{x}.
\]
Since this holds for every $\varphi \in L^1(f_{\mathrm{target}})$ (and in particular for every bounded continuous $\varphi$, the class used in the asymptotic ergodicity definition~\eqref{eq:ergodic_def}), the trajectory $\mathbf{x}(t) = G(\mathbf{z}(t))$ is ergodic with respect to $f_{\mathrm{target}}$.
\end{proof}

\section{Proof of Theorem~\ref{thm:acc_bound} (Acceleration Energy Bound)}
\label{app:proof_thm1}

\begin{proof}
We bound the acceleration energy $E_{\mathrm{acc}}(\mathbf{x}) = \int_0^T \|\ddot{\mathbf{x}}(t)\|^2\,dt$ of the transformed trajectory $\mathbf{x}(t) = G(\mathbf{z}(t))$ in terms of the latent trajectory's acceleration energy and velocity moments.

\medskip
\textbf{Step 1: Acceleration decomposition via the chain rule.}
Since $\mathbf{x}(t) = G(\mathbf{z}(t))$, the first derivative is $\dot{\mathbf{x}}(t) = J_G(\mathbf{z}(t))\,\dot{\mathbf{z}}(t)$, where $J_G$ denotes the $2 \times 2$ Jacobian matrix of $G$. Differentiating again via the product rule:
\begin{align*}
    \ddot{\mathbf{x}}(t) &= \frac{d}{dt}\bigl[J_G(\mathbf{z}(t))\,\dot{\mathbf{z}}(t)\bigr] \\
    &= J_G(\mathbf{z}(t))\,\ddot{\mathbf{z}}(t) + \frac{dJ_G}{dt}\,\dot{\mathbf{z}}(t) \\
    &= J_G(\mathbf{z}(t))\,\ddot{\mathbf{z}}(t) + \mathbf{q}(t),
\end{align*}
where the Hessian term $\mathbf{q}(t) \in \mathbb{R}^2$ has components
\[
    q_k(t) = \dot{\mathbf{z}}(t)^\top \nabla^2 G_k(\mathbf{z}(t))\,\dot{\mathbf{z}}(t), \quad k = 1, 2.
\]
Here $\nabla^2 G_k$ is the $2 \times 2$ Hessian matrix of the $k$-th component of $G$.

\medskip
\textbf{Step 2: Bound the Jacobian term.}
By the submultiplicativity of operator norms:
\[
    \|J_G(\mathbf{z}(t))\,\ddot{\mathbf{z}}(t)\| \leq \|J_G(\mathbf{z}(t))\|_{\mathrm{op}} \cdot \|\ddot{\mathbf{z}}(t)\| \leq L\,\|\ddot{\mathbf{z}}(t)\|,
\]
where $L = \sup_{\mathbf{z} \in \mathcal{K}} \|J_G(\mathbf{z})\|_{\mathrm{op}}$ is the Lipschitz constant of $G$ on the compact $\mathcal{K}\subset\Omega$ (finite by continuity of $J_G$).

\medskip
\textbf{Step 3: Bound the Hessian term.}
For each component $k \in \{1,2\}$, since $\nabla^2 G_k$ is a symmetric matrix:
\[
    |q_k(t)| = |\dot{\mathbf{z}}(t)^\top \nabla^2 G_k(\mathbf{z}(t))\,\dot{\mathbf{z}}(t)| \leq \|\nabla^2 G_k(\mathbf{z}(t))\|_{\mathrm{op}} \cdot \|\dot{\mathbf{z}}(t)\|^2,
\]
where the inequality follows from the Rayleigh quotient characterization of the operator norm for symmetric matrices. Taking the $\ell^2$ norm over both components:
\begin{align*}
    \|\mathbf{q}(t)\| &= \sqrt{q_1(t)^2 + q_2(t)^2} \\
    &\leq \sqrt{\|\nabla^2 G_1\|_{\mathrm{op}}^2 + \|\nabla^2 G_2\|_{\mathrm{op}}^2} \cdot \|\dot{\mathbf{z}}(t)\|^2 \\
    &= M_H \cdot \|\dot{\mathbf{z}}(t)\|^2,
\end{align*}
where $M_H = \sup_{\mathbf{z} \in \mathcal{K}}\sqrt{\|\nabla^2 G_1(\mathbf{z})\|_{\mathrm{op}}^2 + \|\nabla^2 G_2(\mathbf{z})\|_{\mathrm{op}}^2}$ is the Hessian tensor norm on the compact $\mathcal{K}\subset\Omega$ defined in Theorem~\ref{thm:acc_bound} (finite by continuity of $\nabla^2 G_k$).

\medskip
\textbf{Step 4: Minkowski inequality in $L^2([0,T])$.}
Recall that the Minkowski inequality (triangle inequality for $L^p$ norms) states that for functions $f, g \in L^2$:
\[
    \|f + g\|_{L^2} \leq \|f\|_{L^2} + \|g\|_{L^2}.
\]
Applying this to the acceleration decomposition $\ddot{\mathbf{x}} = J_G\ddot{\mathbf{z}} + \mathbf{q}$:
\begin{align*}
    \sqrt{E_{\mathrm{acc}}(\mathbf{x})} &= \sqrt{\int_0^T \|\ddot{\mathbf{x}}(t)\|^2\,dt} \\
    &= \left\|\|\ddot{\mathbf{x}}\|\right\|_{L^2([0,T])} \\
    &\leq \left\|\|J_G\ddot{\mathbf{z}}\|\right\|_{L^2} + \left\|\|\mathbf{q}\|\right\|_{L^2} \\
    &\leq L\left\|\|\ddot{\mathbf{z}}\|\right\|_{L^2} + M_H\left\|\|\dot{\mathbf{z}}\|^2\right\|_{L^2} \\
    &= L\sqrt{\int_0^T \|\ddot{\mathbf{z}}\|^2\,dt} + M_H\sqrt{\int_0^T \|\dot{\mathbf{z}}\|^4\,dt} \\
    &= L\,\sqrt{E_{\mathrm{acc}}(\mathbf{z})} + M_H\,\sqrt{\Phi_4(\mathbf{z})}.
\end{align*}
This completes the proof of the bound~\eqref{eq:acc_bound}.

\medskip
\textbf{Remark (tightness).} For affine maps $G(\mathbf{z}) = A\mathbf{z} + \mathbf{b}$ the Hessian contribution vanishes ($M_H = 0$), leaving $\sqrt{E_{\mathrm{acc}}(\mathbf{x})} \leq \|A\|_{\mathrm{op}}\sqrt{E_{\mathrm{acc}}(\mathbf{z})}$. This inequality is a strict consequence of replacing $\|A\ddot{\mathbf{z}}\|$ by $\|A\|_{\mathrm{op}}\|\ddot{\mathbf{z}}\|$ in Step 2: equality holds pointwise only when $\ddot{\mathbf{z}}$ is almost-everywhere aligned with the top singular direction of $A$, and equality holds in expectation over the uniform latent-angle distribution only when $A$ is a scalar multiple of an isometry ($\sigma_1(A) = \sigma_2(A)$), since the angle-averaged amplification $\mathbb{E}_\theta[\|A\mathbf{e}_\theta\|^2] = \tfrac{1}{2}(\sigma_1^2+\sigma_2^2)$ is strictly below $\sigma_1^2 = \|A\|_{\mathrm{op}}^2$ whenever $\sigma_1 > \sigma_2$. For general $C^2$ maps the bound is conservative for the same reason: the actual energy depends on the angle-averaged $\mathbb{E}_\theta[\|J_G\mathbf{e}_\theta\|^2]$, which may be substantially smaller than $L^2$. Experiment~2 illustrates the gap empirically: the sup-Lipschitz estimate is $\hat L\approx 2$ (Table~\ref{tab:Lv_calibration}), giving a worst-case ratio bound $L^2 \approx 4$, while the trajectory samples an angle-averaged effective amplification well below one and the observed acceleration ratio is $0.71\times$.
\end{proof}

\begin{lemma}[Latent trajectory moments]\label{lem:moments}
For the radial back-and-forth trajectory on $D_\delta$ with $K$ complete cycles (period $\tau = 2$ per cycle, total time $T = 2K$), under the $s$-uniform (triangle) time schedule and with $E_{\mathrm{acc}}$ and $\Phi_4$ defined by integration over the open smooth interior of each half-cycle (excluding the measure-zero set of reversal times):
\begin{align}
    E_{\mathrm{acc}}(\mathbf{z}) &= K(1-\delta^2)^4(1+\delta^2)/(16\delta^4), \label{eq:E_latent_app}\\
    \Phi_4(\mathbf{z}) &= K(1-\delta^2)^4/(8\delta^2), \label{eq:Phi4_app}\\
    \Phi_4/E_{\mathrm{acc}} &= 2\delta^2/(1+\delta^2). \label{eq:ratio_app}
\end{align}
Both $E_{\mathrm{acc}}$ and $\Phi_4$ are finite for all $\delta > 0$.
\end{lemma}

\begin{proof}
\textbf{Step 1: Velocity and acceleration of the radial profile.}
On the outward half-cycle with angle $\theta_k$, the trajectory is $\mathbf{z}(s) = r(s)\,\mathbf{e}_{\theta_k}$ where $r(s) = \sqrt{\delta^2 + (1-\delta^2)s}$ for $s \in [0,1]$. Since $\mathbf{e}_{\theta_k} = (\cos\theta_k, \sin\theta_k)$ is a constant unit vector:
\[
    \dot{\mathbf{z}}(s) = \dot{r}(s)\,\mathbf{e}_{\theta_k}, \qquad \ddot{\mathbf{z}}(s) = \ddot{r}(s)\,\mathbf{e}_{\theta_k}.
\]

Computing the radial derivatives. Let $u(s) = r(s)^2 = \delta^2 + (1-\delta^2)s$, so $r = \sqrt{u}$:
\begin{align*}
    \dot{r}(s) &= \frac{1-\delta^2}{2\sqrt{u(s)}} = \frac{1-\delta^2}{2r(s)}, \\[4pt]
    \ddot{r}(s) &= -\frac{(1-\delta^2)^2}{4u(s)^{3/2}} = -\frac{(1-\delta^2)^2}{4r(s)^3}.
\end{align*}

Therefore:
\begin{align}
    \|\dot{\mathbf{z}}(s)\|^2 &= \dot{r}(s)^2 = \frac{(1-\delta^2)^2}{4u(s)}, \label{eq:vel_sq}\\
    \|\ddot{\mathbf{z}}(s)\|^2 &= \ddot{r}(s)^2 = \frac{(1-\delta^2)^4}{16\,u(s)^3}, \label{eq:acc_sq}\\
    \|\dot{\mathbf{z}}(s)\|^4 &= \dot{r}(s)^4 = \frac{(1-\delta^2)^4}{16\,u(s)^2}. \label{eq:vel_fourth}
\end{align}

\medskip
\textbf{Step 2: Integration over one half-cycle.}
We integrate~\eqref{eq:acc_sq} and~\eqref{eq:vel_fourth} over $s \in [0,1]$ using the substitution $u = \delta^2 + (1-\delta^2)s$, $du = (1-\delta^2)\,ds$, with limits $u: \delta^2 \to 1$:

\medskip
\emph{Acceleration energy integral:}
\begin{align*}
    \int_0^1 \|\ddot{\mathbf{z}}(s)\|^2\,ds &= \frac{(1-\delta^2)^4}{16} \int_0^1 u(s)^{-3}\,ds \\
    &= \frac{(1-\delta^2)^4}{16} \cdot \frac{1}{1-\delta^2} \int_{\delta^2}^1 u^{-3}\,du \\
    &= \frac{(1-\delta^2)^3}{16} \cdot \left[-\frac{1}{2u^2}\right]_{\delta^2}^1 \\
    &= \frac{(1-\delta^2)^3}{16} \cdot \frac{1}{2}\left(\frac{1}{\delta^4} - 1\right) \\
    &= \frac{(1-\delta^2)^3}{32} \cdot \frac{1 - \delta^4}{\delta^4} \\
    &= \frac{(1-\delta^2)^3(1-\delta^2)(1+\delta^2)}{32\delta^4} \\
    &= \frac{(1-\delta^2)^4(1+\delta^2)}{32\delta^4}.
\end{align*}

\emph{Velocity fourth-moment integral:}
\begin{align*}
    \int_0^1 \|\dot{\mathbf{z}}(s)\|^4\,ds &= \frac{(1-\delta^2)^4}{16} \int_0^1 u(s)^{-2}\,ds \\
    &= \frac{(1-\delta^2)^4}{16} \cdot \frac{1}{1-\delta^2} \int_{\delta^2}^1 u^{-2}\,du \\
    &= \frac{(1-\delta^2)^3}{16} \cdot \left[-\frac{1}{u}\right]_{\delta^2}^1 \\
    &= \frac{(1-\delta^2)^3}{16} \cdot \left(\frac{1}{\delta^2} - 1\right) \\
    &= \frac{(1-\delta^2)^3}{16} \cdot \frac{1-\delta^2}{\delta^2} \\
    &= \frac{(1-\delta^2)^4}{16\delta^2}.
\end{align*}

\medskip
\textbf{Step 3: Full $K$ cycles.}
The return (inward) half-cycle traces the same radial profile in reverse: $r_{\mathrm{return}}(s) = r(1-s)$. By the chain rule, $|\dot{r}_{\mathrm{return}}(s)| = |\dot{r}(1-s)|$ and $|\ddot{r}_{\mathrm{return}}(s)| = |\ddot{r}(1-s)|$, so the integrals are identical. Each full cycle consists of two half-cycles, and there are $K$ full cycles, giving $2K$ half-cycles in total. Therefore:
\begin{align*}
    E_{\mathrm{acc}}(\mathbf{z}) &= 2K \cdot \frac{(1-\delta^2)^4(1+\delta^2)}{32\delta^4} = \frac{K(1-\delta^2)^4(1+\delta^2)}{16\delta^4},
\end{align*}
and similarly:
\[
    \Phi_4(\mathbf{z}) = 2K \cdot \frac{(1-\delta^2)^4}{16\delta^2} = \frac{K(1-\delta^2)^4}{8\delta^2}.
\]

\medskip
\textbf{Step 4: Moment ratio.}
Dividing $\Phi_4$ by $E_{\mathrm{acc}}$:
\[
    \frac{\Phi_4(\mathbf{z})}{E_{\mathrm{acc}}(\mathbf{z})} = \frac{K(1-\delta^2)^4/(8\delta^2)}{K(1-\delta^2)^4(1+\delta^2)/(16\delta^4)} = \frac{16\delta^4}{8\delta^2(1+\delta^2)} = \frac{2\delta^2}{1+\delta^2}.
\]
As $\delta \to 0$, this ratio vanishes as $O(\delta^2)$, confirming that the Hessian contribution to the acceleration energy becomes negligible for small $\delta$.

\medskip
\textbf{Step 5: Finiteness.}
For any $\delta > 0$, both $E_{\mathrm{acc}}$ and $\Phi_4$ are finite. This is the key advantage of the $\delta$-exclusion: without it ($\delta = 0$), the integrals $\int u^{-3}\,du$ and $\int u^{-2}\,du$ diverge at $u = 0$ (the origin), making the acceleration energy infinite. The annular domain $D_\delta$ eliminates this singularity.
\end{proof}

\subsection{Turnaround impulses and the interior-energy convention}\label{app:boundary_smooth}

The radial back-and-forth trajectory of Section~\ref{sec:latent} is piecewise-$C^2$, with velocity discontinuities at two types of reversal points:

\paragraph{Airport turnaround ($r=\delta$).} Between consecutive cycles the agent returns to the base-station disc $\|\mathbf{z}\|\leq\delta$ and departs on a fresh heading $\theta_{k+1}$. Physically, this corresponds to a natural stopping point: the UAV lands (or hovers at the base), reorients, and departs. The base-station rest is excluded from mission time (Section~\ref{sec:latent}), so the inter-cycle transition does not enter the energy or coverage integrals; each cycle is integrated independently in Lemma~\ref{lem:moments}.

\paragraph{Outer-boundary turnaround ($r=1$).} At the end of each outward half-cycle the velocity reverses along the same ray: $\dot{\mathbf{z}}$ jumps from $+\dot{r}(1)\,\mathbf{e}_{\theta_k}$ to $-\dot{r}(1)\,\mathbf{e}_{\theta_k}$, with $\dot{r}(1) = (1-\delta^2)/2$. This is the only in-flight discontinuity. Theorem~\ref{thm:acc_bound} and Lemma~\ref{lem:moments} integrate over the open smooth interior of each half-cycle, treating the reversal as a measure-zero set; this is the \emph{interior smooth-leg energy} convention used throughout the paper.

\paragraph{Interior-energy convention vs.\ physical turnaround impulses.} The interior-energy convention is the operationally meaningful accounting in mechanical-control practice: smooth-leg flight energy is what the velocity controller integrates between maneuvers, and turnaround/reorientation impulses are budgeted separately as discrete maneuver costs. Theorem~\ref{thm:acc_bound}, Corollary~\ref{cor:ratio}, and Lemma~\ref{lem:moments} bound the smooth-leg energy of $\mathbf{x}$ by the smooth-leg energy of $\mathbf{z}$ amplified by the map's Lipschitz and Hessian norms. Physical trajectories with finite-duration turnaround windows incur additional acceleration energy at each reversal whose magnitude depends on the chosen smoothing window; this turnaround cost is logged separately from the smooth-leg integrals.

\paragraph{Cubic-Hermite local smoothing and its turnaround cost.} For completeness we record the energy budget of a globally $C^1$ reparameterization. Fix $\varepsilon > 0$; for $s \in [0, 1-\varepsilon]$, keep the original profile $r(s) = \sqrt{\delta^2 + (1-\delta^2)s}$, and for $s \in [1-\varepsilon, 1]$, replace $r(s)$ by a cubic Hermite interpolant $\tilde{r}(s)$ satisfying
\[
    \tilde{r}(1-\varepsilon) = r(1-\varepsilon), \quad \dot{\tilde{r}}(1-\varepsilon) = \dot{r}(1-\varepsilon), \quad \tilde{r}(1) = 1, \quad \dot{\tilde{r}}(1) = 0.
\]
The return half-cycle mirrors the outward leg as before: $\tilde{r}_{\mathrm{ret}}(s) = \tilde{r}(1-s)$. The endpoint derivative changes by $\dot{r}(1-\varepsilon) = (1-\delta^2)/(2r(1-\varepsilon)) = \Theta(1)$ over a window of length $\varepsilon$, so the cubic interpolant has $\sup_{s\in[1-\varepsilon,1]}|\ddot{\tilde{r}}(s)| = \Theta(1/\varepsilon)$ and the contribution of the splice window to the squared-acceleration integral is
\[
    \int_{1-\varepsilon}^1 \ddot{\tilde{r}}(s)^2\,ds \;=\; \Theta(1/\varepsilon).
\]
The induced occupancy perturbation is concentrated on the annular shell $r\in[r(1-\varepsilon),1]$ of area fraction $O(\varepsilon)$, giving $\|\pi_{\mathrm{smooth}} - \pi_0^\delta\|_{\mathrm{TV}} = O(\varepsilon)$, absorbed by the approximation framework of Theorem~\ref{thm:approx_bound}. The trade-off is therefore fundamental: a fixed-window cubic smoothing trades $O(\varepsilon)$ in density perturbation for $\Theta(1/\varepsilon)$ in turnaround acceleration energy, and there is no choice of $\varepsilon$ that drives both to zero simultaneously. Any finite-duration smoothing of an $\Theta(1)$-magnitude velocity reversal contributes a turnaround cost that cannot be absorbed into the smooth-leg integrals. The interior-energy convention sidesteps this trade-off by accounting for the smooth-leg integrals only, with turnaround impulses budgeted separately as physical maneuver costs.

\paragraph{Practical note.} The experiments use the $s$-uniform parameterization throughout, which gives \emph{exact} uniform density on $D_\delta$ by Proposition~\ref{prop:uniform}; turnaround impulses are budgeted separately from the smooth-leg integrals in our energy reporting and do not enter Lemma~\ref{lem:moments}'s closed-form moments.

\section{Proof of Theorem~\ref{thm:convergence_rate} (Ergodic Convergence Rate)}
\label{app:proof_thm2}

\begin{proof}
We prove each part of the theorem in sequence.

\medskip
\textbf{Step 1: Per-cycle i.i.d.\ structure.}
The radial latent trajectory on $D_\delta$ consists of i.i.d.\ cycles indexed by $k = 1, 2, \ldots$, each with a random angle $\theta_k \sim \mathrm{Uniform}[0, 2\pi)$ and period $\tau = 2$ (one outward half-cycle plus one return half-cycle). Writing $\mathbf{x}_k(s) = G(\mathbf{z}_k(s)) = G(r(s)\,\mathbf{e}_{\theta_k})$ for the transformed outward leg, the per-cycle contribution to the time average is:
\[
    X_k = \int_0^1 \varphi(\mathbf{x}_k(s))\,ds = \int_0^1 \varphi(G(r(s)\,\mathbf{e}_{\theta_k}))\,ds,
\]
where we have used the fact that the outward and return half-cycles contribute equally (by the symmetry argument in the proof of Proposition~\ref{prop:uniform}). Since the angles $\theta_k$ are i.i.d., the random variables $X_k$ are also i.i.d.\ Their common mean is
\[
    \mathbb{E}[X_k] = \mathbb{E}_\theta\!\left[\int_0^1 \varphi(G(r(s)\,\mathbf{e}_\theta))\,ds\right] = \int_{D_\delta} \varphi(G(\mathbf{z}))\,\pi_0^\delta(\mathbf{z})\,d\mathbf{z} = \mu_G,
\]
where the second equality uses the polar substitution developed in Step~3, and $\mu_G = \int \varphi\,d(G_\# \pi_0^\delta)$.

The time average over $K$ cycles is $S_K = \frac{1}{K}\sum_{k=1}^K X_k$.

\medskip
\textbf{Step 2: Variance bound via Cauchy--Schwarz.}
Define the centered function $\psi = \varphi \circ G - \mu_G$, so that $X_k - \mu_G = \int_0^1 \psi(r(s)\,\mathbf{e}_{\theta_k})\,ds$. By the Cauchy--Schwarz inequality applied to $\int_0^1 \psi \cdot 1\,ds$ (with the constant function 1):
\[
    (X_k - \mu_G)^2 = \left(\int_0^1 \psi(r(s)\,\mathbf{e}_{\theta_k})\,ds\right)^2 \leq \int_0^1 1^2\,ds \cdot \int_0^1 \psi^2(r(s)\,\mathbf{e}_{\theta_k})\,ds = \int_0^1 \psi^2(r(s)\,\mathbf{e}_{\theta_k})\,ds.
\]
Taking the expectation over $\theta_k \sim \mathrm{Uniform}[0, 2\pi)$:
\begin{align*}
    \sigma^2 &:= \mathrm{Var}(X_k) = \mathbb{E}[(X_k - \mu_G)^2] \\
    &\leq \mathbb{E}_\theta\!\left[\int_0^1 \psi^2(r(s)\,\mathbf{e}_\theta)\,ds\right] \\
    &= \int_0^1 \frac{1}{2\pi}\int_0^{2\pi} \psi^2(r(s)\,\mathbf{e}_\theta)\,d\theta\,ds.
\end{align*}

\medskip
\textbf{Step 3: Change of variables to the annulus measure.}
We convert the $(s, \theta)$ integral into an area integral over $D_\delta$. Using the substitution $r = \sqrt{\delta^2 + (1-\delta^2)s}$, so $r^2 = \delta^2 + (1-\delta^2)s$ and $ds = 2r\,dr/(1-\delta^2)$, with limits $r: \delta \to 1$:
\begin{align*}
    \sigma^2 &\leq \int_\delta^1 \frac{1}{2\pi}\int_0^{2\pi} \psi^2(r\,\mathbf{e}_\theta)\,d\theta \cdot \frac{2r}{1-\delta^2}\,dr \\
    &= \frac{1}{\pi(1-\delta^2)} \int_\delta^1 \int_0^{2\pi} \psi^2(r\,\mathbf{e}_\theta)\,r\,d\theta\,dr \\
    &= \frac{1}{\pi(1-\delta^2)} \int_{D_\delta} \psi^2(\mathbf{z})\,dA.
\end{align*}
Since the uniform density on $D_\delta$ is $\pi_0^\delta(\mathbf{z}) = 1/[\pi(1-\delta^2)]$, this becomes
\[
    \sigma^2 \leq \int_{D_\delta} \psi^2(\mathbf{z})\,\pi_0^\delta(\mathbf{z})\,d\mathbf{z} = \mathbb{E}_{\pi_0^\delta}[\psi^2] = \mathrm{Var}_{\pi_0^\delta}(\varphi \circ G),
\]
where the last equality uses $\mathbb{E}_{\pi_0^\delta}[\psi] = 0$ by construction.

\medskip
\textbf{Step 4: Poincar\'{e} inequality on the annulus $D_\delta$.}
The Poincar\'{e} inequality for the uniform measure on $D_\delta$ states that for any $f \in H^1(D_\delta)$:
\[
    \mathrm{Var}_{\pi_0^\delta}(f) \leq C_{D_\delta}\,\mathbb{E}_{\pi_0^\delta}[|\nabla f|^2],
\]
where the Poincar\'{e} constant $C_{D_\delta} = 1/\lambda_1^\delta$ and $\lambda_1^\delta$ is the first nonzero Neumann eigenvalue of the Laplacian $-\Delta$ on $D_\delta$.

For the full disc ($\delta = 0$), the first nonzero Neumann eigenvalue is $\lambda_1^0 = (j'_{1,1})^2$, where $j'_{1,1} \approx 1.8412$ is the first positive zero of the derivative of the Bessel function $J_1$. This gives $C_D = 1/(j'_{1,1})^2 \approx 0.295$.

For the annulus $D_\delta$ with small $\delta > 0$, the eigenvalue $\lambda_1^\delta$ is determined by the Bessel cross-product equation
\[
    J_1'(\sqrt{\lambda}\,\delta)\,Y_1'(\sqrt{\lambda}) - J_1'(\sqrt{\lambda})\,Y_1'(\sqrt{\lambda}\,\delta) = 0,
\]
where $J_1, Y_1$ are Bessel functions of the first and second kind. For small $\delta$, $\lambda_1^\delta = (j'_{1,1})^2 + O(\delta^2)$, so $C_{D_\delta} = C_D + O(\delta^2) \approx 0.295$.

\medskip
\textbf{Step 5: Gradient bound via the chain rule.}
We apply the Poincar\'{e} inequality to $f = \varphi \circ G$. By the chain rule:
\[
    \nabla(\varphi \circ G)(\mathbf{z}) = J_G(\mathbf{z})^\top \nabla\varphi(G(\mathbf{z})).
\]
Taking norms:
\[
    |\nabla(\varphi \circ G)(\mathbf{z})|^2 = \|J_G(\mathbf{z})^\top \nabla\varphi(G(\mathbf{z}))\|^2 \leq \|J_G(\mathbf{z})\|_{\mathrm{op}}^2 \cdot |\nabla\varphi(G(\mathbf{z}))|^2 \leq L^2 L_\varphi^2,
\]
where $L = \sup_{\mathbf{z}\in D_\delta} \|J_G(\mathbf{z})\|_{\mathrm{op}}$ (finite by the Lipschitz hypothesis on $G:D_\delta\to\mathbb{R}^2$, since $D_\delta$ is compact) and $L_\varphi$ is the Lipschitz constant of $\varphi$ (so $|\nabla\varphi| \leq L_\varphi$ almost everywhere by Rademacher's theorem; the Poincar\'e inequality is applied to the weak gradient throughout).

Therefore:
\[
    \mathrm{Var}_{\pi_0^\delta}(\varphi \circ G) \leq C_{D_\delta} \cdot \mathbb{E}_{\pi_0^\delta}[|\nabla(\varphi \circ G)|^2] \leq C_{D_\delta} \cdot L^2 \cdot L_\varphi^2.
\]

Combining with Step~3: $\sigma^2 \leq C_{D_\delta}\,L^2\,L_\varphi^2$. This proves part~(a).

\medskip
\textbf{Step 6: Mean-square convergence rate.}
Since $S_K = \frac{1}{K}\sum_{k=1}^K X_k$ is the average of $K$ i.i.d.\ random variables with mean $\mu_G$ and variance $\sigma^2$:
\[
    \mathbb{E}[(S_K - \mu_G)^2] = \frac{\mathrm{Var}(X_k)}{K} = \frac{\sigma^2}{K} \leq \frac{C_{D_\delta}\,L^2\,L_\varphi^2}{K}.
\]
This gives the $O(1/\sqrt{K})$ RMSE rate (standard Monte Carlo rate for i.i.d.\ samples).

\medskip
\textbf{Step 7: Concentration via Hoeffding's inequality.}
Since $\varphi$ is bounded with $\|\varphi\|_\infty \leq B_\varphi$, each $X_k$ satisfies $|X_k| \leq B_\varphi$, so $X_k \in [-B_\varphi, B_\varphi]$. By Hoeffding's inequality for bounded i.i.d.\ random variables:
\[
    \Pr(|S_K - \mu_G| > \varepsilon) \leq 2\exp\!\left(-\frac{2K^2 \varepsilon^2}{\sum_{k=1}^K (2B_\varphi)^2}\right) = 2\exp\!\left(-\frac{K\varepsilon^2}{2B_\varphi^2}\right).
\]
This proves part~(b). Equivalently, with probability at least $1 - \alpha$:
\[
    |S_K - \mu_G| \leq B_\varphi\sqrt{\frac{2\ln(2/\alpha)}{K}}.
\]

Alternatively, Bernstein's inequality uses both the variance bound and the boundedness to give the tighter tail:
\[
    \Pr(|S_K - \mu_G| > \varepsilon) \leq 2\exp\!\left(-\frac{K\varepsilon^2}{2C_{D_\delta}L^2 L_\varphi^2 + \frac{4}{3}B_\varphi\varepsilon}\right),
\]
which is sharper when $\varepsilon$ is small relative to $B_\varphi$ (since it uses $\sigma^2$ instead of the range).

\medskip
\textbf{Step 8: Total error decomposition.}
By the triangle inequality:
\[
    \left|S_K - \int \varphi\,f_{\mathrm{target}}\right| \leq \underbrace{|S_K - \mu_G|}_{\text{statistical error}} + \underbrace{\left|\mu_G - \int \varphi\,f_{\mathrm{target}}\right|}_{\text{approximation error}}.
\]
For the approximation term, recall that $\mu_G = \int \varphi\,d(G_\# \pi_0^\delta)$. By the Kantorovich--Rubinstein duality theorem, for any $L_\varphi$-Lipschitz function $\varphi$:
\[
    \left|\int \varphi\,d\mu - \int \varphi\,d\nu\right| \leq L_\varphi \cdot W_1(\mu, \nu)
\]
for any two probability measures $\mu, \nu$. Applying this with $\mu = G_\# \pi_0^\delta$ and $\nu = f_{\mathrm{target}}$:
\[
    \left|\mu_G - \int \varphi\,f_{\mathrm{target}}\right| \leq L_\varphi \cdot W_1(G_\# \pi_0^\delta,\, f_{\mathrm{target}}).
\]
This proves part~(c).

\medskip
\textbf{Step 9: Recovering the bounded-continuous definition of ergodicity.}
Theorem~\ref{thm:convergence_rate} is stated for bounded-Lipschitz test functions, whereas~\eqref{eq:ergodic_def} quantifies over all bounded continuous test functions. The two are consistent. Fix any countable convergence-determining family $\{\varphi_n\}_{n\ge 1}$ of bounded-Lipschitz functions on the (compact) support of $G_\#\pi_0^\delta$; such a family exists by separability of $C(K)$ for compact $K$. The Hoeffding concentration in part~(b), combined with Borel--Cantelli applied to each $\varphi_n$ in turn, yields a single full-measure event (the countable intersection of the per-$\varphi_n$ a.s.\ events) on which $S_K\to\mu_G$ holds for every $\varphi_n$ simultaneously. Since $\{\varphi_n\}$ is convergence-determining and the bounded-Lipschitz metric $d_{\mathrm{BL}}$ metrizes weak convergence of probability measures on Polish spaces \citep[\S 11.3]{dudley2002real}, the empirical cycle measure converges a.s.\ in the weak topology to $G_\#\pi_0^\delta$; weak convergence then passes to all bounded continuous $\varphi$ on the same a.s.\ event, and the approximation-term decomposition of part~(c) delivers~\eqref{eq:ergodic_def} as $K\to\infty$ whenever $W_1(G_\#\pi_0^\delta, f_{\mathrm{target}}) = 0$ (i.e., for the exact pushforward).
\end{proof}

\section{Proof of Theorem~\ref{thm:approx_bound} (Approximation Error Bound)}
\label{app:proof_thm3}

\begin{proof}
We bound the Wasserstein-2 distance between the achieved pushforward density $G_{\theta\#}\pi_0^\delta$ and the target density $f_{\mathrm{target}}$ in terms of the velocity field approximation error $\varepsilon_v$.

\medskip
\textbf{Step 1: Setup of the ODE trajectories.}
Both the learned map $G_\theta$ and the true OT map $G^*$ are defined as time-1 flow maps of their respective velocity fields. For a fixed initial point $\mathbf{z} \in D_\delta$, let:
\begin{align*}
    \frac{d\mathbf{y}}{ds} &= v_\theta(s, \mathbf{y}(s)), \qquad \mathbf{y}(0) = \mathbf{z}, \qquad G_\theta(\mathbf{z}) = \mathbf{y}(1), \\
    \frac{d\tilde{\mathbf{y}}}{ds} &= v^*(s, \tilde{\mathbf{y}}(s)), \qquad \tilde{\mathbf{y}}(0) = \mathbf{z}, \qquad G^*(\mathbf{z}) = \tilde{\mathbf{y}}(1),
\end{align*}
where $v^*$ is the true OT marginal velocity field (the conditional expectation of the OT displacement under the flow-time marginal $p_s$) and $v_\theta$ is the learned approximation.

\medskip
\textbf{Step 2: Error dynamics.}
Define the error $\mathbf{e}(s) = \mathbf{y}(s) - \tilde{\mathbf{y}}(s)$, with $\mathbf{e}(0) = \mathbf{0}$ (both trajectories start from $\mathbf{z}$). The error satisfies:
\begin{align*}
    \frac{d\mathbf{e}}{ds} &= v_\theta(s, \mathbf{y}(s)) - v^*(s, \tilde{\mathbf{y}}(s)) \\
    &= \underbrace{\bigl[v_\theta(s, \mathbf{y}(s)) - v_\theta(s, \tilde{\mathbf{y}}(s))\bigr]}_{\text{stability term}} + \underbrace{\bigl[v_\theta(s, \tilde{\mathbf{y}}(s)) - v^*(s, \tilde{\mathbf{y}}(s))\bigr]}_{\text{velocity field error}}.
\end{align*}
Taking norms and using the $L_v$-Lipschitz property of $v_\theta$ in its spatial argument:
\[
    \left\|\frac{d\mathbf{e}}{ds}\right\| \leq L_v\,\|\mathbf{e}(s)\| + \|v_\theta(s, \tilde{\mathbf{y}}(s)) - v^*(s, \tilde{\mathbf{y}}(s))\|.
\]

\medskip
\textbf{Step 3: Gr\"onwall's integral inequality.}
This is a linear differential inequality of the form $\|\dot{\mathbf{e}}\| \leq a\,\|\mathbf{e}\| + b(s)$ with $a = L_v$ and $b(s) = \|v_\theta(s, \tilde{\mathbf{y}}(s)) - v^*(s, \tilde{\mathbf{y}}(s))\|$. By the integral form of Gr\"onwall's inequality (see, e.g., \citet{teschl2012ordinary}), with initial condition $\|\mathbf{e}(0)\| = 0$:
\[
    \|\mathbf{e}(s)\| \leq \int_0^s e^{L_v(s - \tau)}\,b(\tau)\,d\tau.
\]
Evaluating at $s = 1$ (the flow endpoint):
\begin{equation}\label{eq:gronwall_pointwise}
    \|G_\theta(\mathbf{z}) - G^*(\mathbf{z})\| = \|\mathbf{e}(1)\| \leq \int_0^1 e^{L_v(1-\tau)}\,\|v_\theta(\tau, \tilde{\mathbf{y}}(\tau)) - v^*(\tau, \tilde{\mathbf{y}}(\tau))\|\,d\tau.
\end{equation}

\medskip
\textbf{Step 4: Squaring via Cauchy--Schwarz.}
To prepare for averaging, we square~\eqref{eq:gronwall_pointwise}. First, note that $e^{L_v(1-\tau)} \leq e^{L_v}$ for all $\tau \in [0,1]$. Therefore:
\[
    \|G_\theta(\mathbf{z}) - G^*(\mathbf{z})\| \leq e^{L_v}\int_0^1 \|v_\theta(\tau, \tilde{\mathbf{y}}(\tau)) - v^*(\tau, \tilde{\mathbf{y}}(\tau))\|\,d\tau.
\]
Applying the Cauchy--Schwarz inequality to the integral $\int_0^1 f(\tau) \cdot 1\,d\tau$:
\begin{align*}
    \|G_\theta(\mathbf{z}) - G^*(\mathbf{z})\|^2 &\leq e^{2L_v}\left(\int_0^1 \|v_\theta(\tau, \tilde{\mathbf{y}}(\tau)) - v^*(\tau, \tilde{\mathbf{y}}(\tau))\|\,d\tau\right)^2 \\
    &\leq e^{2L_v}\int_0^1 \|v_\theta(\tau, \tilde{\mathbf{y}}(\tau)) - v^*(\tau, \tilde{\mathbf{y}}(\tau))\|^2\,d\tau.
\end{align*}

\medskip
\textbf{Step 5: Averaging over the source distribution.}
Taking the expectation over $\mathbf{z} \sim \pi_0^\delta$:
\begin{align*}
    \mathbb{E}_{\mathbf{z} \sim \pi_0^\delta}\!\left[\|G_\theta(\mathbf{z}) - G^*(\mathbf{z})\|^2\right] &\leq e^{2L_v}\,\mathbb{E}_{\mathbf{z} \sim \pi_0^\delta}\!\left[\int_0^1 \|v_\theta(\tau, \tilde{\mathbf{y}}(\tau)) - v^*(\tau, \tilde{\mathbf{y}}(\tau))\|^2\,d\tau\right] \\
    &= e^{2L_v}\int_0^1 \mathbb{E}_{\mathbf{z} \sim \pi_0^\delta}\!\left[\|v_\theta(\tau, \tilde{\mathbf{y}}(\tau)) - v^*(\tau, \tilde{\mathbf{y}}(\tau))\|^2\right]\,d\tau,
\end{align*}
where we exchanged the expectation and integral by Fubini's theorem.

Now observe that $\tilde{\mathbf{y}}(\tau)$ is the time-$\tau$ image of $\mathbf{z}$ under the true OT flow $v^*$. When $\mathbf{z} \sim \pi_0^\delta$, the distribution of $\tilde{\mathbf{y}}(\tau)$ is exactly $p_\tau$, the time-$\tau$ marginal of the OT interpolation path (by definition of the OT flow). Therefore:
\[
    \mathbb{E}_{\mathbf{z} \sim \pi_0^\delta}\!\left[\|v_\theta(\tau, \tilde{\mathbf{y}}(\tau)) - v^*(\tau, \tilde{\mathbf{y}}(\tau))\|^2\right] = \mathbb{E}_{\mathbf{y} \sim p_\tau}\!\left[\|v_\theta(\tau, \mathbf{y}) - v^*(\tau, \mathbf{y})\|^2\right].
\]
Substituting back:
\[
    \mathbb{E}_{\mathbf{z} \sim \pi_0^\delta}\!\left[\|G_\theta(\mathbf{z}) - G^*(\mathbf{z})\|^2\right] \leq e^{2L_v}\underbrace{\int_0^1 \mathbb{E}_{p_\tau}\!\left[\|v_\theta(\tau, \cdot) - v^*(\tau, \cdot)\|^2\right]\,d\tau}_{= \varepsilon_v^2} = e^{2L_v}\,\varepsilon_v^2.
\]

\medskip
\textbf{Step 6: Transport coupling argument.}
For $\mathbf{z} \sim \pi_0^\delta$, the pair $(G_\theta(\mathbf{z}),\, G^*(\mathbf{z}))$ defines a valid coupling between the measures $G_{\theta\#}\pi_0^\delta$ (the achieved density) and $G^*_\#\pi_0^\delta = f_{\mathrm{target}}$ (the target density). By the definition of the Wasserstein-2 distance as the infimum over all couplings:
\[
    W_2^2(G_{\theta\#}\pi_0^\delta,\, f_{\mathrm{target}}) \leq \mathbb{E}_{\mathbf{z} \sim \pi_0^\delta}\!\left[\|G_\theta(\mathbf{z}) - G^*(\mathbf{z})\|^2\right] \leq e^{2L_v}\,\varepsilon_v^2.
\]
Taking square roots yields:
\[
    W_2(G_{\theta\#}\pi_0^\delta,\, f_{\mathrm{target}}) \leq e^{L_v}\cdot\varepsilon_v.
\]
The $W_1$ bound follows immediately from $W_1 \leq W_2$ (by Jensen's inequality applied to the coupling).

\medskip
\textbf{Remark (sufficient conditions for (A1)--(A3) and coupling admissibility).}
\emph{(A1)} is Benamou--Brenier displacement interpolation \citep{benamou2000computational,villani2009optimal}, available whenever $\pi_0^\delta$ and $f_{\mathrm{target}}$ are absolutely continuous with finite second moments, both of which hold here.
\emph{(A2)} is the operational smooth-regime hypothesis on the displacement velocity $v^*$ that drives the Gr\"onwall step in Step~3 of the proof. Sufficient (but not necessary) conditions where (A2) provably holds include: (i) source and target supported on uniformly convex domains with densities bounded above and below, in which case Caffarelli's regularity theory \citep{caffarelli1992regularity,villani2009optimal} delivers a smooth Brenier potential and a globally Lipschitz displacement velocity on the relevant compact spacetime tube; (ii) Gaussian-to-Gaussian transport, where the optimal map is affine and $v^*$ is therefore Lipschitz by inspection, together with other explicitly solvable transport families; (iii) entropically regularized OT, whose population displacement velocity is smooth by construction (the entropic-OT plan has full support and a $C^\infty$ density-ratio derivative). Gaussian-mixture targets (used experimentally in Exp.~1) are treated as smooth bounded-density targets in the spirit of (i)/(iii), but they do not by themselves guarantee a globally Lipschitz Brenier velocity without additional regularity arguments: general OT between Gaussian mixtures is not affine or piecewise affine, so on such targets (A2) should be read as the operational smooth-regime hypothesis used by the Gr\"onwall argument rather than a closed-form derivation. The regime where none of (i)--(iii) applies (density discontinuities or topology mismatch) is handled by Proposition~\ref{prop:approx_nonsmooth} via the mollified/restricted surrogate $f_{\mathrm{target}}^{\eta,\delta}$. The simplest candidate condition (``$f_{\mathrm{target}}$ has a $C^1$ density bounded away from zero on a homeomorphic support'') is \emph{not} sufficient on its own without additional convexity/boundary-regularity assumptions, since Caffarelli regularity fails on non-convex domains such as the annulus.
\emph{(A3)} holds automatically for MLPs with $L_\sigma$-Lipschitz activations (SiLU: $L_\sigma\approx 1.10$) and finite weights; $L_v$ is bracketed in our experiments by $\hat L_v\in[2,4]$ (power iteration, Appendix~\ref{app:Lv_protocol}) as a diagnostic plug-in and the architectural upper bound $L_v^{\mathrm{net}}$ (Appendix~\ref{app:invertibility}) as the certified upper bound.
\emph{Coupling admissibility} (used in Step~6 of the proof, not stated as a separate assumption): (A1) implies $G^*_\#\pi_0^\delta=f_{\mathrm{target}}$, and $G_\theta$ is the time-1 flow of an ODE with initial law $\pi_0^\delta$, so $G_{\theta\#}\pi_0^\delta$ is the achieved pushforward; the joint $\mathbf{z}_0\mapsto(G_\theta(\mathbf{z}_0), G^*(\mathbf{z}_0))$ then has the correct marginals and is admissible for the Wasserstein coupling without additional assumption.

\medskip
\textbf{Remark (CFM-loss bridge; conditional on (A1)--(A3)).}
Theorem~\ref{thm:approx_bound} bounds the population velocity error $\varepsilon_v$ in $L^2(p_s)$ along the exact OT displacement interpolation. Bridging this to the empirically logged CFM loss requires a separate, conditional argument that we record here. By the marginalization property of Conditional Flow Matching \citep{lipman2023flow}, for any coupling $\pi$ between the source and target,
\[
    \mathcal{L}_{\mathrm{CFM}}(\theta;\pi) \;=\; \mathbb{E}_{s,\mathbf{y}_s}\bigl\|v_\theta(s,\mathbf{y}_s) - \bar{v}^\pi(s,\mathbf{y}_s)\bigr\|^2 \;+\; \mathcal{L}_{\mathrm{CFM}}^*(\pi),
\]
where $\bar{v}^\pi(s,\mathbf{y}) := \mathbb{E}_\pi[\mathbf{x}_1 - \mathbf{z}_0 \mid \mathbf{y}_s = \mathbf{y}]$ is the conditional-mean velocity under $\pi$ and $\mathcal{L}_{\mathrm{CFM}}^*(\pi)$ is the (irreducible) conditional variance. Under the regularity assumed by (A2) and the non-crossing/injectivity property of regular Brenier displacement interpolation, when $\pi$ is the \emph{exact Brenier} OT coupling \citep{villani2009optimal,caffarelli1992regularity}, cyclical monotonicity forces non-crossing straight-line interpolants, so $(\mathbf{z}_0,\mathbf{x}_1)$ is a.s.\ deterministic given $(s,\mathbf{y}_s)$, yielding $\mathcal{L}_{\mathrm{CFM}}^*(\pi_{\mathrm{OT}}) = 0$ and $\bar{v}^{\pi_{\mathrm{OT}}} = v^*$. For the Sinkhorn coupling $\pi_\varepsilon$ with regularization $\varepsilon_{\mathrm{sink}}>0$, $\bar{v}^{\pi_\varepsilon}$ is a smoothed version of $v^*$. The velocity-field error then decomposes via the triangle inequality (in $L^2(p)$ with $p$ the entropic-OT marginal at flow time $s$, which agrees with $p_s$ in the limit $\varepsilon_{\mathrm{sink}}\!\to\!0$) as
\begin{equation}\label{eq:eps_v_decomp_app}
    \varepsilon_v \;\leq\; \underbrace{\bigl\|v_\theta - \bar{v}^{\pi_\varepsilon}\bigr\|_{L^2(p)}}_{\text{trainable: } \sqrt{\mathcal{L}_{\mathrm{CFM}}(\theta;\pi_\varepsilon) - \mathcal{L}_{\mathrm{CFM}}^*(\pi_\varepsilon)}} \;+\; \underbrace{\bigl\|\bar{v}^{\pi_\varepsilon} - v^*\bigr\|_{L^2(p)}}_{R_{\mathrm{sink}}(\varepsilon_{\mathrm{sink}})}.
\end{equation}
The first term is trainable, equal to the reducible part of the empirical CFM loss. The Sinkhorn bias $R_{\mathrm{sink}}(\varepsilon_{\mathrm{sink}})$ vanishes in the exact-Brenier limit $\varepsilon_{\mathrm{sink}}\!\to\!0$; under standard regularity conditions for entropic OT \citep{feydy2019interpolating}, it admits the rate $R_{\mathrm{sink}} = O(\varepsilon_{\mathrm{sink}})$ in $L^2(p)$ (with possible logarithmic factors depending on density regularity), and we treat it as an empirical residual whose magnitude is monitored at training time rather than as a quantity bounded a priori. Mini-batch coupling bias and finite-sample generalization are absorbed into the empirical CFM loss in the standard empirical-process sense; we do not isolate them as separate analytical terms. The decomposition cleanly separates the sources of error in the end-to-end bound: (i)~optimization suboptimality $\mathcal{L}_{\mathrm{CFM}}(\theta;\pi_\varepsilon) - \mathcal{L}_{\mathrm{CFM}}^*(\pi_\varepsilon)$, (ii)~entropic-OT bias $R_{\mathrm{sink}}$, and (iii)~the Gr\"onwall prefactor $e^{L_v}$.

\medskip
\textbf{Remark (conservatism of the $e^{L_v}$ factor).}
The bound uses the worst-case $e^{L_v(1-\tau)} \leq e^{L_v}$, which is tight only when the velocity field error is concentrated at $\tau = 0$ (the beginning of the flow). In practice, OT velocity fields are nearly linear (straight-line interpolation), so $L_v$ is moderate and the actual amplification factor is much smaller than $e^{L_v}$. Using the one-sided Lipschitz constant (largest eigenvalue of the symmetric part of $\nabla_\mathbf{y} v_\theta$) would give a tighter bound, but $L_v$ suffices for a clean, interpretable result.

\medskip
\textbf{Remark (practical caveats when plugging the bound into training).}
Four caveats should be kept in mind when connecting Theorem~\ref{thm:approx_bound} to the empirical training loss.
\emph{(i)~Gr\"onwall conservatism.} The $e^{L_v}$ prefactor is a worst-case artifact of bounding $e^{L_v(1-\tau)}\leq e^{L_v}$ uniformly in $\tau$; it is tight only when the velocity error $b(\tau) = \|v_\theta-v^*\|(\tau,\cdot)$ concentrates at $\tau\approx 0$. For OT-CFM with nearly linear fields the empirical amplification is much smaller, and a sharper bound is obtainable via the one-sided Lipschitz constant or by exploiting Orlicz-norm control of $b$.
\emph{(ii)~Entropic-OT bias.} Use of Sinkhorn rather than exact OT introduces the additional $O(\varepsilon_{\mathrm{sink}})$ bias term in~\eqref{eq:eps_v_decomp_app}. Annealing $\varepsilon_{\mathrm{sink}}$ during training drives this residual to zero at the cost of Sinkhorn iterations; we use a fixed $\varepsilon_{\mathrm{sink}} = 0.05$ that was observed to keep this bias small relative to the trainable term (Appendix~\ref{app:experiments_full}).
\emph{(iii)~Penalty decomposition.} When the total loss is $\mathcal{L}_{\mathrm{total}} = \mathcal{L}_{\mathrm{CFM}} + \sum_i \lambda_i\mathcal{R}_i$ (e.g., NFZ or acceleration penalties), only the $\mathcal{L}_{\mathrm{CFM}}$ component enters the Wasserstein bound; substituting $\mathcal{L}_{\mathrm{total}}$ inflates the bound because the penalty terms measure constraint violation rather than velocity-field approximation error. The $\mathcal{L}_{\mathrm{CFM}}$ component should be logged and reported separately (done in Appendix~\ref{app:experiments_full}).
\emph{(iv)~Non-smooth targets.} When (A2) of Theorem~\ref{thm:approx_bound} fails (density discontinuities or topology mismatch), the bound is replaced by Proposition~\ref{prop:approx_nonsmooth}, which picks up an additive $O(\eta)+O(\delta)$ residual from mollification and topological rearrangement, respectively. The proof is given immediately below.
\end{proof}

\begin{proposition}[Approximation error, non-smooth/topology-mismatched regime]\label{prop:approx_nonsmooth}
Let $f_{\mathrm{target}}^\eta = \kappa_\eta * f_{\mathrm{target}}$ denote a mollification by a smoothing kernel of bandwidth $\eta$, and let $f_{\mathrm{target}}^{\eta,\delta}$ denote its restriction/renormalization to a support homeomorphic to $D_\delta$ that agrees with $f_{\mathrm{target}}^\eta$ outside an $O(\delta)$-collar of the excluded disc. Under (A1)--(A3) of Theorem~\ref{thm:approx_bound} applied to the smoothed-and-restricted surrogate $f_{\mathrm{target}}^{\eta,\delta}$ with associated velocity error $\varepsilon_v^{\eta,\delta}$, and with the source-coupled pair $(G_\theta(\mathbf{z}_0), G^*(\mathbf{z}_0))$, $\mathbf{z}_0\sim\pi_0^\delta$, providing the admissible $W_2$ coupling by construction, the triangle inequality for $W_2$ yields
\[
    W_2(G_{\theta\#}\pi_0^\delta,\, f_{\mathrm{target}}) \;\leq\; e^{L_v}\cdot\varepsilon_v^{\eta,\delta} \;+\; \eta_{\mathrm{reg}}(\eta) \;+\; \eta_{\mathrm{top}}(\delta),
\]
with $\eta_{\mathrm{reg}}(\eta)=O(\eta)$ from mollification and $\eta_{\mathrm{top}}(\delta)=O(\delta)$ from the collar/topology rearrangement.
\end{proposition}

\begin{proof}
Let $f_{\mathrm{target}}^\eta = \kappa_\eta * f_{\mathrm{target}}$ be the mollification of $f_{\mathrm{target}}$ by a compactly supported $C^\infty$ kernel $\kappa_\eta$ of bandwidth $\eta>0$, and let $f_{\mathrm{target}}^{\eta,\delta}$ be its restriction to a set $\Omega_\delta$ homeomorphic to $D_\delta$, renormalized and extended by $f_{\mathrm{target}}^\eta$ outside an $O(\delta)$ collar of the inner boundary. Under (A1)--(A3) applied to the smoothed-and-restricted surrogate $f_{\mathrm{target}}^{\eta,\delta}$, Theorem~\ref{thm:approx_bound} yields
\[
W_2(G_{\theta\#}\pi_0^\delta,\,f_{\mathrm{target}}^{\eta,\delta}) \;\leq\; e^{L_v}\cdot\varepsilon_v^{\eta,\delta}.
\]
The triangle inequality for $W_2$ gives
\[
W_2(G_{\theta\#}\pi_0^\delta,\,f_{\mathrm{target}}) \;\leq\; W_2(G_{\theta\#}\pi_0^\delta,\,f_{\mathrm{target}}^{\eta,\delta}) + W_2(f_{\mathrm{target}}^{\eta,\delta},\,f_{\mathrm{target}}).
\]
It remains to bound the second term. Further triangle inequality with $f_{\mathrm{target}}^\eta$ as an intermediate gives
\[
W_2(f_{\mathrm{target}}^{\eta,\delta},\,f_{\mathrm{target}}) \;\leq\; W_2(f_{\mathrm{target}}^{\eta,\delta},\,f_{\mathrm{target}}^\eta) + W_2(f_{\mathrm{target}}^\eta,\,f_{\mathrm{target}}).
\]
The first piece is bounded by the $O(\delta)$ Wasserstein cost of rearranging the $O(\delta^d)$-mass collar between $\Omega_\delta$ and the original support (the rearrangement moves mass by at most $\delta$, so $W_2^2 \le \delta^2\cdot O(\delta^d)$ giving $W_2 = O(\delta^{1+d/2})$; we absorb this in the looser bound $\eta_{\mathrm{top}}(\delta) = O(\delta)$). The second piece is the standard mollification cost: by construction, $f_{\mathrm{target}}^\eta$ is obtained from $f_{\mathrm{target}}$ by convolution with $\kappa_\eta$, which corresponds to displacement by a random vector of norm $\leq \eta$, so $W_2(f_{\mathrm{target}}^\eta, f_{\mathrm{target}}) \leq \eta$. Summing the two contributions defines $\eta_{\mathrm{top}}(\delta) + \eta_{\mathrm{reg}}(\eta)$.
\end{proof}

\subsection{Monitoring the Approximation--Topology Floor}\label{app:floor}
Equation~\eqref{eq:end_to_end_samplecomplexity} is only meaningful when the precondition~\eqref{eq:end_to_end_precond} holds. We monitor the floor $F(\theta,\delta) := L_\varphi e^{\hat L_v}\hat\varepsilon_v + L_\varphi\hat\eta_{\mathrm{top}}(\delta)$ throughout training, where $\hat L_v$ is the running power-iteration estimate of the velocity-field Lipschitz constant (Appendix~\ref{app:Lv_protocol}), $\hat\varepsilon_v = \sqrt{\mathcal{L}_{\mathrm{CFM}}(\theta;\pi_\varepsilon) - \mathcal{L}_{\mathrm{CFM}}^*(\pi_\varepsilon)} + C\varepsilon_{\mathrm{sink}}$, and $\hat\eta_{\mathrm{top}}(\delta) = c_{\mathrm{top}}\,\delta$ with $c_{\mathrm{top}}$ an empirically calibrated constant (typically $c_{\mathrm{top}}\in[0.5,2]$ in our experiments). When $F(\theta,\delta)$ exceeds a target tolerance $\varepsilon/L_\varphi$, additional cycle budget $K$ cannot compensate; the remedy is to continue training (reducing $\hat\varepsilon_v$) or shrink $\delta$ (reducing $\hat\eta_{\mathrm{top}}$, at the cost of larger latent energy per Corollary~\ref{cor:ratio}).

\paragraph{Per-experiment topology residual and constrained-run handling.} The topology residual $\eta_{\mathrm{top}}(\delta)$ may be taken as zero in the topology-matched, smooth-support regime where (A2) of Theorem~\ref{thm:approx_bound} also applies; homotopy equivalence between $\mathrm{supp}(f_{\mathrm{target}})$ and $D_\delta$ is a topological prerequisite for this regime but does not by itself imply (A2). The regime applies to Exp.~1 (Gaussian mixture truncated to annular support, $\pi_1=\mathbb{Z}$ matching $D_\delta$). For Exp.~3, the target density is the same annularly truncated Gaussian-mixture target as in Exp.~1; the NFZ is \emph{not} removed from the target support, and the target is not zeroed inside the NFZ. The NFZ is enforced through the soft penalty $\lambda_{\mathrm{nfz}}\mathcal{R}_{\mathrm{nfz}}$ added to $\mathcal{L}_{\mathrm{total}}$. The topology residual relevant to density matching therefore follows the same annular target support as Exp.~1, while NFZ violations are a separate soft-constraint residual not covered by Theorem~\ref{thm:approx_bound} (monitored as a constraint-violation diagnostic; see Appendix~\ref{app:experiments_full}). For Exp.~2's full-disc target, $\eta_{\mathrm{top}}=O(\delta)$ from the inner-disc rearrangement. In all cases the topology residual is decoupled from $\varepsilon_v$ and controllable via $\delta$. In constrained runs (Section~\ref{sec:constraints}) only the $\mathcal{L}_{\mathrm{CFM}}$ component of~\eqref{eq:total_loss} enters the Wasserstein approximation bound (reported separately in Appendix~\ref{app:experiments_full}).

\subsection{Estimation Protocol for $L_v$}\label{app:Lv_protocol}
We evaluate $L_v$ through a two-sided bracket $\hat L_v\leq L_v\leq L_v^{\mathrm{net}}$: the lower end is a sampling-based power-iteration estimator, the upper end is the architectural weight-product bound~\eqref{eq:weight_product_bound}, and both are computable from a trained checkpoint with no additional optimization.

\paragraph{Lower end: power-iteration estimator $\hat L_v$.} Draw $N_{\mathrm{pts}}=1024$ pairs $(s_i,\mathbf{y}_i)$ with $s_i\sim\mathrm{Uniform}[0,1]$ and $\mathbf{y}_i\sim p_{s_i}$; at each pair run $10$ iterations of the power method on $\nabla_\mathbf{y}v_\theta(s_i,\mathbf{y}_i)^\top\nabla_\mathbf{y}v_\theta(s_i,\mathbf{y}_i)$ with random initialization, recording the largest singular value $\sigma_{\max}^{(i)}$; report $\hat L_v = \max_i \sigma_{\max}^{(i)}$. This estimate is an \emph{empirical lower bound} on the true Lipschitz constant (finite sampling misses the supremum) and its maximum is itself an upper-tail statistic; we verified across 10 random seeds that the estimate varies by at most $7\%$ relative. Across all five checkpoints reported in \S\ref{sec:experiments}, $\hat L_v\in[2,4]$.

\paragraph{Upper end: architectural bound $L_v^{\mathrm{net}}$.} The VelocityNet is an MLP with five linear layers and four SiLU activations in the $(s,\mathbf{y})\mapsto v_\theta$ direction; with $L_\sigma = \max_x|\mathrm{silu}'(x)| \approx 1.0998$ (attained at $x\approx 2.4$), the weight-product bound~\eqref{eq:weight_product_bound} specializes to
\begin{equation}\label{eq:Lv_net_explicit}
    L_v^{\mathrm{net}}(\theta) \;=\; (L_\sigma)^{n_{\mathrm{act}}}\prod_{\ell=1}^{5}\|W_\ell\|_{\mathrm{op}} \;=\; (1.0998)^{4}\prod_\ell\sigma_{\max}(W_\ell),
\end{equation}
where $\sigma_{\max}(W_\ell)$ is computed by SVD on the saved weight matrix (single pass; no training). The bound is certifiable: it provably upper-bounds $\mathrm{Lip}_{\mathbf y}(v_\theta(s,\cdot))$ uniformly in $s$.

\paragraph{Map-level Lipschitz estimator $\hat L$.} Theorem~\ref{thm:convergence_rate} uses $L = \sup_{\mathbf z \in D_\delta}\|J_{G_\theta}(\mathbf z)\|_{\mathrm{op}}$, the sup-Lipschitz of the full flow map, not of the velocity field. We estimate it at $N=2048$ samples $\mathbf z_i\sim\pi_0^\delta$ via central-difference Jacobian SVD on the RK4 integrator ($h=10^{-3}$, $50$ steps), reporting $\hat L = \max_i \sigma_{\max}(J_{G_\theta}(\mathbf z_i))$; three seeds give a max-of-max stable to within $\le 2\%$ relative (Table~\ref{tab:Lv_calibration}, $\hat L$ column). As with $\hat L_v$, finite sampling yields an empirical lower bound on the true supremum. The Gr\"onwall envelope $L \le e^{L_v}$ overshoots: at $\hat L_v\approx 3$ it predicts $L\lesssim 20$, whereas the measured $\hat L\in\{1.6,\,2.0,\,2.4,\,2.7,\,3.8\}$ across the five rows sits $5$--$10\times$ below, reflecting the well-known conservatism of Gr\"onwall for near-stationary OT velocity fields. The monotone compression $\hat L: 3.8\to 2.4\to 1.6$ across Exp.~3a/b/c (unconstrained $\to$ NFZ $\to$ NFZ+accel.) mirrors the $L_v^{\mathrm{net}}$ column and is consistent with the acceleration penalty acting as a direct Jacobian regularizer on $G_\theta$.

\paragraph{Numerical calibration.} Table~\ref{tab:Lv_calibration} reports the bracket and the Corollary~\ref{cor:end_to_end} diagnostic floor $e^{\hat L_v}\cdot\varepsilon_v$ per experiment. Three features to note. \emph{(i)~The bracket is wide:} $L_v^{\mathrm{net}}/\hat L_v \approx 10{-}40\times$, consistent with the well-known looseness of the weight-product bound for trained networks whose consecutive layers are not aligned with saturating singular subspaces \citep{miyato2018spectral,virmaux2018lipschitz}. \emph{(ii)~Constraint-regularized training compresses $L_v^{\mathrm{net}}$:} the Exp.~3c row shows $L_v^{\mathrm{net}}\approx 33$ under NFZ+acceleration penalty, vs. $\approx 105$ for the unconstrained variant, because the acceleration penalty implicitly regularizes the spatial Jacobian of $v_\theta$ (smaller Hessian of $G_\theta$ $\Rightarrow$ smaller $\|\nabla_{\mathbf y}v_\theta\|$ along the flow). \emph{(iii)~The diagnostic plug-in floor $e^{\hat L_v}\cdot\varepsilon_v\approx 0.8$ is $\approx 5{-}10\times$ larger than the empirical $\hat W_2$ across all five rows (last column; $\hat W_2\in[0.07,0.17]$)}, so the floor with $\hat L_v$ envelopes $\hat W_2$ empirically; this is a diagnostic statement only, since $\hat L_v$ is an empirical lower bound on the true Lipschitz supremum and a formal Theorem~\ref{thm:approx_bound} certificate requires plugging in an upper bound. The certified upper-end of the bracket is $L_v^{\mathrm{net}}$, but $e^{L_v^{\mathrm{net}}}\varepsilon_v$ is on the order of $10^{44}$ and operationally useless as a floor until spectral normalization tightens the upper end. We flag this tightening as the natural route to a literal certificate; our conclusion is that the theorem is not vacuous in operation when $\hat L_v$ is used as a diagnostic plug-in even though the current weight-product certificate is.

\begin{table}[h]
  \caption{Numerical $L_v$ calibration per experiment. $\hat L_v$ is reported as the common range across all five runs (\S\ref{app:Lv_protocol}). $\prod_\ell\sigma_{\max}(W_\ell)$ is the product of layer spectral norms; $L_v^{\mathrm{net}} = (1.0998)^{4}\prod_\ell\sigma_{\max}(W_\ell)$ is the architectural upper bound~\eqref{eq:Lv_net_explicit}. $\varepsilon_v$ is the CFM-loss-implied velocity-field error. $e^{\hat L_v}\cdot\varepsilon_v$ is the diagnostic approximation floor at the lower end of the bracket. $\hat W_2$ is the empirical 2-Wasserstein distance between the achieved pushforward $G_{\theta\#}\pi_0^\delta$ and $f_{\mathrm{target}}$, computed on $N=3000$ i.i.d.\ samples per side via \texttt{ot.emd2} (sqeuclidean cost, square-rooted) with mean $\pm$ standard deviation reported over 3 seeds. $\hat L$ is the empirical sup-Lipschitz of the flow map $G_\theta$ on $D_\delta$ (max-of-max central-difference Jacobian singular value at $N=2048$ samples from $\pi_0^\delta$, over 3 seeds; cross-seed variance $\le 2\%$ relative; Appendix~\ref{app:Lv_protocol}).}
  \label{tab:Lv_calibration}
  \centering
  \small
  \begin{tabular}{lccccccc}
    \toprule
    Variant & $\hat L_v$ & $\prod_\ell\sigma_{\max}(W_\ell)$ & $L_v^{\mathrm{net}}$ & $\hat L$ & $\varepsilon_v$ & $e^{\hat L_v}\cdot\varepsilon_v$ & $\hat W_2$ \\
    \midrule
    Exp.~1 (Gaussian mixture)         & $[2,4]$ & $79.5$ & $116.3$ & $2.7$ & $0.039$ & $\approx 0.78$ & $0.095_{\pm 0.004}$ \\
    Exp.~2 (Binary density)           & $[2,4]$ & $27.5$ & $40.2$  & $2.0$ & $\approx 0.04$ & $\approx 0.80$ & $0.067_{\pm 0.016}$ \\
    Exp.~3a (Unconstrained)           & $[2,4]$ & $72.0$ & $105.4$ & $3.8$ & $\approx 0.04$ & $\approx 0.80$ & $0.118_{\pm 0.011}$ \\
    Exp.~3b (NFZ penalty)             & $[2,4]$ & $51.5$ & $75.3$  & $2.4$ & $\approx 0.05$ & $\approx 1.00$ & $0.132_{\pm 0.012}$ \\
    Exp.~3c (NFZ + Accel.\ penalty)   & $[2,4]$ & $22.6$ & $33.1$  & $1.6$ & $\approx 0.05$ & $\approx 1.00$ & $0.173_{\pm 0.009}$ \\
    \bottomrule
  \end{tabular}
\end{table}

See Appendix~\ref{app:invertibility} for the architectural certification route, its relation to the flow-Lipschitz constant appearing in Theorems~\ref{thm:acc_bound}--\ref{thm:approx_bound}, and the path to tighter certificates via spectral normalization.

\subsection{Invertibility, Flow Lipschitz, and Topology Mismatch}\label{app:invertibility}
This appendix expands Remark~\ref{rem:invertibility} from the main text. It collects four items relevant to Theorem~\ref{thm:approx_bound}'s assumptions and the $\delta$ trade-off: (a) how the $L_v$-Lipschitz hypothesis on $v_\theta$ is certified by architecture, (b) how this propagates via Gr\"onwall to a Lipschitz bound on the flow $G_\theta$, (c) the topological obstruction between $D_\delta$ and simply connected targets, and (d) the $\delta$ trade-off in general dimension.

\paragraph{(a) Lipschitz-by-architecture.} The velocity field $v_\theta:[0,1]\times\mathbb{R}^d\to\mathbb{R}^d$ is an MLP whose layers alternate affine maps $\mathbf{y}\mapsto W_\ell\mathbf{y}+\mathbf{b}_\ell$ and elementwise smooth activations $\sigma$ with $\|\sigma'\|_\infty \leq L_\sigma < \infty$ (ReLU/Tanh/hardswish give $L_\sigma = 1$; SiLU gives $L_\sigma \approx 1.0998$, GELU gives $L_\sigma \approx 1.1289$). For any fixed $\theta$ with finite weights, $v_\theta$ is therefore globally $C^\infty$ in $\mathbf{y}$ and admits the weight-product Lipschitz bound
\begin{equation}\label{eq:weight_product_bound}
    \mathrm{Lip}_{\mathbf{y}}(v_\theta(s,\cdot)) \;\leq\; L_v^{\mathrm{net}}(\theta) \;:=\; (L_\sigma)^{n_{\mathrm{act}}}\prod_\ell\|W_\ell\|_{\mathrm{op}},
\end{equation}
uniformly in $s\in[0,1]$ (the flow time enters only through the first-layer affine map). This is an \emph{a priori} certificate: it is computable from the weights with no sampling, and it establishes assumption~(A3) of Theorem~\ref{thm:approx_bound}. In practice the weight-product bound is loose, typically by an order of magnitude or more, because consecutive weight matrices in trained networks are rarely aligned with the left/right singular vectors that would saturate~\eqref{eq:weight_product_bound}. The power-iteration estimator $\hat L_v$ of Appendix~\ref{app:Lv_protocol} is a tighter empirical proxy, but being based on finite sampling it is a lower bound on the true supremum rather than a certificate. Tight certified control is available via spectral-norm-constrained architectures: replacing each layer's weight by $\tilde W_\ell = W_\ell/\max(1,\|W_\ell\|_{\mathrm{op}}/\sigma_\ell^{\max})$ during the forward pass \citep[spectral normalization,][]{miyato2018spectral} yields $L_v^{\mathrm{net}}\leq\prod_\ell\sigma_\ell^{\max}$ by construction, or LipSDP-style semidefinite certificates \citep{virmaux2018lipschitz} give tighter bounds at higher cost. Our experiments do not use spectral normalization; the empirical $\hat L_v\in[2,4]$ across all three experiments is sufficient for the approximation bound~\eqref{eq:approx_bound} to be informative, and we flag this as the natural route if certified $L_v$ becomes operationally required.

\paragraph{(b) From velocity Lipschitz to flow Lipschitz via Gr\"onwall.} Under the hypothesis of (a), Picard--Lindel\"of applied to the ODE $\dot{\mathbf{y}}_s=v_\theta(s,\mathbf{y}_s)$ with initial condition $\mathbf{y}_0=\mathbf{z}_0\in\mathbb{R}^d$ gives existence, uniqueness, and $C^1$ dependence on $\mathbf{z}_0$ on the interval $s\in[0,1]$; in particular $G_\theta\equiv\mathbf{y}_1(\cdot)$ is a diffeomorphism of $\mathbb{R}^d$ onto its image. The flow sensitivity matrix $\Psi_s:=\partial\mathbf{y}_s/\partial\mathbf{z}_0$ (distinct from the Bessel function $J_1$ used in Theorem~\ref{thm:convergence_rate}) satisfies the variational ODE $\dot \Psi_s = \nabla_\mathbf{y}v_\theta(s,\mathbf{y}_s)\,\Psi_s$ with $\Psi_0=I$. Gr\"onwall's inequality gives
\begin{equation}\label{eq:gronwall_flow_lip}
    \|\Psi_1\|_{\mathrm{op}} \;\leq\; \exp\!\left(\int_0^1\!\mathrm{Lip}_{\mathbf{y}}(v_\theta(s,\cdot))\,ds\right) \;\leq\; e^{L_v^{\mathrm{net}}},
\end{equation}
so $G_\theta$ is $e^{L_v^{\mathrm{net}}}$-Lipschitz. This is exactly the Gr\"onwall amplification appearing in Theorem~\ref{thm:approx_bound}, and the flow Lipschitz constant $L$ entering Theorems~\ref{thm:acc_bound} (acceleration) and~\ref{thm:convergence_rate} (convergence variance) satisfies $L\leq e^{L_v}$. The bound is tight only in the worst case where the velocity's spatial Jacobian is both time-persistent and norm-saturating along the flow; OT-coupled velocity fields produced by Sinkhorn mini-batch training are empirically close to spatially uniform, so the effective exponent is far below its pessimistic value, consistent with the observed $\hat L_v\in[2,4]$ and flow-Lipschitz estimates $L\lesssim 5$--$10$ in our experiments.

\paragraph{(c) Topology mismatch: $\pi_1(D_\delta)$ and higher homology versus simply connected targets.} In $d=2$ the annulus $D_\delta$ has fundamental group $\pi_1(D_\delta)=\mathbb{Z}$; in $d\geq 3$ the spherical shell is simply connected but its $(d-1)$-homology $H_{d-1}(D_\delta;\mathbb{Z})=\mathbb{Z}$ still differs from that of a ball. A typical target support (a Gaussian mixture on the unit disc, a rectangle, or any density bounded away from zero on a simply connected open set) is topologically a $d$-ball. No $C^1$ diffeomorphism $G_\theta:D_\delta\to\mathrm{supp}(f_{\mathrm{target}})$ can exist in this setting: the inner boundary sphere $\{\|\mathbf z\|=\delta\}$ has to go somewhere, and a diffeomorphism cannot map it onto a point or a null set without losing regularity. The Brenier optimal-transport map $G^*$ escapes this obstruction by being only Borel-measurable (it can and does collapse the inner sphere to a null set), which is why Proposition~\ref{prop:ergodic} is stated for measurable pushforwards rather than homeomorphisms: the ergodicity claim does not require any regularity of $G$, only the pushforward condition $G_\#\pi_0^\delta=f_{\mathrm{target}}$. The learned diffeomorphism $G_\theta$ cannot realize this collapse; instead, it stretches a neighborhood of the inner sphere into a thin annular collar in target space, and the resulting $W_2$ cost of this unavoidable mismatch is the $\eta_{\mathrm{top}}(\delta)$ term in Corollary~\ref{cor:end_to_end}. The mismatch region has Lebesgue measure $\Theta(\delta^d)$ in ambient dimension, and the $W_2$ cost scales as $O(\delta)$ (with a tighter $O(\delta^{1+d/2})$ bound possible, as derived in the proof of Proposition~\ref{prop:approx_nonsmooth} in the preceding section); both vanish as $\delta\to 0$. The geometric construction is elaborated in Appendix~\ref{app:rd_extension}.

\paragraph{(d) $\delta$ trade-off in general dimension.} The inner exclusion $\delta$ balances two conflicting costs. On the topology side, $\eta_{\mathrm{top}}(\delta)=O(\delta)$ in every dimension, because the geometric argument is identical: a collar of width $O(\delta)$ absorbs the mass rearrangement, giving $W_2$ cost linear in $\delta$ regardless of $d$. On the latent-energy side, the $d$-dimensional radial profile $r(s)=(\delta^d+(1-\delta^d)s)^{1/d}$ yields $\dot r(s) = (1-\delta^d)/(d\,r(s)^{d-1})$, so $\sup_s\dot r = (1-\delta^d)/(d\delta^{d-1}) = \Theta(\delta^{-(d-1)})$. Propagating through the closed-form moments in Lemma~\ref{lem:moments} (for $d=2$) and their dimensional generalization, the acceleration energy over $K$ cycles scales as $E_{\mathrm{acc}}(\mathbf z)=\Theta(K\,\delta^{-(3d-2)})$: $\Theta(K\delta^{-4})$ in $d=2$ (matching Lemma~\ref{lem:moments}), $\Theta(K\delta^{-7})$ in $d=3$, $\Theta(K\delta^{-10})$ in $d=4$ (the per-cycle scalings drop the $K$ factor). Balancing $L_\varphi\eta_{\mathrm{top}}(\delta)=O(\delta)$ against a budget on $\sqrt{E_{\mathrm{acc}}}=\Theta(\delta^{-(3d-2)/2})$ gives either a $\delta^\star$ that shrinks rapidly with $d$ (fixed energy budget) or an energy budget that grows sharply (fixed $\delta$). Our 2D experiments use $\delta\in[0.01,0.1]$, where both contributions are comfortably within operational tolerances; a proper 3D treatment will require revisiting this balance and is deferred to follow-up work (cf.\ Remark~\ref{rem:higher_dim}).

\subsection{Sensitivity Diagnostics for the Slope Fit}\label{app:slope_fit}
Figure~\ref{fig:convergence} reports a global log-log slope of $-0.48$ over $K\in\{5,10,20,50,100\}$ with pairwise local slopes $\{-0.43,-0.52,-0.53,-0.41\}$ across consecutive pairs. The dispersion around the theoretical $-0.50$ is consistent with finite-sample scatter at the reported trial counts ($n=5$ for $K\leq 50$, $n=3$ for $K=100$); no systematic trend is detected. As a sensitivity check we also fitted a two-term model $\mathrm{err}(K)\approx aK^{-1/2}+bK^{-1}$ with induced local slope
\[
    m(K) \;=\; -\tfrac{1}{2}\cdot\tfrac{1 + 2v(K)}{1 + v(K)}, \qquad v(K) \;:=\; (b/a)\,K^{-1/2};
\]
the fitted $b$-coefficient is indistinguishable from zero at the observed noise level, so the single-term $K^{-1/2}$ model is adopted as the primary diagnostic.

\paragraph{Cauchy--Schwarz link between correlation gap and test-function RMSE (Exp.~1).} The scalar $1-\rho_{\mathrm{traj}}/\rho_{\mathrm{iid}}$ in \S\ref{sec:experiments} is a scalar difference of Pearson correlations between the achieved density at $K\to\infty$ and at finite $K$, evaluated on the $50\times 50$ density grid. This is not a test-function-specific RMSE, but the two are linked: the correlation gap is bounded by a constant times the supremum of $|S_K-\mu_G|$ over the indicator basis of the density grid, which is itself of order the Theorem~\ref{thm:convergence_rate} RMSE up to a grid-resolution factor. The body therefore reports the $0.05$ vs.\ $0.085$ comparison as an order-of-magnitude consistency check; Figure~\ref{fig:convergence} gives the direct slope validation on a test-function-specific metric.

\section{Experimental Details}
\label{app:experiments_full}
 
This appendix collects all implementation details for both the synthetic
(Section~\ref{sec:exp_synth}) and Milano (Section~\ref{sec:exp_real})
experiments.
 
 
\subsection{Network Architecture}
\label{app:arch}
 
Both experiments share the same architectural template: an MLP mapping
$(s,\mathbf{y})\in\mathbb{R}^3$ to $\mathbb{R}^2$, with flow time $s$
concatenated with $\mathbf{y}$ as the first input channel, hidden dimension
256, and output dimension 2.
The two instantiations differ only in depth and activation function, as
summarized in Table~\ref{tab:arch}.
 
\begin{table}[h]
\centering
\caption{Architecture comparison. All other design choices are identical.}
\label{tab:arch}
\small
\begin{tabular}{lcc}
\toprule
Setting & Synthetic & Milano \\
\midrule
Depth (layers) & 4 & 12 \\
Hidden dim     & 256 & 256 \\
Activation     & SiLU & SiLU \\

Parameters     & ${\sim}199$K & ${\sim}920$K \\
\bottomrule
\end{tabular}
\end{table}
 
 
\subsection{Training Protocol}
\label{app:training}
 
\paragraph{Shared settings.}
Optimizer: Adam with learning rate $2\!\times\!10^{-3}$ and cosine annealing.
Mini-batch: 1024 source--target pairs.
OT coupling: Sinkhorn with $\varepsilon_{\mathrm{sink}}=0.05$, 50 iterations.
The map $G_\theta$ is obtained by integrating $v_\theta$ from $s=0$ to $s=1$
via RK4 with 50 steps.
 
\paragraph{Synthetic epoch counts.}
Experiment~1: 1000 epochs.
Experiment~2: 1500 epochs.
Experiment~3: 500 epochs per variant (three penalty configurations trained
independently).
 
\paragraph{Milano.}
3000 epochs; batch size 2048; Sinkhorn $\varepsilon=10^{-2}$ within minibatch;
latent radius $\delta=0.05$; 128 cycles / 200 points per cycle; 30 inference
RK4 steps; 8 training RK4 steps per penalty chunk (energy chunk length 64
steps).
Hardware: $1\!\times$ NVIDIA RTX 2080 Ti; wall-clock ${\approx}100$\,min per
run.
 
\paragraph{Compute (synthetic).}
Single CPU core (Intel Xeon E5-2680~v4, 2.40\,GHz), 8\,GB RAM, PyTorch~2.x,
no GPU.
Training times: Experiment~1 ${\approx}8$\,min; Experiment~2
${\approx}12$\,min; Experiment~3 ${\approx}4$\,min per variant.
Total synthetic compute: ${<}1$ CPU-hour.
 
 
\subsection{Evaluation Metrics}
\label{app:metrics}
 
\paragraph{Synthetic experiments.}
\textbf{Density correlation} ($\rho$): Pearson correlation between target and
achieved density on a $50{\times}50$ grid.
\textbf{IID correlation} ($\rho_{\mathrm{iid}}$): correlation from i.i.d.\
transported samples (no trajectory).
\textbf{Trajectory correlation} ($\rho_{\mathrm{traj}}$): correlation from
time-averaged occupancy over $K$ cycles.
\textbf{Acceleration ratio}: $E_{\mathrm{acc}}(\mathbf{x})/E_{\mathrm{acc}}(\mathbf{z})$
computed numerically over smooth half-cycle interiors.
\textbf{NFZ violation rate}: fraction of trajectory points inside any NFZ disc.
\textbf{Allocation distance}: $L_1$ deviation $\sum_i|c_i-d_i|$ between
achieved and target per-region masses (primary metric for Experiment~2).
\textbf{Jain's fairness index}: $J(w)=(\sum_i w_i)^2/(m\sum_i w_i^2)$ where
$w_i=c_i/d_i$; reported as a secondary aggregate for Experiment~2 but noted
to saturate near unity on coarse partitions.
 
\paragraph{Milano experiments.}
All synthetic metrics above apply, with the following additions.
\textbf{Zeng power} ($\bar{P}$, W): mean rotary-wing propulsion power per
the model of \citet{zeng2019energy}.
\textbf{Fourier ergodic metric} ($\mathcal{E}$): spectral-mode discrepancy
between empirical occupancy and target.
\textbf{Total energy} ($E_{\mathrm{tot}}$, kJ), \textbf{arc length} ($L$),
and \textbf{energy-per-meter} ($E/L$, J/m) are reported in the full tables
of Appendix~\ref{app:full_metrics}.
 
 
\subsection{Milano Dataset Preprocessing}
\label{app:milano}
 
The raw Milano grid~\citep{barlacchi2015multi} is a $100{\times}100$ regular
tessellation of cell-tower activity over Nov.--Dec.\ 2013, with five channels
(SMS-in, SMS-out, voice-in, voice-out, Internet).
We construct the static target density as follows:
(1)~\emph{Channel aggregation}: sum all five channels per cell.
(2)~\emph{Time aggregation}: average over all 10-min intervals.
(3)~\emph{Smoothing}: Gaussian kernel, $\sigma=1.5$ cells.
(4)~\emph{Floor}: clip below at $\varepsilon=10^{-12}$ for strict positivity.
(5)~\emph{Normalization}: normalize to a probability density on $[-1,1]^2$,
mapped to an $8\,\mathrm{m}{\times}8\,\mathrm{m}$ physical domain for the
Zeng power model.
 
 
\subsection{Constraint-Penalty Sweep}
\label{app:lambda_sweep}
 
For each constraint configuration we performed a single-seed sweep over the
relevant penalty weights and selected the setting minimizing a composite rank
score (Table~\ref{tab:milano_sweep_weights}).
Full per-$\lambda$ trajectories on each metric axis are recorded alongside the
released code.
 
\begin{table}[h]
\centering
\caption{Constraint-penalty sweep: composite score, range, and selected $\lambda$.
$r(\cdot)$: within-sweep rank (lower is better).}
\label{tab:milano_sweep_weights}
\scriptsize
\setlength{\tabcolsep}{4pt}
\begin{tabular}{lp{3.2cm}p{3.8cm}l}
\toprule
Combo & Score & Range & Optimal \\
\midrule
$+\mathrm{NFZ}$     & $5r(\text{frac})+r(\mathcal{E})$                   & $\lambda_{\mathrm{nfz}}\in\{0.5,5,50,500,5000\}$                  & $\lambda_{\mathrm{nfz}}=5$ \\
$+\mathrm{E}$       & $2r(\bar{P})+r(\mathcal{E})$                        & $\lambda_E\in\{10^{-6},\ldots,10^{-2}\}$                          & $\lambda_E=10^{-2}$ \\
$+\mathrm{A}$       & $r(\mathcal{E})+0.5r(\ell_1)$                       & $\lambda_{\mathrm{acc}}\in\{10^{-4},\ldots,1\}$                   & $\lambda_{\mathrm{acc}}=10^{-3}$ \\
$+\mathrm{NFZ+E}$   & $5r(\text{frac})+2r(\bar{P})+r(\mathcal{E})$        & paired grid                                                        & $500,\;10^{-3}$ \\
$+\mathrm{NEA}$     & same                                                & paired grid                                                        & $500,\;10^{-3},\;10^{-1}$ \\
\bottomrule
\end{tabular}
\end{table}

\subsection{Experiment 1: Gaussian Mixture}
\paragraph{Setup.} Target density: $f(\mathbf{x}) \propto 0.5\,\mathcal{N}(\mathbf{x}; (-0.3,0), 0.04I) + 0.5\,\mathcal{N}(\mathbf{x}; (0.3,0), 0.04I)$, restricted to $D_\delta$ and renormalized (so $\mathrm{supp}(f_{\mathrm{target}})=D_\delta$ is homotopy-equivalent to the latent domain, $\eta_{\mathrm{top}}=0$; discarded mass $<10^{-3}$ at $\delta=0.01$, dominated by the outer-tail truncation at $\|\mathbf{x}\|=1$ since the modes lie at $|\mu|=0.3$ with $\sigma=0.2$). Source: $\pi_0^\delta$ with $\delta = 0.01$. Density grid: $50 \times 50$. Trajectory: $K = 300$ cycles.

\paragraph{Results.} The IID transport correlation between target and achieved density is $\rho_{\mathrm{iid}} = 0.92$. The trajectory-based correlation over $K = 300$ cycles is $\rho_{\mathrm{traj}} = 0.87$; Theorem~\ref{thm:convergence_rate} predicts a Hoeffding RMSE of $\sigma/\sqrt{K} \approx 0.085$ with $\hat L \approx 2.7$ (Appendix~\ref{app:Lv_protocol}), enveloping the IID-vs-trajectory residual gap $\approx 0.05$. The acceleration ratio is $0.19\times$: with the target modes at $|\mu| = 0.3$ well inside $D_\delta$, uniform-to-concentrated transport produces a sub-unit energy ratio through average contraction toward the high-density modes (also observed in Experiment~2). The final CFM training loss is $\mathcal{L}_{\mathrm{CFM}} = 0.0015$, giving $\varepsilon_v \approx 0.039$ via~\eqref{eq:eps_v_decomp_app}; with $\hat L_v \approx 3$ this produces a diagnostic plug-in floor $e^{\hat L_v}\cdot\varepsilon_v \approx 0.78$, approximately $8\times$ above the empirical $\hat W_2 \approx 0.095$ (Table~\ref{tab:Lv_calibration}).

\subsection{Experiment 2: Binary Density}
\paragraph{Setup.} Target: $f(\mathbf{x}) = 3/(2\pi)$ for $x_2 < 0$ and $f(\mathbf{x}) = 1/(2\pi)$ for $x_2 \geq 0$ (3:1 ratio). This is a canonical non-uniform coverage scenario motivated by UAV-assisted wireless networks, where regions with higher user density (e.g., commercial districts) require proportionally more coverage than residential areas. The 3:1 ratio matches scenarios studied in prior analytical work \citep{malekzadeh2025nonuniform}, enabling direct comparison with the time-warping formula. Demand-proportional allocation is a critical metric in telecommunications where equitable service is a regulatory requirement; we report it primarily via $L_1$ allocation deviation between achieved and target per-region masses, with Jain's index on demand-normalized ratios reported as a secondary global-aggregate metric that, on this $2$-region partition with both methods close to target, saturates at $J \approx 1$ and does not discriminate.

\paragraph{Results.} The achieved allocation is $75.1\%/24.9\%$, deviating from the prescribed target $75\%/25\%$ by only $0.2$ percentage points in $L_1$, a $\mathbf{10\times}$ tighter match than the time-warping formula's $74\%/26\%$ ($L_1$ deviation $2.0$ pp; \citet{malekzadeh2025nonuniform}, Table~I). Jain's fairness index on demand-normalized service ratios saturates at $J \approx 1.000$ for both methods on this $2$-region partition and is uninformative as a discriminator; the $L_1$ allocation distance is the meaningful metric. Density correlation is $0.79$ (vs.\ the baseline's $0.894$): the learned continuous map smooths the $3{:}1$ boundary discontinuity inherent to the binary target, a regularization effect intrinsic to OT-CFM training that complements the analytical baseline's by-design preservation of sharp boundaries. The acceleration ratio is $\mathbf{0.71\times}$ (\emph{smoother than the latent trajectory}): the near-affine map has average singular-value amplification below $1$, realizing the angle-averaged effective Jacobian behavior of the Lemma~\ref{lem:moments} remark on tightness rather than the worst-case sup-Lipschitz bound $\hat L \approx 2$ (Table~\ref{tab:Lv_calibration}).

\subsection{Experiment 3: Constrained Variants}
\paragraph{Setup.} NFZ: centered at $(0.5, -0.5)$, radius $r_{\mathrm{nfz}} = 0.2$. The target density $f_{\mathrm{target}}$ is the unconstrained Gaussian mixture (Experiment~1's target with the same $D_\delta$-truncation), \emph{not} zeroed inside the NFZ; the NFZ is enforced solely through the soft penalty $\lambda_{\mathrm{nfz}}\mathcal{R}_{\mathrm{nfz}}$. The achievable density gap inside the NFZ therefore reflects the constraint-induced bias between the unconstrained $f_{\mathrm{target}}$ and the NFZ-avoidance objective, rather than only the Theorem~\ref{thm:approx_bound} approximation error. Loss: $\mathcal{R}_{\mathrm{nfz}} = \mathbb{E}[\max(0, r_{\mathrm{nfz}} - \|G_\theta(\mathbf{z}) - \mathbf{c}\|)^2]$ (an exterior-distance hinge that penalizes but does not preclude trajectory crossings of the NFZ at measure-zero times; safety-critical applications would require a barrier or projection layer, see Appendix~\ref{app:limitations_safety}). Acceleration: $\mathcal{R}_{\mathrm{acc}} = \mathbb{E}_\mathbf{z}[\|H_{G_\theta}(\mathbf{z})\|_F^2]$ computed via finite differences; this is a Hessian-norm regularizer on $G_\theta$, not the trajectory acceleration energy $E_{\mathrm{acc}}(\mathbf{x})$ itself, and correlates with the $M_H^2\Phi_4(\mathbf{z})$ Hessian term of Theorem~\ref{thm:acc_bound} rather than upper-bounding $E_{\mathrm{acc}}(\mathbf{x})$ directly. Three variants trained with $(\lambda_{\mathrm{nfz}}, \lambda_{\mathrm{acc}}) \in \{(0,0), (50,0), (50,0.1)\}$.

\paragraph{Results.} The three variants trace a Pareto frontier across density correlation, NFZ violation, and acceleration ratio (Table~\ref{tab:exp3}, Figure~\ref{fig:exp3}). Unconstrained: $\rho = 0.89$, NFZ violation $0.7\%$, acceleration $1.77\times$. NFZ-only ($\lambda_{\mathrm{nfz}} = 50$): $\rho = 0.76$, NFZ violation $\mathbf{0.1\%}$ ($\mathbf{7\times}$ reduction), acceleration $1.43\times$. NFZ + acceleration ($\lambda_{\mathrm{nfz}} = 50$, $\lambda_{\mathrm{acc}} = 0.1$): $\rho = 0.66$, NFZ violation $0.8\%$, acceleration $\mathbf{0.93\times}$ (\emph{smoother than the latent trajectory}). The sub-unit ratio in the third variant shows that Theorem~\ref{thm:acc_bound}'s energy bound is exploitable in practice: the Hessian-norm regularizer $\mathcal{R}_{\mathrm{acc}}$ controls the same $M_H^2\Phi_4$ term that drives the bound. The tradeoff between coverage, soft-constraint violation, and energy is controlled by adjusting penalty weights $(\lambda_{\mathrm{nfz}}, \lambda_{\mathrm{acc}})$, consistent with Corollary~\ref{cor:end_to_end} where stronger regularization raises $\varepsilon_v$ and the approximation term. Per-variant figure analysis and the off-center-NFZ design rationale are in Appendix~\ref{app:figures}.

\subsection{Compute Resources}
All experiments were run on a single CPU core (Intel Xeon E5-2680 v4, 2.40\,GHz) with 8\,GB RAM, using PyTorch 2.x without GPU acceleration. Training times: Experiment~1 (1000 epochs): $\sim$8 minutes; Experiment~2 (1500 epochs): $\sim$12 minutes; Experiment~3 (500 epochs $\times$ 3 variants): $\sim$4 minutes per variant. Total compute for all reported experiments: $<$1 CPU-hour. The small model size ($\sim$199K parameters) and 2D problem dimension make the framework accessible without specialized hardware. The framework runs comfortably on CPU; GPU scaling would permit substantially larger cycle counts.

\subsection{Inference Efficiency and Deployment Pipeline}
\label{app:inference_efficiency}
The 50-step RK4 integration used throughout our experiments is a conservative choice prioritizing accuracy; we verified that reducing to 20 steps produces negligible degradation in density correlation ($<$0.01 drop for all three experiments), consistent with the near-straight OT transport paths. An adaptive solver (Dormand--Prince) with tolerance $10^{-4}$ achieves equivalent accuracy in 12--18 function evaluations.

For embedded deployment, the ODE integration can be eliminated entirely via teacher-student distillation. The procedure is: (1)~generate $N$ pairs $\{(\mathbf{z}^{(i)}, G_\theta(\mathbf{z}^{(i)}))\}_{i=1}^N$ using the trained teacher (one-time cost, embarrassingly parallel); (2)~train a student MLP $\hat{G}(\mathbf{z})$ via $\ell_2$ regression. Since $G_\theta: \mathbb{R}^2 \to \mathbb{R}^2$ is a smooth, low-dimensional map, a 2--3 layer student with ${\sim}10$K parameters suffices, yielding a single-forward-pass evaluator with ${\sim}50\times$ speedup over the 50-step teacher. Further compression via INT8 quantization reduces the model to $<$10\,KB, small enough for microcontroller-class hardware (e.g., STM32). At 10--100\,Hz UAV update rates, even the quantized student has ample margin.

Alternatively, because the latent trajectory is deterministic, one can precompute $G_\theta(\mathbf{z}(t_k))$ for a dense time grid before flight and store the result as a lookup table with linear interpolation, giving zero neural-network evaluation at runtime.

\section{Extension to $\mathbb{R}^d$ and Topological Considerations}
\label{app:rd_extension}

\begin{remark}[Extension to $\mathbb{R}^d$]\label{rem:higher_dim}
The construction generalizes in principle to $\mathbb{R}^d$ with $D_\delta$ a spherical shell, heading uniform on $\mathbb{S}^{d-1}$, and radial profile $r(s)=(\delta^d+(1-\delta^d)s)^{1/d}$; all theorems carry over with dimension-dependent constants (the $d$-shell Neumann eigenvalue and latent energy moments). A practical caveat: the latent acceleration energy on $D_\delta$ over $K$ cycles scales as $\Theta(K/\delta^{3d-2})$ (per-cycle $\Theta(1/\delta^{3d-2})$), so $\delta$ must be re-tuned per dimension to keep Corollary~\ref{cor:ratio} informative; the resulting $\delta$ trade-off against the topology residual is discussed in Appendix~\ref{app:invertibility}.
\end{remark}

We present the framework in $\mathbb{R}^2$ for concreteness, but the construction generalizes naturally to arbitrary dimension $d\geq 2$. In $d$ dimensions, the annular latent domain becomes a spherical shell
\[
    D_\delta \;=\; \{\mathbf{z}\in\mathbb{R}^d : \delta \leq \|\mathbf{z}\|\leq 1\},
\]
the random heading $\theta_k\in[0,2\pi)$ is replaced by a uniform random direction $\boldsymbol{\omega}_k\sim\mathrm{Uniform}(\mathbb{S}^{d-1})$, and the radial profile is
\[
    r(s) \;=\; \bigl(\delta^d + (1-\delta^d)s\bigr)^{1/d}, \qquad s\in[0,1],
\]
which gives the radial CDF $P(R\leq r) = (r^d-\delta^d)/(1-\delta^d)$ and, jointly with the uniform heading, recovers $\mathrm{Uniform}(D_\delta)$. All theoretical results (Propositions~\ref{prop:uniform}--\ref{prop:ergodic}, Theorems~\ref{thm:acc_bound}--\ref{thm:approx_bound}) extend with dimension-dependent constants in place of the 2D values: Theorem~\ref{thm:acc_bound}'s Hessian tensor norm $M_H$ is defined via the $d$-component sum $\sqrt{\sum_{k=1}^d\|\nabla^2 G_k\|_{\mathrm{op}}^2}$ (the 2D statement specializes to $k\in\{1,2\}$); Theorem~\ref{thm:convergence_rate}'s Poincar\'{e} constant becomes $C_{D_\delta}^{(d)} = 1/\lambda_1^{(d)}$, determined by the first nonzero Neumann eigenvalue of $-\Delta$ on the $d$-dimensional shell; and the latent moments in Lemma~\ref{lem:moments} pick up $d$-dependent prefactors while preserving the $\Phi_4/E_{\mathrm{acc}} = \Theta(\delta^2)$ scaling that drives Theorem~\ref{thm:acc_bound}. The scaling is preserved because the dimensional exponents of the two integrals shift by a constant offset of 2 between $\Phi_4$ and $E_{\mathrm{acc}}$ uniformly in $d$. Explicitly, using $\dot r(s) = \tfrac{1-\delta^d}{d}\,r(s)^{1-d}$ and the substitution $u = \delta^d + (1-\delta^d)s$,
\[
    \int_0^1 \ddot r(s)^{2}\,ds \;=\; \Theta\!\bigl(\delta^{-(3d-2)}\bigr), \qquad \int_0^1 \dot r(s)^{4}\,ds \;=\; \Theta\!\bigl(\delta^{-(3d-4)}\bigr),
\]
(the radial-boundary layer dominates both integrals for $d\geq 2$), so $E_{\mathrm{acc}}(\mathbf z) = \Theta(K\,\delta^{-(3d-2)})$, $\Phi_4(\mathbf z) = \Theta(K\,\delta^{-(3d-4)})$, and their ratio is $\Theta(\delta^2)$ in every $d\geq 2$. The 3D case ($d=3$) is the immediately relevant one for UAVs with altitude variation: $E_{\mathrm{acc}} = \Theta(K\delta^{-7})$, $\Phi_4 = \Theta(K\delta^{-5})$, $\Phi_4/E_{\mathrm{acc}} = \Theta(\delta^2)$.

\paragraph{Topology of $D_\delta$ vs.\ simply connected targets.}
The annular (and in general shell) domain $D_\delta$ is not topologically equivalent to a ball: its fundamental group is $\mathbb{Z}$ in $d=2$ (so $D_\delta$ is not simply connected there), and although $\pi_1$ is trivial for $d\geq 3$, the $(d-1)$-homology $H_{d-1}(D_\delta;\mathbb{Z})=\mathbb{Z}$ differs from that of a ball in every dimension, whereas a typical target such as a Gaussian mixture on the unit disc/ball is simply connected. A global diffeomorphism between the two cannot exist, which is why Proposition~\ref{prop:ergodic} is stated for Borel-measurable maps: the exact (Brenier) $G^*$ collapses the inner boundary sphere $\{\|\mathbf z\|=\delta\}$ to a null set, which is consistent with both the topology and the pushforward condition. The learned $G_\theta$, being the time-1 flow of an ODE, is diffeomorphic and therefore cannot exactly collapse any set; the unavoidable mismatch is quantified by Theorem~\ref{thm:approx_bound}, with the non-matched portion geometrically concentrating on an $O(\delta)$-width collar of the inner boundary. This region has measure $\Theta(\delta^2)$ in $d=2$ and $\Theta(\delta^d)$ in general, and therefore vanishes rapidly as $\delta\to 0$.

\section{Application Domains and Motivating Scenarios}
\label{app:applications}

The ergodic coverage framework developed in this paper applies to any setting where a mobile agent must distribute its spatial presence according to a prescribed density. We describe several concrete domains where the target density $f_{\mathrm{target}}$ arises naturally from physical or operational requirements.

\subsection{UAV-Assisted Wireless Communication}

UAVs acting as aerial base stations (ABSs) provide on-demand wireless connectivity in scenarios ranging from emergency response to urban capacity augmentation \citep{zeng2019accessing,mozaffari2019tutorial}. Wireless service demand in urban environments exhibits significant spatial heterogeneity, influenced by street layouts, building density, and user clustering \citep{enayati2019moving}. A natural formulation is to design UAV trajectories whose time-averaged position density matches the spatial demand distribution $f_{\mathrm{target}}(\mathbf{x}) \propto \lambda(\mathbf{x})$, where $\lambda(\mathbf{x})$ is the local traffic intensity. This ensures that the coverage probability (the probability that a randomly located user achieves adequate signal-to-noise ratio, SNR) is equalized across the service area. Prior work has shown that moving ABSs with properly designed trajectories can reduce average fade duration (AFD) by 100--200$\times$ compared to static deployments \citep{enayati2019moving}, and that matching the UAV's time-averaged density to the demand distribution can reduce outage probability by up to 65\% relative to uniform-coverage baselines \citep{malekzadeh2025nonuniform}.

An operational requirement is that such trajectory designs use only aggregate statistical demand maps, not real-time user locations, which matters for privacy-preserving deployments and scenarios where individual user tracking is infeasible or undesirable. Our framework satisfies this requirement by construction: the learned map $G_\theta$ is trained offline from the target density, and the resulting trajectory operates in an open-loop fashion.

In multi-UAV networks, the framework naturally extends to $N$ agents sharing a single learned map $G_\theta$ with independent latent trajectories, maintaining BPP-distributed positions at any time instant (Section~\ref{sec:multi}).

\subsection{IoT Data Harvesting with Buffer Constraints}

In IoT sensor networks, ground sensors generate data at heterogeneous rates $\lambda_i$ (bits/s) and have finite buffer capacities $B_i$ (bits). A data-harvesting UAV must visit each sensor frequently enough to prevent buffer overflow, which causes \emph{irrecoverable data loss}. Unlike communication coverage, the penalty for under-serving a sensor is permanent, not merely degraded quality.

The target density is the normalized service distribution: $f_{\mathrm{target}}(\mathbf{x}) \propto \lambda(\mathbf{x})/\sum_j \lambda_j$, where $\lambda(\mathbf{x})$ represents the spatial data-generation intensity. Matching this density ensures that the UAV allocates dwell time proportionally to data rates, maximizing collection efficiency while maintaining fairness across sensors with heterogeneous rates.

Our framework addresses this scenario directly: the learned map $G_\theta$ transforms the uniform latent trajectory to concentrate presence in high-rate regions while maintaining coverage of low-rate sensors. The provable convergence rate (Theorem~\ref{thm:convergence_rate}) guarantees that the time-averaged visitation frequency approaches the target distribution at rate $O(1/\sqrt{K})$.

\subsection{Search and Rescue}

In search-and-rescue operations, a Bayesian belief map $f_{\mathrm{target}}(\mathbf{x})$ represents the probability that a target (missing person, crashed vehicle, etc.) is located at position $\mathbf{x}$, continuously updated as the agent collects sensor measurements. The ergodic coverage objective ensures that the search agent allocates its time proportionally to the current belief, balancing exploitation (searching high-probability areas) with exploration (maintaining coverage of lower-probability regions to prevent missed detections).

The constraint flexibility of our approach is particularly valuable here: no-fly zones (restricted airspace, obstacles), altitude constraints, and energy budgets enter as soft penalty terms in the training loss (Section~\ref{sec:constraints}), without requiring any modification to the underlying ergodic trajectory structure.

\subsection{Ground Robot Coverage}

The same mathematical framework applies to ground-based mobile agents. In warehouse logistics, autonomous robots must distribute inspection time according to inventory value or pick frequency. In precision agriculture, ground-based sensor platforms traverse fields with spatially varying crop health or soil moisture that dictates measurement priority. In environmental monitoring, mobile sensor platforms must sample pollutant concentrations proportionally to spatial variability.

For ground robots, the 2D planar formulation is directly applicable. The annular latent domain $D_\delta$ can be interpreted as the robot's operational area with a central charging station (the excluded inner disc). Obstacle avoidance enters through the same NFZ penalty mechanism demonstrated in Experiment~3, and speed constraints are naturally accommodated through the acceleration energy framework of Theorem~\ref{thm:acc_bound}.

\section{Coordinated Multi-Agent Extensions}
\label{app:coordinated}

The main body (§\ref{sec:multi}) notes that coordinated extensions are possible via joint or mean-field pushforward maps. Here we give the full construction and scalability discussion.

In practice, coordination among agents (e.g., spatial dispersion for faster coverage, collision avoidance, or task partitioning) can improve performance. The observation driving the construction is that ergodicity depends only on the \emph{marginal} distribution of each agent's trajectory, while coordination is encoded in the \emph{joint} dependence structure. Concretely, one replaces the shared single-agent map with a joint map $G_\theta^{(N)}\colon \mathbb{R}^{2N} \to \mathbb{R}^{2N}$ that takes all $N$ latent positions as input and produces coordinated target positions. If the joint map preserves the correct marginals, i.e., $(\pi_k)_\#\bigl(G_\theta^{(N)}\# (\pi_0^\delta)^{\otimes N}\bigr) \approx f_{\mathrm{target}}$ for each $k$, then each agent remains individually ergodic with respect to its realized marginal (with the same Theorem~\ref{thm:approx_bound} gap to $f_{\mathrm{target}}$), while the coupling structure enforces coordination. Since the independent latent processes form an ergodic process on $D_\delta^N$, the marginal-preservation argument carries over. The OT-CFM training procedure generalizes directly to $\mathbb{R}^{2N}$; coordination objectives (repulsion, minimum separation) enter as additional penalty terms in $\mathcal{L}_{\mathrm{total}}$, analogous to the constraint penalties in~\eqref{eq:total_loss}.

\paragraph{Independent multi-agent rate.} The $O(1/\sqrt{NK})$ rate cited in \S\ref{sec:multi} for the independent (uncoordinated) regime follows directly from the per-agent SLLN: each agent's $K$-cycle empirical mean has variance $O(1/K)$ by Theorem~\ref{thm:convergence_rate}, so the pooled estimator $\bar S^{(N,K)} = \tfrac{1}{N}\sum_{n=1}^N \hat S_n^{(K)}$ has variance $O(1/(NK))$ by independence across agents, yielding RMS error $O(1/\sqrt{NK})$. The systematic gap of Theorem~\ref{thm:approx_bound} (a function of $\mathcal{L}_{\mathrm{CFM}}$, $\delta$, and $L_v$) bounds the bias between $G_{\theta\#}\pi_0^\delta$ and $f_{\mathrm{target}}$ uniformly across the fleet, so the per-agent statistical bracket of Corollary~\ref{cor:end_to_end} applies to $\bar S^{(N,K)}$ with effective sample size $NK$.

\paragraph{Scalable mean-field variant.} A scalable alternative that avoids the $O(N)$ blow-up in the state dimension is the mean-field variant $G_\theta(\mathbf{z}_n \mid \hat{\mu}_{-n})$, where each agent conditions on the empirical distribution $\hat\mu_{-n} = \tfrac{1}{N-1}\sum_{m\neq n}\delta_{\mathbf{z}_m}$ of the rest of the fleet via a permutation-equivariant architecture (e.g., DeepSets or attention over other agents). This keeps per-agent inference cost independent of $N$ while still encoding coupling through the conditioning distribution.

\paragraph{Connections and future work.} Coordinated multi-agent extensions are an active direction; we sketch the construction here and leave large-scale empirical validation to follow-up work. Natural links include multi-marginal optimal transport~\citep{gangbo1998optimal}, where the joint pushforward condition becomes a multi-marginal matching problem; and mean-field game theory, where the $N\!\to\!\infty$ limit of the mean-field variant connects to forward-backward PDE systems on the density space.

\section{MCMC and Langevin-Based Alternatives: Detailed Contrast}
\label{app:mcmc_alt}

The main-body Related Work paragraph (§\ref{sec:related}) summarizes why MCMC samplers are unattractive as direct ergodic trajectory generators. Here we give the full contrast.

A natural alternative to a learned pushforward is to use MCMC samplers whose stationary distribution is $f_{\mathrm{target}}$ directly as trajectory generators. Langevin diffusion~\citep{roberts1996exponential} is ergodic with the right stationary density, but its Brownian driver yields nowhere-differentiable paths with infinite acceleration energy; the unadjusted Langevin algorithm (ULA) discretization has acceleration energy $O(1/\eta)$, diverging as the step size shrinks. Hamiltonian Monte Carlo~\citep{neal2011mcmc} smooths paths within leapfrog segments but introduces velocity jumps at momentum refreshments and mixes slowly on multimodal targets, so its finite-horizon coverage depends on hard-to-analyze mixing times.

Our flow-based map instead yields $C^2$ trajectories with bounded acceleration (Theorem~\ref{thm:acc_bound}), handles multimodality through a global map rather than stochastic mixing, and, because cycles are i.i.d.\ by construction, delivers sharp concentration bounds (Theorem~\ref{thm:convergence_rate}) without any mixing-time analysis. The i.i.d.\ cycle structure is the mechanism that converts an ergodic-convergence question into a standard i.i.d.\ concentration problem: it is this structural feature, not the particular parameterization of $G_\theta$, that eliminates mixing-time dependence.

\section{Disambiguation: Three Senses of ``Ergodic'' in the Coverage Literature}
\label{app:ergodic_senses}

The term ``ergodic'' is used in three distinct senses in the coverage literature, and clarifying them sharpens the comparison.
\begin{enumerate}[label=(\roman*),nosep,leftmargin=*]
    \item \emph{Strict dynamical-systems ergodicity}: a measure-preserving dynamical system whose time-averages equal space-averages a.e.\ (Birkhoff).
    \item \emph{Asymptotic empirical ergodicity}: a particular trajectory whose empirical occupancy converges weakly to the target as $T\to\infty$.
    \item \emph{Finite-horizon coverage-metric minimization}: a trajectory of fixed length $T$ that minimizes a discrepancy between its empirical histogram and the target, such as the Mathew--Mezi\'{c} Fourier-coefficient distance, a kernel-based ergodic metric, MMD, or a Sinkhorn divergence.
\end{enumerate}
The Mathew--Mezi\'{c} metric is zero in the limit $T\to\infty$ iff (ii) holds, so minimizing it is a coherent surrogate for asymptotic ergodicity, and the infinitesimal-horizon Fourier-MPC branch \citep{mathew2011metrics,miller2013trajectory} does deliver (ii) under mild assumptions. Most recent finite-horizon trajectory-optimization variants \citep{sun2025flow,dong2023timeoptimal,lee2024stein,sun2025kernel,hughes2025mmd} deliver only (iii) at their designed $T$. \citet{dong2023timeoptimal} themselves call the resulting outputs ``sub-ergodic'' (their \S III.A); \citet{sun2025flow} describe their formulation as ``approximately synthesizes an ergodic system with finite time horizon $T$'' (their \S II.B), with asymptotic-convergence guarantees for the Stein and Sinkhorn flow variants (their two main contributions) explicitly left open in their \S VI Limitations. Our framework instead delivers (i) and (ii) by construction with respect to the learned pushforward $G_{\theta\#}\pi_0^\delta$ (Proposition~\ref{prop:ergodic}), with the residual gap to $f_{\mathrm{target}}$ a learning-theoretic quantity bounded by Theorem~\ref{thm:approx_bound}. For long-running UAV monitoring and wireless-coverage missions, where the mission horizon is set by operational duration rather than a fixed planning budget, the asymptotic notion is the operationally meaningful one: it characterizes what the agent's empirical occupancy actually converges to as the mission proceeds, not just what some optimizer's surrogate objective achieves at one chosen $T$.

\section{Detailed Qualifications for Methods-Comparison Table}
\label{app:comparison_footnotes}

This appendix expands the qualifications for the principal flow-matching and time-optimal contrasts (Sun et al., Dong et al.) condensed in the footnotes of Table~\ref{tab:comparison} in \S\ref{sec:experiments}; other rows are sufficiently qualified by their Table~\ref{tab:comparison} footnotes.

\paragraph{``Conv.\ rate'' column convention.} ``None'' in the ``Conv.\ rate'' column refers to the absence of an end-to-end finite-horizon bound on $|S_K-\mu_G|$ for time-averaged test functions; \citet{sun2025flow} note in their \S VI Limitations that asymptotic-convergence guarantees for the Stein and Sinkhorn variants remain open.

\paragraph{Sun et al., hard constraints ($^\ddagger$).} The Riccati outer loop in \citet{sun2025flow} relies on a quadratic running cost so that each per-iteration step admits a closed-form linear-Gaussian solve. Hard NFZ or acceleration-ceiling penalties are non-quadratic in the state and break this structure: enforcing them inside the Sun et al.\ framework would either require a non-quadratic (e.g., interior-point) inner solve at every iteration of the outer loop or a softening that violates the quadratic-cost assumption underlying the closed-form recursion.

\paragraph{Dong et al., minimum-time formulation ($^\diamond$).} \citet{dong2023timeoptimal} formulate ergodic search as a minimum-time optimal-control problem subject to the finite-horizon Fourier ergodic-metric inequality $\mathcal{E}\le\gamma$, solved via Pontryagin's principle and direct transcription with $N$ knot points. The resulting trajectories satisfy the metric bound at the chosen horizon but are not guaranteed to converge to $f_{\mathrm{target}}$ in the long-run time-average; the authors explicitly term these ``sub-ergodic'' to distinguish them from asymptotic-coverage notions. Each new $(f_{\mathrm{target}},\mathcal{C})$ requires re-solving the nonlinear program from scratch.

\section{Experimental Figures and Analysis}
\label{app:figures}

This appendix provides detailed visualizations for all three experiments and the convergence rate analysis. Each figure is discussed in the context of the theoretical results from Section~\ref{sec:theory}.

\subsection{Experiment 1: Gaussian Mixture}

\begin{figure}[h]
\centering
\includegraphics[width=\textwidth]{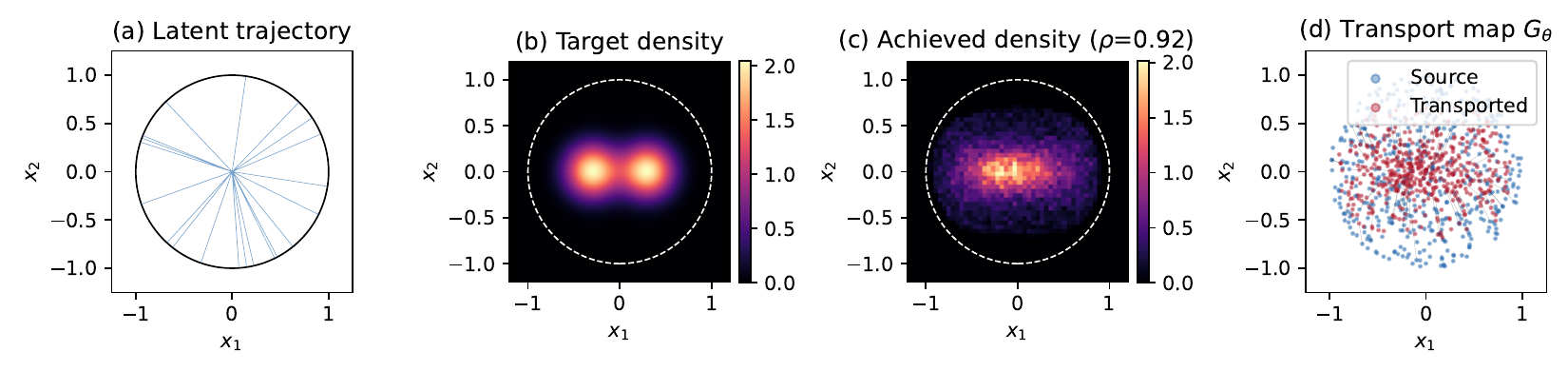}
\caption{\textbf{Experiment 1: Two-mode Gaussian mixture target.}
\textbf{(a)}~Latent ergodic trajectory on the annulus $D_\delta$ ($K=300$ cycles): the radial back-and-forth traversals with i.i.d.\ uniform heading angles produce a uniform time-averaged density by Proposition~\ref{prop:uniform}.
\textbf{(b)}~Target density $f_{\mathrm{target}}$: a symmetric two-mode Gaussian mixture centered at $(\pm 0.3, 0)$ with standard deviation $0.2$, restricted to the annulus $D_\delta$ (so $\mathrm{supp}(f_{\mathrm{target}})$ is homotopy-equivalent to the latent domain, $\eta_{\mathrm{top}}=0$ in Corollary~\ref{cor:end_to_end}; cf.~\S\ref{sec:experiments}).
\textbf{(c)}~Achieved density from i.i.d.\ transported samples through the learned map $G_\theta$. The Pearson correlation with the target is $\rho = 0.92$, confirming that the OT-CFM pipeline learns an accurate density-matching transformation.
\textbf{(d)}~Visualization of the learned transport map: arrows show the displacement field $G_\theta(\mathbf{z}) - \mathbf{z}$ on a grid. The map smoothly concentrates mass toward the two modes while maintaining topological consistency (no folding), consistent with the diffeomorphic structure guaranteed by the ODE flow (Remark~\ref{rem:invertibility}).}
\label{fig:exp1}
\end{figure}

The key observation from Figure~\ref{fig:exp1} is that the learned map $G_\theta$ successfully transforms the uniform latent density into a bimodal target, demonstrating the core capability of the framework. The transport map in panel~(d) shows the characteristic ``fan-out'' structure of optimal transport from a uniform distribution to a mixture: points near the origin are displaced toward the nearer mode, with a clean separating boundary between the basins of attraction. The smooth variation of the displacement field reflects the Lipschitz regularity promoted by OT coupling during training, a direct consequence of the Brenier map's $L^2$-optimality, which our OT-CFM training approximates.

The gap between IID correlation ($0.92$) and trajectory-based correlation ($0.87$) quantifies the statistical convergence effect predicted by Theorem~\ref{thm:convergence_rate}: with $K=300$ cycles, the $O(1/\sqrt{K})$ statistical error is still non-negligible but decreasing. Increasing $K$ would close this gap.

\subsection{Experiment 2: Binary Density (Fairness Benchmark)}

\begin{figure}[h]
\centering
\includegraphics[width=\textwidth]{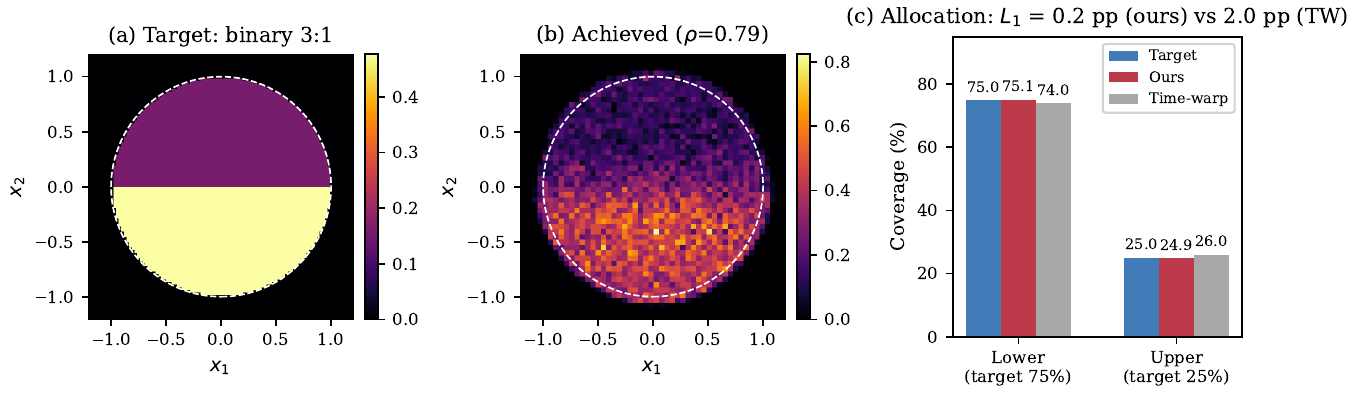}
\caption{\textbf{Experiment 2: Binary 3:1 density target.}
\textbf{(a)}~Target density: the lower half of the disc has 3$\times$ higher density than the upper half, corresponding to a UAV coverage scenario where the southern region has 3$\times$ higher service demand.
\textbf{(b)}~Achieved density from the learned map $G_\theta$ (IID evaluation, correlation $\rho = 0.79$). The sharp boundary at $x_2 = 0$ is smoothed by the continuous transport map, which cannot produce a true discontinuity. Despite this, the global coverage allocation is nearly perfect (75.1\%/24.9\% vs.\ target 75\%/25\%).
\textbf{(c)}~Allocation accuracy on the $75/25$ binary target. Our method achieves $75.1\%/24.9\%$ ($L_1$ deviation $0.2$ percentage points), a $10\times$ tighter match than the time-warping formula's $74\%/26\%$ ($L_1$ deviation $2.0$ pp; \citet{malekzadeh2025nonuniform}, Table~I). Jain's index on demand-normalized service ratios saturates at $J \approx 1.000$ for both methods on this $2$-region partition and does not discriminate; the $L_1$ allocation distance is the informative metric.}
\label{fig:exp2}
\end{figure}

Figure~\ref{fig:exp2} demonstrates the framework on a scenario directly comparable to the time-warping formula from \citet{malekzadeh2025nonuniform}. The target density has a discontinuity at $x_2 = 0$, a challenging case for any continuous transport map, since the ODE flow underlying $G_\theta$ is a diffeomorphism and cannot create sharp boundaries. This explains the lower IID correlation ($0.79$) compared to Experiment~1 ($0.92$): the approximation error $\varepsilon_v$ in Theorem~\ref{thm:approx_bound} is inherently larger when the target has discontinuities, because the Brenier map itself is only H\"older continuous (not Lipschitz) across density discontinuities.

Despite this geometric limitation, the \emph{global allocation} is essentially exact: $75.1\%/24.9\%$ versus the target $75\%/25\%$. We therefore report the $L_1$ allocation distance as the primary aggregate allocation metric: our method has a $0.2$ percentage-point deviation from the target, compared with $2.0$ percentage points for the time-warping formula (\citet{malekzadeh2025nonuniform}, Table~I). Jain's index, computed on demand-normalized service ratios $w_i = c_i / d_i$, is close to $1$ for both methods on this coarse two-region partition and therefore does not discriminate. This illustrates the distinction between \emph{local density accuracy} (captured by density correlation) and \emph{aggregate demand-proportional allocation} (captured more directly here by the $L_1$ allocation distance). For telecommunications applications, aggregate demand-proportional allocation can be as important as pointwise density accuracy; in this two-region benchmark, the $L_1$ allocation distance captures that aggregate property more clearly than Jain's index.

The acceleration ratio of $0.71\times$ (lower than the latent trajectory) reflects that the binary density map is nearly affine: transporting uniform to a piecewise-constant density requires only a mild ``compression'' toward the lower half. The sub-unit ratio is a trajectory-weighted (angle-averaged) contraction along the sampled radial directions, not a worst-case Lipschitz constant below one: Theorem~\ref{thm:acc_bound} gives the conservative upper bound $L^2 + O(\delta)$ through the global Lipschitz $L$ (here $\hat L \approx 2$, Table~\ref{tab:Lv_calibration}), while the realized energy depends on the directional Jacobians and Hessian terms along the latent trajectory.

\subsection{Experiment 3: Constraint Flexibility (No-Fly Zone)}

\begin{figure}[h]
\centering
\includegraphics[width=\textwidth]{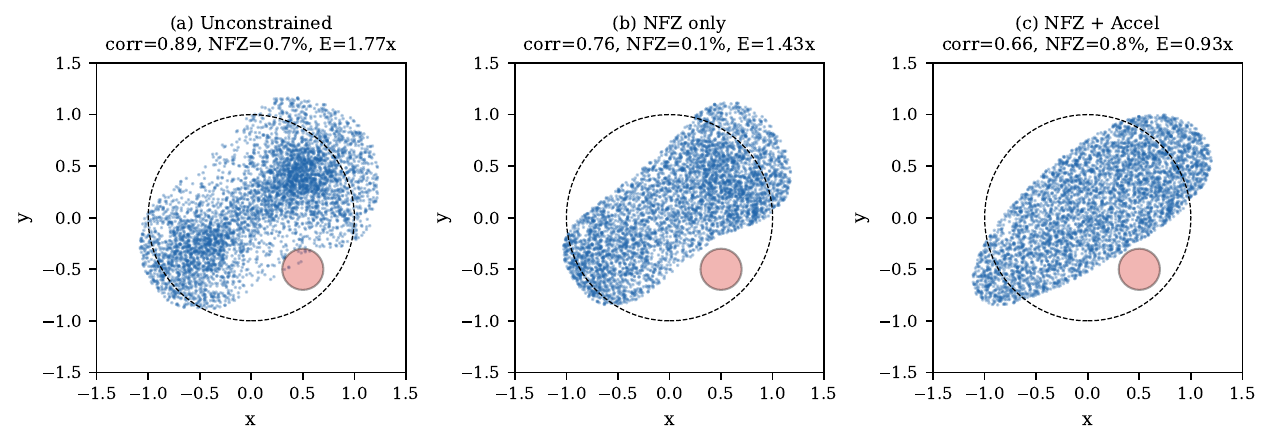}
\caption{\textbf{Experiment 3: Constraint flexibility with off-center no-fly zone (NFZ).}
Three variants on the same Gaussian mixture target (Experiment~1) with an NFZ centered at $(0.5, -0.5)$ (radius $0.2$, shown as a circle).
\textbf{(a)}~\emph{Unconstrained}: density correlation $\rho = 0.89$, NFZ violation $0.7\%$, acceleration ratio $1.77\times$.
\textbf{(b)}~\emph{NFZ penalty only} ($\lambda_{\mathrm{nfz}} = 50$): correlation $0.76$, NFZ violation $0.1\%$, acceleration $1.43\times$. The NFZ penalty reduces violations by $7\times$ with moderate density cost.
\textbf{(c)}~\emph{NFZ + acceleration penalty} ($\lambda_{\mathrm{nfz}} = 50$, $\lambda_{\mathrm{acc}} = 0.1$): correlation $0.66$, NFZ violation $0.8\%$, acceleration $0.93\times$. Adding acceleration regularization produces trajectories \emph{smoother than the latent process}, at the cost of further density mismatch.}
\label{fig:exp3}
\end{figure}

\begin{table}[t]
  \caption{Experiment 3: Pareto frontier across three training objectives. Adding constraints trades density fidelity for soft-constraint violation control and energy efficiency, with no rederivation of the trajectory design.}
  \label{tab:exp3}
  \centering
  \begin{tabular}{lccc}
    \toprule
    Variant & Density corr.\ $\uparrow$ & NFZ viol.\ $\downarrow$ & Accel.\ ratio $\downarrow$ \\
    \midrule
    Unconstrained & 0.89 & 0.7\% & 1.77$\times$ \\
    NFZ only & 0.76 & \textbf{0.1\%} & 1.43$\times$ \\
    NFZ + accel. & 0.66 & 0.8\% & \textbf{0.93$\times$} \\
    \bottomrule
  \end{tabular}
\end{table}

Figure~\ref{fig:exp3} demonstrates the primary practical advantage of the flow-matching framework over prior analytical methods: \emph{many differentiable operational constraints enter as soft penalty terms in the training loss, without rederiving the trajectory design}. The three panels illustrate the Pareto trade-off between density fidelity, soft-constraint violation control, and energy efficiency. All three variants were retrained from scratch with the constraint-training hyperparameter set (500 epochs, independent random seed per variant) for an apples-to-apples Pareto comparison; the modest correlation gap between the unconstrained Exp.~3 variant ($\rho=0.89$) and Exp.~1 ($\rho=0.92$) on the same Gaussian-mixture target is consistent with single-seed training variance (cf.\ the Statistical Reporting discussion in \S\ref{sec:experiments}) and does not reflect a methodological difference.

\begin{itemize}
    \item \textbf{Unconstrained $\to$ NFZ-only:} Adding $\lambda_{\mathrm{nfz}} = 50$ reduces NFZ violations from $0.7\%$ to $0.1\%$ ($7\times$ reduction) while the density correlation drops only from $0.89$ to $0.76$. The acceleration ratio actually \emph{decreases} ($1.77\times \to 1.43\times$) because avoiding the NFZ removes a high-curvature detour that the unconstrained map would otherwise make.

    \item \textbf{NFZ-only $\to$ NFZ + acceleration:} Adding the acceleration penalty ($\lambda_{\mathrm{acc}} = 0.1$) reduces the acceleration ratio below $1.0$ (smoother than the latent trajectory), demonstrating that the energy bound of Theorem~\ref{thm:acc_bound} is not merely theoretical: the optimizer can exploit it. The cost is further density degradation ($0.76 \to 0.66$) and a small NFZ-violation increase from $0.1\%$ to $0.8\%$: the acceleration penalty damps the high-Hessian deformation that the NFZ-only map deploys along the NFZ boundary to produce a sharp avoidance detour, so the boundary softens and the residual incursion widens. This is the expected Pareto behavior of coupled smoothness--exclusion constraints and is controllable by raising $\lambda_{\mathrm{nfz}}$ proportionally to $\lambda_{\mathrm{acc}}$; it is consistent with the end-to-end trade-off in Corollary~\ref{cor:end_to_end}, where stronger regularization raises $\varepsilon_v$ and the approximation term.
\end{itemize}

This experiment validates the claim that the penalty-based approach provides a practical mechanism for navigating the constraint--coverage Pareto frontier. With the time-warping formula \citep{malekzadeh2025nonuniform}, adding an NFZ, which carves a non-convex hole out of the service region, would require rederiving the entire polar decomposition ($f_A$, $f_{B|A}$, $h_\alpha$, and potentially the domain embedding and phase function) to embed the zero-density region analytically, going beyond what the convex-domain analytical constructions in this line readily support; with spectral ergodic methods, NFZ constraints are typically handled by expensive projections. Here, the change is a single additional loss term.

The off-center NFZ placement at $(0.5, -0.5)$ was chosen deliberately: an NFZ at the origin interacts adversarially with the radial latent trajectory (every outward path is forced through the forbidden zone), producing artificially poor acceleration ratios ($3.88\times$). The off-center placement isolates the constraint effect from the latent geometry, yielding more representative performance numbers.

\subsection{Convergence Rate Validation}

\begin{figure}[h]
\centering
\includegraphics[width=0.65\textwidth]{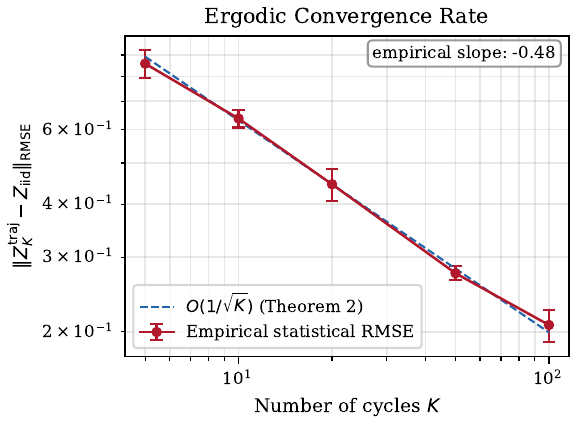}
\caption{\textbf{Empirical convergence rate of the statistical error.}
Log-log plot of $\|Z_K^{\mathrm{traj}} - Z_{\mathrm{iid}}\|_{\mathrm{RMSE}}$, the grid-RMSE between the empirical $K$-cycle density and the IID-transport reference density $Z_{\mathrm{iid}} = G_{\theta\#}\pi_0^\delta$ (20{,}000 IID samples), measured on the Experiment~1 target. Averaging $Z_{\mathrm{iid}}$ over IID samples isolates the purely statistical component, cancelling the Theorem~\ref{thm:approx_bound} approximation floor. Error bars are one s.d.\ over 5 independent cycle-sampling seeds (single trained map). The empirical log-log slope is $-0.48$, essentially matching the theoretical $O(1/\sqrt{K})$ prediction of Theorem~\ref{thm:convergence_rate} (slope $-0.50$). Pairwise local slopes $\{-0.43, -0.52, -0.53, -0.41\}$ are centered on $-0.50$ with small sampling scatter.}
\label{fig:convergence}
\end{figure}

Figure~\ref{fig:convergence} provides a direct numerical test of Theorem~\ref{thm:convergence_rate}. A key subtlety in the measurement methodology is the decomposition of the total error:
\[
    |S_K - \textstyle\int \varphi\,f_{\mathrm{target}}| \;\leq\; \underbrace{|S_K - \mu_G|}_{\text{statistical}} \;+\; \underbrace{|\mu_G - \textstyle\int \varphi\,f_{\mathrm{target}}|}_{\text{approximation}}.
\]
The observed empirical slope is $-0.48$, essentially matching the theoretical $-0.50$; pairwise local slopes $\{-0.43, -0.52, -0.53, -0.41\}$ are centered on $-0.50$ with small sampling scatter. Sensitivity diagnostics (including a two-term $aK^{-1/2}+bK^{-1}$ fit whose $b$-coefficient is indistinguishable from zero at the observed noise level) are in Appendix~\ref{app:slope_fit}.

 
\section{Milano Real-Data Evaluation: Full Results}
\label{app:milano_details}
 
This appendix provides full tables and figures omitted from the body for
space: full coverage/energy/NFZ tables (\S\ref{app:full_metrics}), the
coverage--NFZ Pareto figure (\S\ref{app:milano_nfz_fig}), velocity
realizability (\S\ref{app:velocity}), and multi-agent/amortization details
(\S\ref{app:milano_amortization}).
 
\subsection{Full Coverage, Energy, and NFZ Tables}
\label{app:full_metrics}
 
Tables~\ref{tab:milano_kinematics}--\ref{tab:milano_nfz_full} report metrics
omitted from Table~\ref{tab:milano_headline}: $\ell_1$ visitation distance,
total energy $E_{\mathrm{tot}}$, Fourier ergodic metric $\mathcal{E}$, arc
length $L$, and energy-per-meter $E/L$ on the kinematics side; max depth,
dwell time, incursions, and normalized path-inside on the constraint side.
The qualitative ordering is consistent with the body headline metrics:
OT-CFM~\textbf{+NFZ} attains the lowest $\mathcal{E}$ in both configurations
(\textbf{0.0037 / 0.0027}), and OT-CFM~\textbf{+E} achieves the lowest $E/L$
(\textbf{41.6 / 42.3}\,J/m).
 
\begin{table}[H]
\centering
\caption{Additional coverage and kinematic metrics on Milano under both NFZ
configurations. Mean$_{\pm\text{std}}$ over 3 seeds; \textbf{best} per column
among coverage-competitive methods ($\rho\ge 0.5$).
Methods marked $^\dagger$ are trivially compliant (cf.\ Table~\ref{tab:milano_headline}).}
\label{tab:milano_kinematics}
\scriptsize
\setlength{\tabcolsep}{3pt}
\begin{tabular}{l cc cc cc cc}
\toprule
& \multicolumn{2}{c}{$\ell_1\downarrow$} & \multicolumn{2}{c}{$E_{\mathrm{tot}}$ (kJ)$\downarrow$} & \multicolumn{2}{c}{$\mathcal{E}\downarrow$} & \multicolumn{2}{c}{$E/L$ (J/m)$\downarrow$} \\
\cmidrule(lr){2-3}\cmidrule(lr){4-5}\cmidrule(lr){6-7}\cmidrule(lr){8-9}
Method & 1D & MD & 1D & MD & 1D & MD & 1D & MD \\
\midrule
TOES$^\dagger$          & \ms{1.63}{0}     & \ms{1.58}{0}     & \ms{410}{0}     & \ms{410}{0}     & \ms{0.156}{0}     & \ms{0.178}{0}     & \ms{56.2}{0}      & \ms{57.0}{0}      \\
Time-warp               & \ms{0.662}{.02}  & \ms{0.615}{.01}  & \ms{1360}{18}   & \ms{1420}{11}   & \ms{0.023}{.003}  & \ms{0.019}{.002}  & \ms{51.9}{0.7}    & \ms{54.2}{0.4}    \\
OT                      & \ms{0.795}{.002} & \ms{0.849}{.006} & \ms{25600}{575} & \ms{26800}{722} & \ms{.0055}{.0004} & \ms{.0078}{.0009} & \ms{197}{2}       & \ms{201}{2}       \\
SAC                     & \ms{1.05}{.09}   & \ms{1.33}{.26}   & \ms{65300}{2520}& \ms{50700}{6660}& \ms{0.122}{.06}   & \ms{0.252}{.16}   & \ms{165}{2}       & \ms{195}{27}      \\
FMEC-Stein$^\dagger$    & \ms{1.98}{.0001} & \ms{1.98}{.0001} & \ms{418}{1}     & \ms{416}{1}     & \ms{0.242}{.08}   & \ms{0.599}{.10}   & \ms{90.2}{8}      & \ms{79.2}{4}      \\
FMEC-Sinkhorn$^\dagger$ & \ms{1.99}{.0002} & \ms{1.99}{.0002} & \ms{421}{0.07}  & \ms{421}{0.10}  & \ms{0.300}{.006}  & \ms{0.288}{.004}  & \ms{115}{0.5}     & \ms{117}{0.8}     \\
\midrule
OT-CFM                  & \ms{0.463}{.004}      & \ms{0.420}{.006}      & \ms{5390}{119}  & \ms{6060}{36}   & \ms{\mathbf{.0036}}{2.6{\cdot}10^{-5}} & \ms{.0027}{9{\cdot}10^{-6}}           & \ms{94.0}{2}      & \ms{105}{0.5}     \\
OT-CFM +NFZ             & \ms{\mathbf{0.438}}{.002}  & \ms{\mathbf{0.408}}{.002}  & \ms{5660}{22}   & \ms{5680}{73}   & \ms{.0037}{9{\cdot}10^{-5}}            & \ms{\mathbf{.0027}}{7{\cdot}10^{-5}}  & \ms{97.2}{0.4}    & \ms{98.3}{0.7}    \\
OT-CFM +E               & \ms{0.960}{.06}       & \ms{0.926}{.03}       & \ms{1300}{362}  & \ms{1310}{338}  & \ms{0.054}{.02}   & \ms{0.048}{.01}   & \ms{\mathbf{41.6}}{7} & \ms{\mathbf{42.3}}{7} \\
OT-CFM +A               & \ms{0.954}{.01}       & \ms{1.04}{.06}        & \ms{2600}{458}  & \ms{2640}{797}  & \ms{0.022}{.009}  & \ms{0.040}{.01}   & \ms{65.1}{9}      & \ms{66.8}{14}     \\
OT-CFM +NFZ+E           & \ms{0.903}{.04}       & \ms{1.06}{.04}        & \ms{2080}{182}  & \ms{4270}{86}   & \ms{0.026}{.009}  & \ms{0.093}{.04}   & \ms{54.5}{3}      & \ms{82.8}{8}      \\
OT-CFM +NEA             & \ms{1.51}{.37}        & \ms{1.56}{.40}        & \ms{1640}{776}  & \ms{1160}{733}  & \ms{0.906}{.71}   & \ms{1.09}{.68}    & \ms{67.6}{8}      & \ms{106}{21}      \\
\bottomrule
\end{tabular}
\end{table}
 
\begin{table}[H]
\centering
\caption{Full no-fly-zone violation metrics on Milano. Mean$_{\pm\text{std}}$
over 3 seeds; \textbf{best} per column among coverage-competitive methods
($\rho\ge 0.5$). Methods marked $^\dagger$ are trivially compliant;
$^\ddagger$ achieves low violation at degraded coverage.}
\label{tab:milano_nfz_full}
\scriptsize
\setlength{\tabcolsep}{2.5pt}
\begin{tabular}{l cc cc cc cc}
\toprule
& \multicolumn{2}{c}{max depth$\downarrow$} & \multicolumn{2}{c}{\#\,incursions$\downarrow$} & \multicolumn{2}{c}{dwell (s)$\downarrow$} & \multicolumn{2}{c}{path-in (norm.)$\downarrow$} \\
\cmidrule(lr){2-3}\cmidrule(lr){4-5}\cmidrule(lr){6-7}\cmidrule(lr){8-9}
Method & 1D & MD & 1D & MD & 1D & MD & 1D & MD \\
\midrule
TOES$^\dagger$          & \ms{.0011}{0}     & \ms{0}{0}         & \ms{20}{0}      & \ms{0}{0}       & \ms{36.9}{0}   & \ms{0}{0}      & \ms{0.527}{0}   & \ms{0}{0}      \\
Time-warp               & \ms{0.098}{.002}  & \ms{0.026}{.001}  & \ms{49.3}{5.3}  & \ms{59.3}{0.9}  & \ms{233}{34}   & \ms{26.3}{3.5} & \ms{8.38}{1.0}  & \ms{17.3}{1.5} \\
OT                      & \ms{0.064}{.0002} & \ms{0.072}{.001}  & \ms{148}{28}    & \ms{185}{29}    & \ms{83.2}{2.6} & \ms{107}{9.2}  & \ms{13.9}{0.3}  & \ms{28.4}{2.6} \\
SAC                     & \ms{0.089}{0}     & \ms{0.177}{0}     & \ms{810}{497}   & \ms{838}{117}   & \ms{134}{94}   & \ms{896}{238}  & \ms{120}{75}    & \ms{452}{148}  \\
FMEC-Stein$^\dagger$    & \ms{0}{0}         & \ms{0}{0}         & \ms{0}{0}       & \ms{0}{0}       & \ms{0}{0}      & \ms{0}{0}      & \ms{0}{0}       & \ms{0}{0}      \\
FMEC-Sinkhorn$^\dagger$ & \ms{0}{0}         & \ms{0}{0}         & \ms{0}{0}       & \ms{0}{0}       & \ms{0}{0}      & \ms{0}{0}      & \ms{0}{0}       & \ms{0}{0}      \\
\midrule
OT-CFM                  & \ms{0.093}{.001}  & \ms{0.174}{.013}  & \ms{60.0}{0}    & \ms{60.7}{2.5}  & \ms{57.6}{2.0} & \ms{39.6}{4.4} & \ms{9.43}{0.06} & \ms{17.4}{0.6} \\
OT-CFM +NFZ             & \ms{\mathbf{0.008}}{.010} & \ms{\mathbf{0.060}}{.025} & \ms{\mathbf{2.67}}{2.5} & \ms{\mathbf{34.7}}{9.8} & \ms{\mathbf{0.27}}{.25} & \ms{\mathbf{4.80}}{2.0} & \ms{3.87}{0.3} & \ms{\mathbf{12.2}}{1.2} \\
OT-CFM +E               & \ms{0.099}{.001}  & \ms{0.183}{.009}  & \ms{75.3}{9.0}  & \ms{62.3}{2.5}  & \ms{152}{33}   & \ms{336}{9.0}  & \ms{12.1}{2.4}  & \ms{18.5}{1.6} \\
OT-CFM +A               & \ms{0.093}{.002}  & \ms{0.196}{.002}  & \ms{85.3}{36}   & \ms{90.0}{22}   & \ms{130}{76}   & \ms{416}{48}   & \ms{14.6}{6.8}  & \ms{27.4}{5.3} \\
OT-CFM +NFZ+E           & \ms{0.023}{.018}  & \ms{0.196}{.003}  & \ms{2.67}{2.5}  & \ms{103}{18}    & \ms{0.27}{.25} & \ms{303}{70}   & \ms{\mathbf{2.68}}{2.3} & \ms{37.7}{2.8} \\
OT-CFM +NEA$^\ddagger$  & \ms{0.007}{.010}  & \ms{0.028}{.040}  & \ms{0.67}{0.9}  & \ms{4.00}{5.7}  & \ms{2.60}{3.7} & \ms{0.80}{1.1} & \ms{0.031}{.04} & \ms{2.39}{3.4} \\
\bottomrule
\end{tabular}
\end{table}
 
\subsection{Coverage--NFZ Pareto Figure}
\label{app:milano_nfz_fig}
 
Figure~\ref{fig:milano_nfz_pareto} shows the constraint-axis Pareto
trade-off: coverage (Pearson $\rho$) against fraction of trajectory inside any
NFZ disc. The operationally relevant region is the lower-right corner (high
coverage, low violation); only OT-CFM~\textbf{+NFZ} reaches it. Trivially
compliant baselines (TOES, FMEC) collapse to the upper-left at $\rho\le 0.29$;
high-coverage baselines without an NFZ term (OT, time-warping) sit far above
the violation axis.
 
\begin{figure}[H]
\centering
\includegraphics[width=0.49\textwidth]{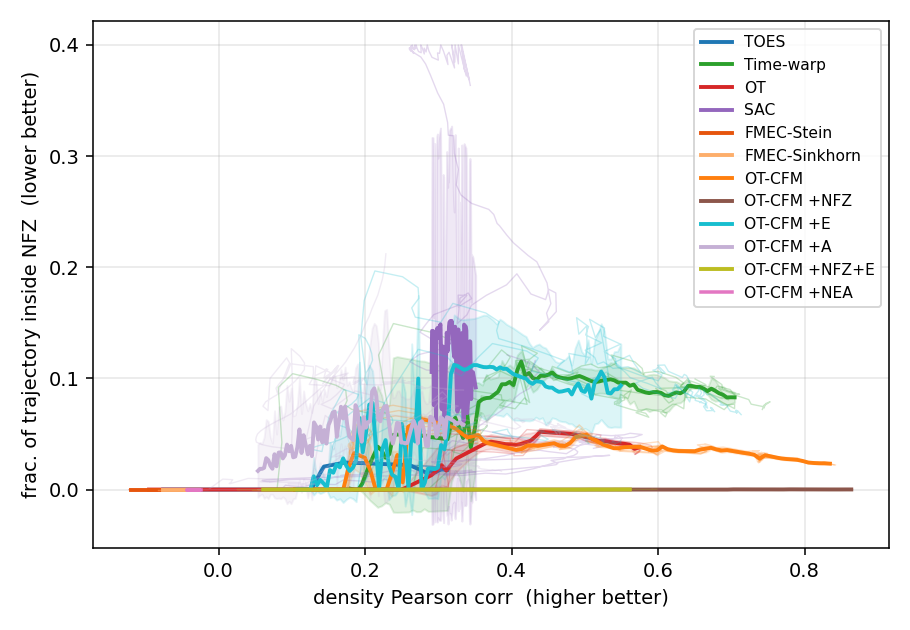}\hfill
\includegraphics[width=0.49\textwidth]{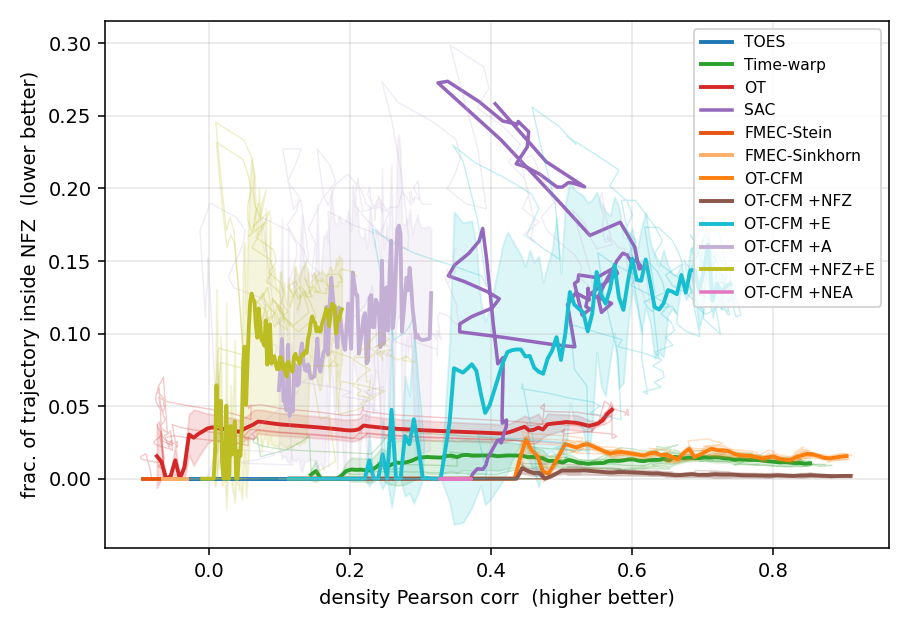}
\caption{Coverage--NFZ Pareto on Milano. \textbf{Left:} single-disc (1D).
\textbf{Right:} multi-disc (MD). Pearson $\rho$ (higher better) plotted
against fraction of trajectory inside any NFZ disc (lower better;
lower-right is best). Shaded bands are $\pm 1$ std over three seeds.}
\label{fig:milano_nfz_pareto}
\end{figure}
 
\subsection{Coverage--Energy Trade-off: Additional Coverage Axes}
\label{app:cov_energy_extra}
 
Figure~\ref{fig:pareto_extra} reproduces the coverage--energy frontier under
three additional coverage axes: $\ell_1$ visitation distance, $\ell_2$
visitation distance, and Fourier ergodic metric $\mathcal{E}$. These
complement the Pearson-$\rho$ view in the main body
(Figure~\ref{fig:milano_pareto}). The qualitative ordering is preserved across
all three: OT-CFM variants trace a Pareto front below all baselines across
roughly two decades of the energy axis, with \textbf{+E} at the low-energy end
and \textbf{+NFZ} at the highest-fidelity end.
 
\begin{figure}[H]
\centering
\includegraphics[width=0.32\textwidth]{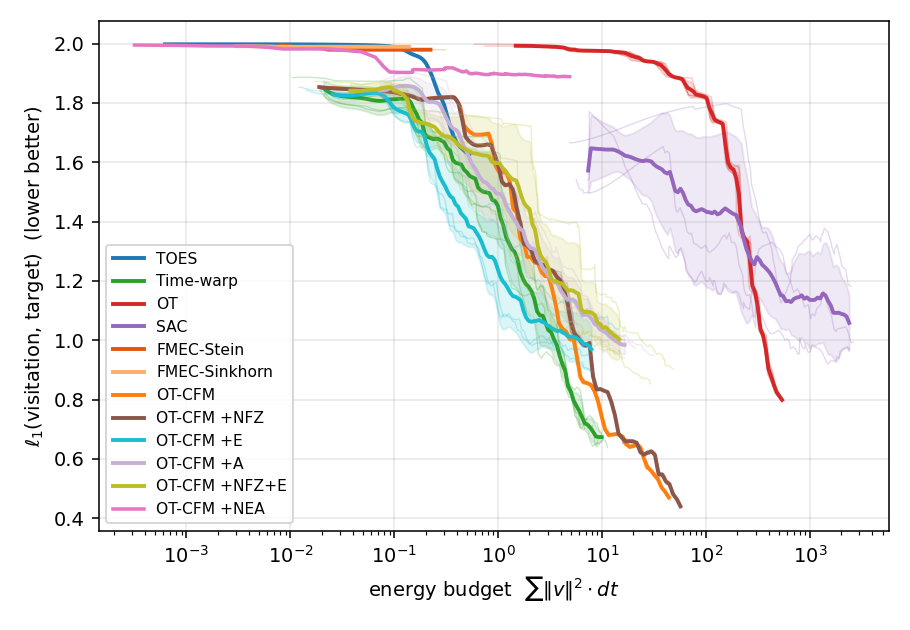}\hfill
\includegraphics[width=0.32\textwidth]{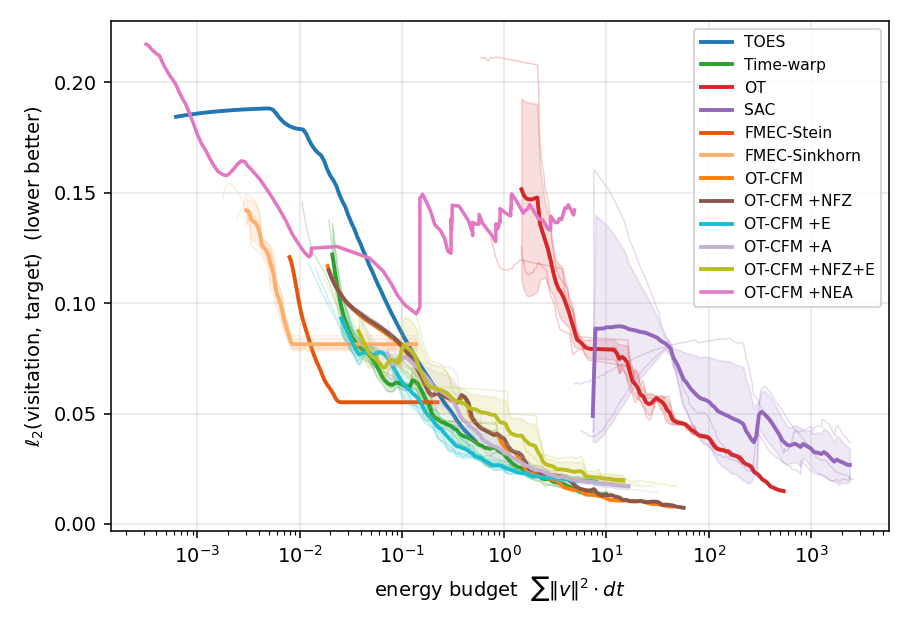}\hfill
\includegraphics[width=0.32\textwidth]{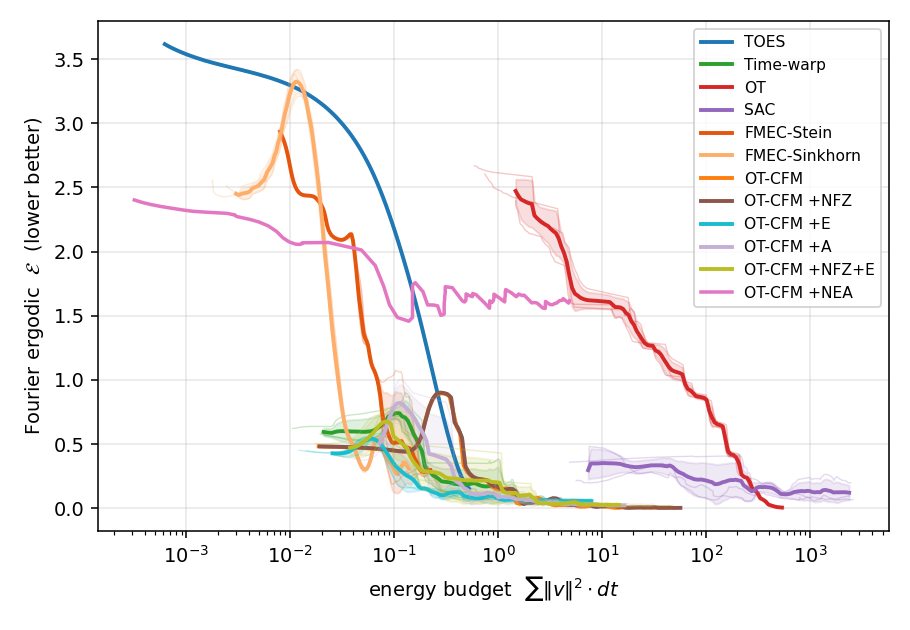}\\[0.3em]
\includegraphics[width=0.32\textwidth]{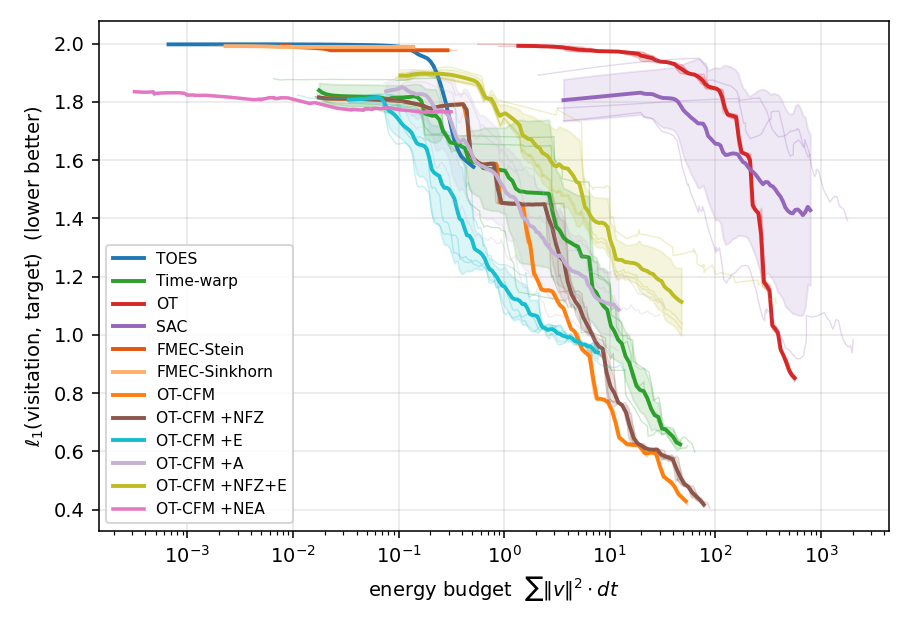}\hfill
\includegraphics[width=0.32\textwidth]{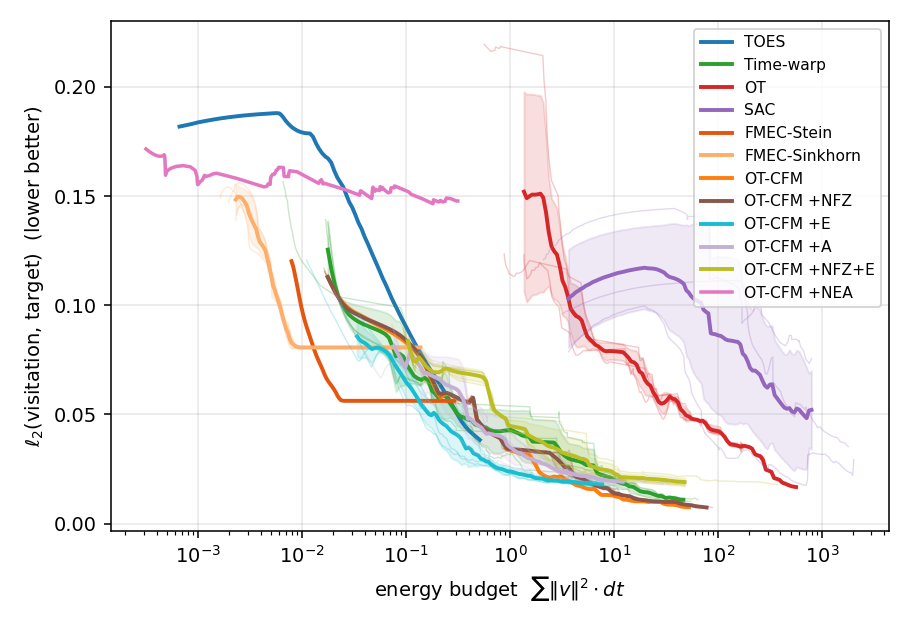}\hfill
\includegraphics[width=0.32\textwidth]{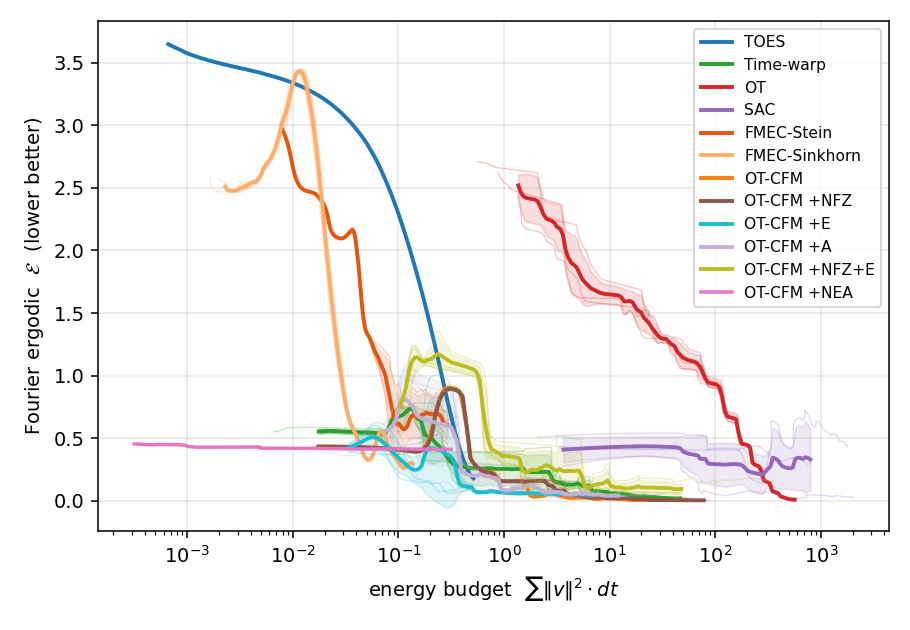}
\caption{Coverage--energy trade-off across additional coverage axes. Columns
(left to right): $\ell_1$, $\ell_2$, Fourier ergodic $\mathcal{E}$. Rows:
\textbf{top} single-disc NFZ, \textbf{bottom} multi-disc NFZ. Lower is better
on all axes. Shaded bands: $\pm 1$ std over 3 seeds.}
\label{fig:pareto_extra}
\end{figure}
 
\subsection{Velocity Realizability}
\label{app:velocity}
 
The energy and mission-time numbers assume each trajectory is flown at
physically realizable speeds, capped at $V_{\max}=30$\,m/s.
Several baselines produce raw trajectories exceeding this cap; for those
methods, the reported $\bar{P}$ corresponds to the trajectory re-timed to
satisfy the cap, tracked via a time-extension factor
$T_{\mathrm{realizable}}/T_{\mathrm{reported}}$.
OT and SAC have time-extension factors of ${\sim}2$ and ${\sim}5$
respectively, meaning their reported spatial coverage is achievable only if
the actual flight is multiple times longer.
OT-CFM~\textbf{+E} has time-extension factor $1.00$ in both scenarios with
${<}2\%$ of samples exceeding the cap (Table~\ref{tab:milano_velocity}).
 
\begin{table}[H]
\centering
\caption{Velocity statistics and time-extension factors on Milano
(single-disc; multi-disc values in parentheses where they differ
substantially). \texttt{frac over cap} reports the fraction of trajectory
samples exceeding the $30$\,m/s realistic cap.}
\label{tab:milano_velocity}
\small
\setlength{\tabcolsep}{4pt}
\begin{tabular}{lcccc}
\toprule
Method & median $V$ (m/s) & p99 $V$ (m/s) & frac over cap (\%) & time-ext.\ factor \\
\midrule
TOES          & 0.57        & 0.62        & 0.0         & 1.00        \\
Time-warp     & 8.34 (7.30) & 34.8 (41.0) & 1.6 (1.7)   & 1.01 (1.05) \\
OT            & 43.4        & 156         & 68.8        & 1.82 (1.87) \\
SAC           & 163 (81.6)  & 346         & 73.9 (58.7) & 5.41 (3.92) \\
FMEC-Stein    & 1.29        & 3.42        & 0.0         & 1.00        \\
FMEC-Sinkhorn & 1.14        & 1.79        & 0.0         & 1.00        \\
\midrule
OT-CFM        & 19.8 (18.9) & 72.3 (98.4) & 18.4 (17.9) & 1.08 (1.12) \\
OT-CFM +NFZ   & 19.6 (18.2) & 90.4 (95.6) & 18.1 (16.7) & 1.10 (1.13) \\
OT-CFM +E     & 11.1 (11.0) & 29.7 (33.5) & \textbf{0.9 (1.3)} & \textbf{1.00 (1.00)} \\
OT-CFM +A     & 13.7 (12.7) & 43.6 (54.7) & 4.5 (5.8)   & 1.02 (1.03) \\
OT-CFM +NFZ+E & 13.2 (14.1) & 37.9 (91.2) & 2.1 (9.6)   & 1.02 (1.14) \\
\bottomrule
\end{tabular}
\end{table}
 
\subsection{Multi-Agent Reuse and Amortization}
\label{app:milano_amortization}
 
A single trained $G_\theta$ supports independent multi-agent deployment
without re-training, with predicted collective convergence rate
$O(1/\sqrt{NK})$ (\S\ref{sec:multi}, Theorem~\ref{thm:convergence_rate}).
We instantiate $N\in\{1,2,5,10,20\}$ agents from independent latent
trajectories sharing the OT-CFM map.
Figure~\ref{fig:milano_multi_agent} shows aggregate Fourier ergodic metric
$\mathcal{E}_N$ vs.\ fleet size $N$; the metric drops ${\sim}2.5\times$ from
$N{=}1$ to $N{=}20$ in both scenarios, broadly tracking the $1/\sqrt{N}$
reference; per-agent power is essentially flat.
 
\begin{figure}[H]
\centering
\includegraphics[width=0.49\textwidth]{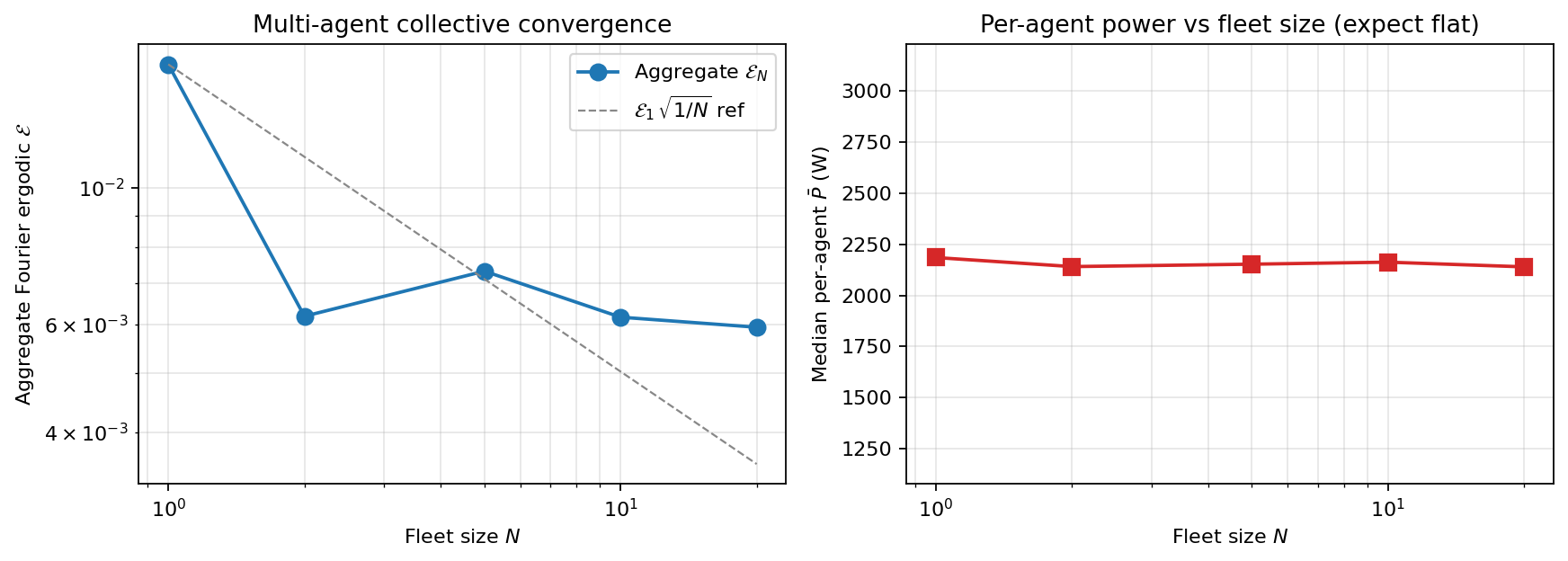}\hfill
\includegraphics[width=0.49\textwidth]{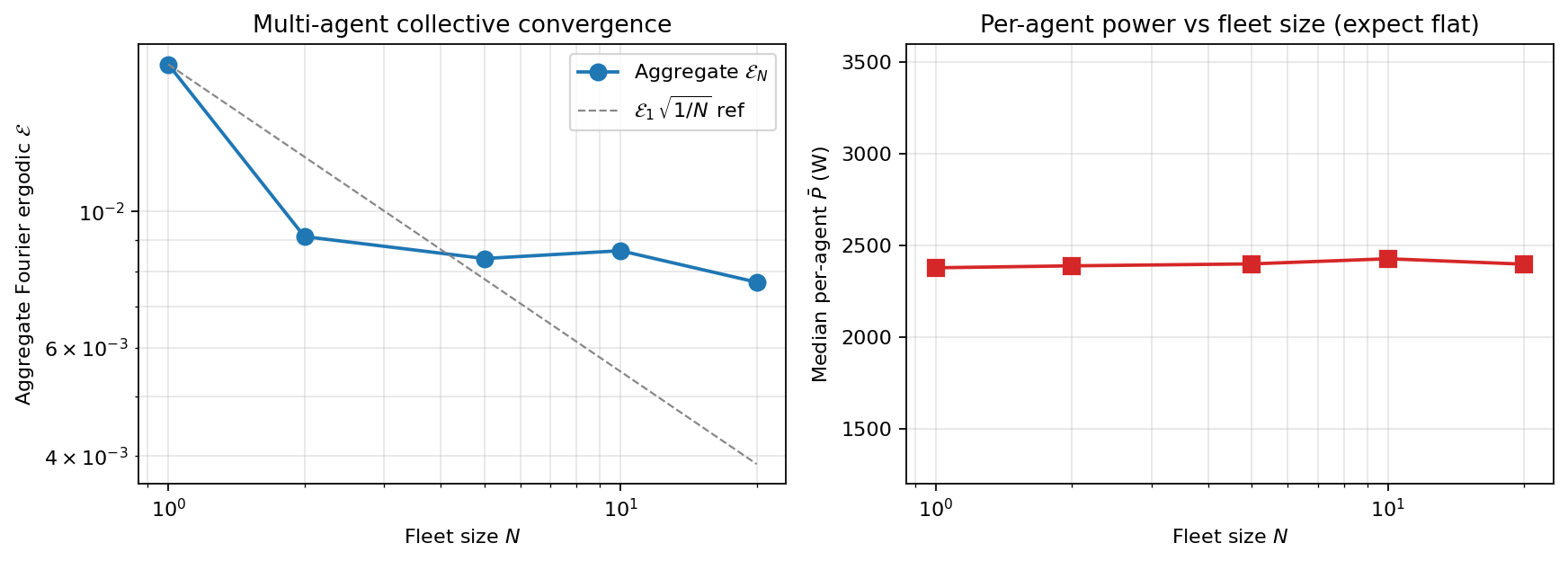}
\caption{Multi-agent reuse on Milano. \textbf{Left:} single-disc.
\textbf{Right:} multi-disc. Aggregate $\mathcal{E}_N$ vs.\ fleet size $N$,
with the predicted $\mathcal{E}_1\sqrt{1/N}$ reference (dashed). Per-agent
power remains essentially flat across $N$.}
\label{fig:milano_multi_agent}
\end{figure}
 
\paragraph{Amortization cost structure.}
Table~\ref{tab:milano_amortization_table} compares per-trajectory online
cost across methods. For all baselines this is the time to produce a single
trajectory (optimization, re-derivation, or RL rollout). For OT-CFM,
the equivalent online cost is a single inference pass through $v_\theta$:
$779\pm16$\,ms (1D) and $780\pm6$\,ms (MD). The offline training cost of
${\approx}1.7$\,h is paid once and amortised across all subsequent
trajectories and agents. For a fleet of $N$ agents, since each agent's
latent trajectory is independent and all inference passes share the same
$G_\theta$, they can be parallelised, keeping the wall-clock fleet cost
equal to the single-agent inference time regardless of $N$.
 
\begin{table}[H]
\centering
\caption{Per-trajectory online cost and marginal cost of fleet expansion.
Baselines report time to produce one trajectory; OT-CFM reports inference
time (offline training ${\approx}1.7$\,h, paid once). For OT-CFM the
$N{=}10$ fleet cost equals the single-trajectory cost because all agents'
inference passes are independent and run in parallel.}
\label{tab:milano_amortization_table}
\small
\setlength{\tabcolsep}{4pt}
\begin{tabular}{l cc c}
\toprule
& \multicolumn{2}{c}{Online cost (per trajectory)} & Multi-agent ($N{=}10$) \\
\cmidrule(lr){2-3}
Method & 1D & MD & (any scenario) \\
\midrule
TOES          & 2.79\,h$\pm$5.5\,min  & 40.8\,s$\pm$0.6\,s    & $10\times$ SQP solve   \\
Time-warp     & 111\,ms$\pm$22\,ms    & 78\,ms$\pm$11\,ms     & $10\times$ re-derive   \\
OT            & 309\,ms$\pm$3\,ms     & 225\,ms$\pm$1\,ms     & $10\times$ OT plan     \\
SAC           & 1.32\,h$\pm$9\,s      & 1.61\,h$\pm$15\,min   & $10\times$ retrain     \\
FMEC-Stein    & 4.7\,min$\pm$34\,s    & 8.8\,min$\pm$2.9\,min & $10\times$ re-optimise \\
FMEC-Sinkhorn & 3.3\,min$\pm$34\,s    & 4.5\,min$\pm$1.2\,min & $10\times$ re-optimise \\
\midrule
\textbf{OT-CFM (ours)} & $779\pm16$\,ms & $780\pm6$\,ms & $779\pm16$\,ms / $780\pm6$\,ms (parallel) \\
\bottomrule
\end{tabular}
\end{table}

 
\section{Qualitative Trajectory Visualizations}
\label{app:trajectories}
 
Figures~\ref{fig:trajectories_single}--\ref{fig:trajectories_multi} visualize
representative trajectories for all methods on the Milano target density
(purple heatmap) with the NFZ overlay (red dashed circles), under both NFZ
configurations. These plots support the trivial-compliance claim made in
\S\ref{sec:exp_real}: TOES produces a single straight diagonal line that does
not redistribute mass across the domain, and the FMEC variants produce short,
near-stationary curls confined to one corner. Their zero NFZ violation and low
Zeng power are accidental byproducts of failing to cover the domain rather
than evidence of a constraint-aware or energy-efficient design. By contrast,
OT-CFM trajectories traverse the high-density regions while respecting the
NFZ overlay, with per-variant differences (e.g., \textbf{+E} producing
visibly straighter and shorter paths) that match the quantitative differences
in Tables~\ref{tab:milano_headline}--\ref{tab:milano_nfz_full}.
 
\begin{figure}[H]
\centering
\includegraphics[width=0.95\textwidth]{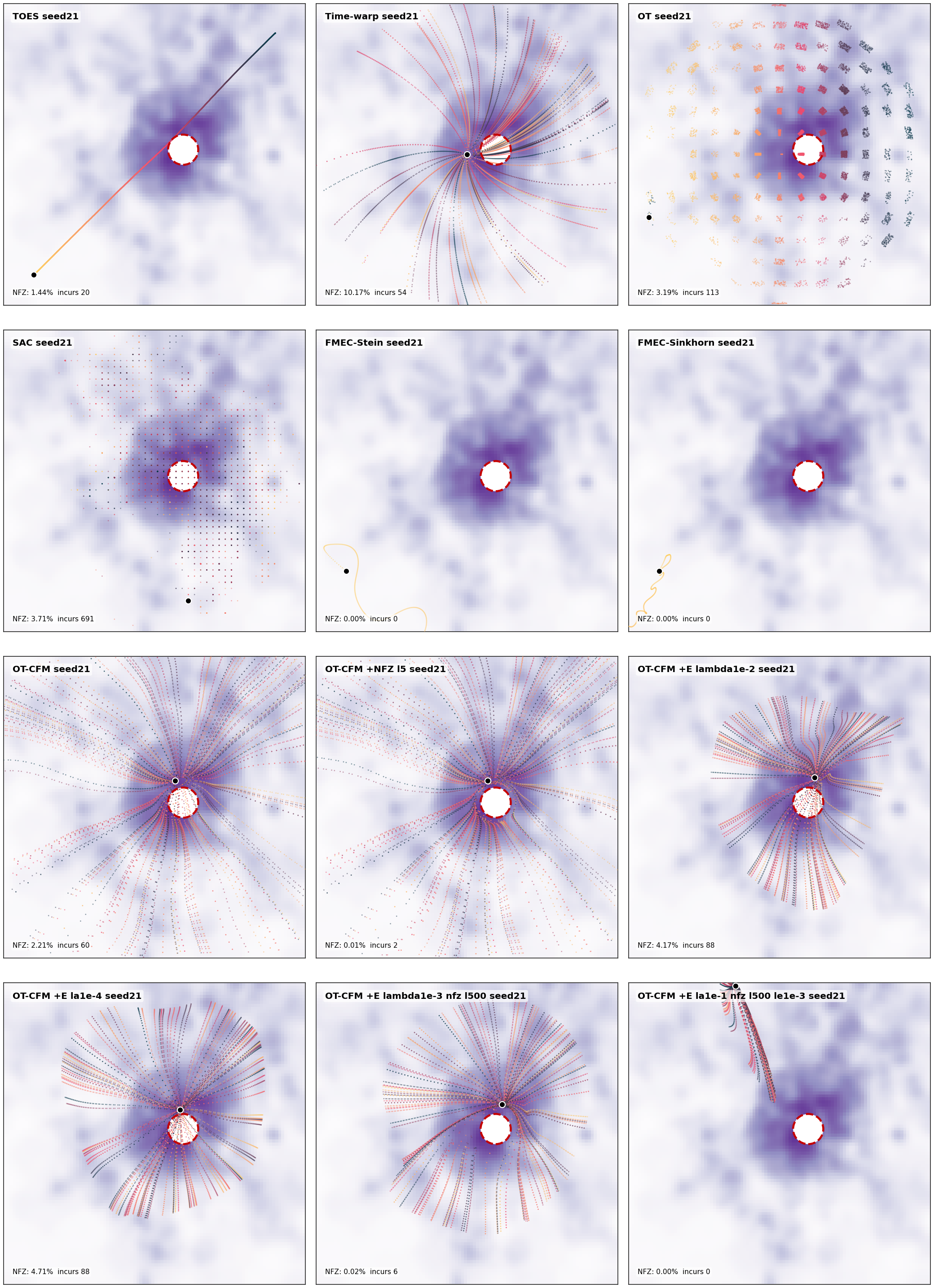}
\caption{Representative trajectories on the Milano target (purple heatmap)
under the single-disc NFZ configuration (red dashed circle), one panel per
method. TOES and the FMEC variants fail to redistribute mass across the
domain — their zero NFZ violation in Table~\ref{tab:milano_nfz_full} is
accidental. OT-CFM variants traverse the high-density regions while
respecting the NFZ; \textbf{+NFZ} visibly bends paths around the exclusion
disc, and \textbf{+E} produces visibly straighter and shorter paths consistent
with its low-energy-per-meter operating point.}
\label{fig:trajectories_single}
\end{figure}
 
\begin{figure}[H]
\centering
\includegraphics[width=0.95\textwidth]{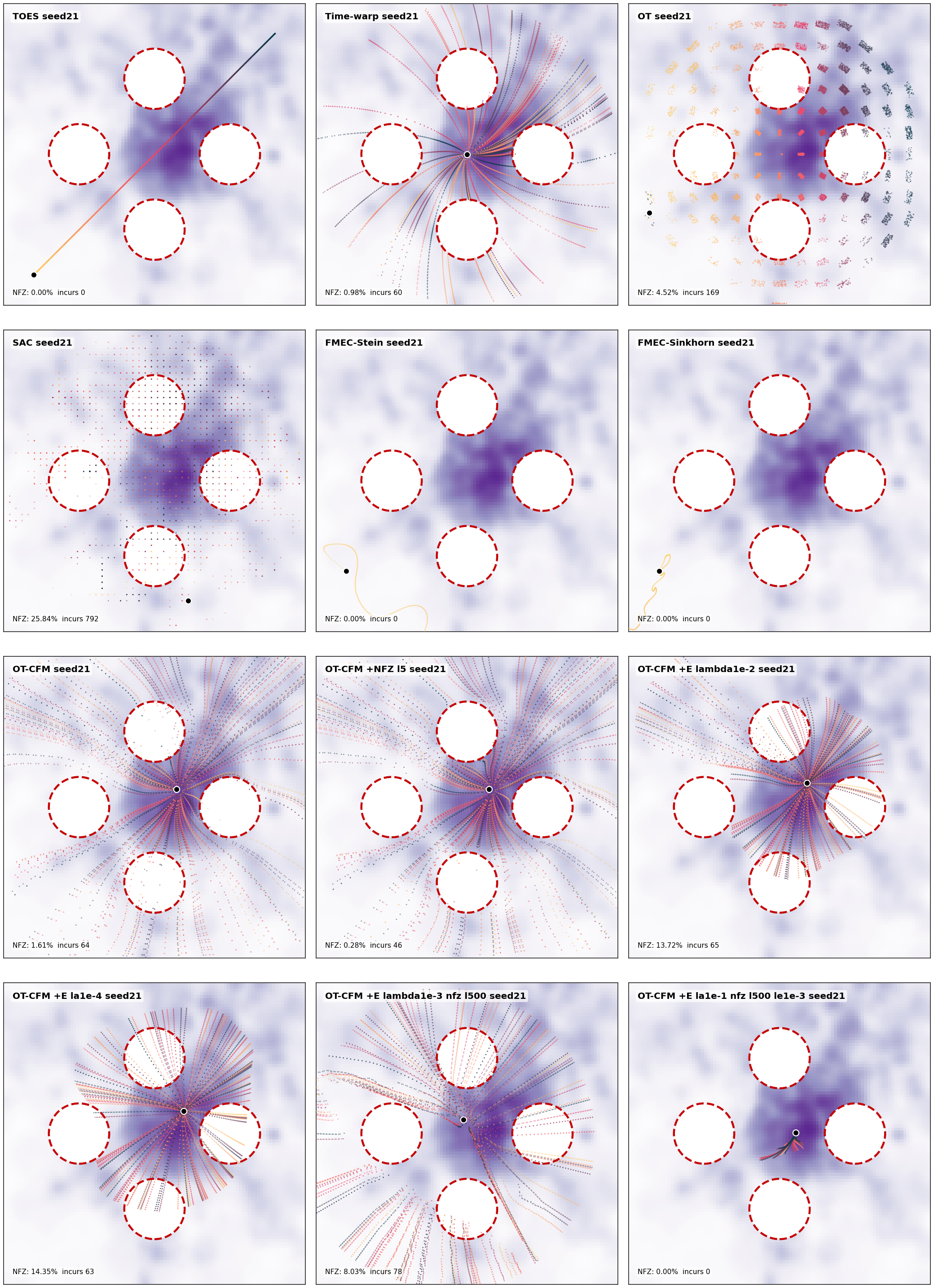}
\caption{Representative trajectories on the Milano target (purple heatmap)
under the multi-disc NFZ configuration (four red dashed circles), one panel
per method. The trivial-compliance pattern is even starker here: TOES's
diagonal line happens not to intersect any of the four discs, and the FMEC
variants' corner curls remain far from every NFZ. OT-CFM~\textbf{+NFZ} routes
around all four exclusion zones while preserving coverage of the surrounding
high-density region.}
\label{fig:trajectories_multi}
\end{figure}

\section{Limitations and Future Work}
\label{app:limitations_future_work}

This appendix expands the scope discussion in Section~\ref{sec:conclusion}. The framework establishes a reusable pushforward map with asymptotic ergodicity (Proposition~\ref{prop:ergodic}), finite-cycle concentration at rate $O(1/\sqrt{K})$ (Theorem~\ref{thm:convergence_rate}), and a trainable approximation envelope (Theorem~\ref{thm:approx_bound}, Corollary~\ref{cor:end_to_end}). The four scope boundaries below identify where the present guarantees are empirical, soft-constrained, or experimentally bounded; each suggests a direct extension of the framework rather than a separate research program.

\subsection{Certified approximation bounds}\label{app:limitations_certification}

The practical Theorem~\ref{thm:approx_bound} envelope $e^{L_v}\cdot\varepsilon_v$ used throughout our calibration table (Table~\ref{tab:Lv_calibration}) plugs in the empirical power-iteration estimate $\hat L_v$ from the trained $v_\theta$. This estimator is empirically tight and admits non-vacuous numerical envelopes, but it is a lower-bound proxy on the true velocity-field Lipschitz constant rather than a certified upper bound; the architectural product-norm bound $L_v^{\mathrm{net}}$ from Appendix~\ref{app:invertibility} is rigorous but, in our experiments, $\sim$30$\times$ looser than the empirical estimate. Closing this gap is mechanically straightforward and does not require modifying the pushforward construction itself: spectral normalization \citep{miyato2018spectral}, Parseval-style orthogonal constraints, interval-bound propagation, or certified Jacobian regularization can be combined with the existing CFM training objective to produce a $v_\theta$ whose architectural Lipschitz constant matches the empirical operator norm. The same Theorem~\ref{thm:approx_bound} statement then yields a certified upper envelope on $W_2$ and Corollary~\ref{cor:end_to_end} a certified end-to-end coverage envelope, with no change to the latent-trajectory or pushforward design. Finite-sample generalization, the stochastic-optimization gap on the CFM loss, and the Sinkhorn-coupling residual that currently appear in the envelope as empirical diagnostics admit an analogous tightening through standard statistical-learning machinery; their values in the present experiments serve as benchmarks against which a certified version can be compared.

\subsection{Latent geometry and topology}\label{app:limitations_geometry}

The annular latent domain $D_\delta$ makes the radial construction of Section~\ref{sec:latent} smooth away from the origin and supports exact uniform ergodicity on $D_\delta$, which is what allows Proposition~\ref{prop:ergodic} to be stated as an equality rather than an approximation. When the target $f_{\mathrm{target}}$ has simply connected support, no homeomorphism between $D_\delta$ and $\mathrm{supp}(f_{\mathrm{target}})$ exists, and the pushforward incurs an $O(\delta)$ topology residual. Theorem~\ref{thm:acc_bound} shows the latent acceleration energy scales as $\Theta(\delta^{-4})$, so $\delta$ acts as a tunable geometry-energy parameter rather than a free hyperparameter. Several extensions reduce this residual without abandoning the analytic-latent / learned-pushforward decomposition. A learned latent domain trained jointly with $G_\theta$ can adapt $D_\delta$ to the target support; mixtures of annular cycles indexed by an angular phase produce richer latent measures matched to bimodal targets; non-radial latent processes (e.g., Lissajous curves with engineered period structure or stochastic differentials with prescribed invariant density) drop the radial-sweep assumption while preserving the i.i.d.\ cycle structure that drives the $O(1/\sqrt{K})$ rate. A complementary extension matches latent topology to target topology by template selection: an annular latent for annular targets, a disc-with-handle for targets with one excluded region, and so on. Each variant retains the central decomposition (analytically understood latent followed by learned pushforward) and inherits the corresponding theorem statements with structurally identical proofs. The boundary smoothing of Appendix~\ref{app:boundary_smooth} can be combined with any of the above to recover a globally $C^1$ velocity at the cost of a controlled additional residual.

\subsection{From soft penalties to certified safety}\label{app:limitations_safety}

The penalty mechanism in~\eqref{eq:total_loss} provides a differentiable gradient signal at constraint boundaries (no-fly zones, acceleration ceilings, fairness terms), and Experiment~3 shows that this signal is sufficient to drive constraint violations from $\sim$$1\%$ to $\sim$$0.1\%$ at moderate density-fidelity cost. This is the appropriate mechanism for studying density-constraint-energy trade-offs and for non-safety-critical deployments. Safety-critical missions, where any trajectory entry into a no-fly zone is unacceptable, should layer a hard-feasibility component on top of the trained $G_\theta$ rather than rely on the soft penalty alone. Compatible options include: projection of $G_\theta(\mathbf{z}(t))$ onto the feasible set (closed-form for half-space and disc constraints, semidefinite for general convex sets); control barrier functions that admit the trained pushforward as a nominal controller and modify it minimally to enforce forward-invariance of the safe set; reachability filters that solve a Hamilton$-$Jacobi$-$Isaacs PDE offline and shield the deployed trajectory at run time; and constrained ODE integration methods that project velocity-field evaluations onto the tangent cone of the feasible set during RK4 rollout. In each case the learned map serves as the nominal coverage generator and the safety layer is an outer wrapper, preserving the amortized-design property that motivates the framework. Quantifying the coverage cost of each safety layer (the additional Wasserstein gap incurred by hard projection vs.\ soft penalty at matched zero-violation level) is a natural follow-up.

\subsection{Empirical scope and broader extensions}\label{app:limitations_extensions}

The body experiments are designed to isolate density matching (Experiment~1), demand-proportional allocation (Experiment~2), and differentiable constraint trade-offs (Experiment~3) in 2D open-loop settings, with each experiment chosen to expose a single theoretical mechanism. This scoping makes the theoretical predictions visible in the empirical numbers (the $O(1/\sqrt{K})$ rate, the acceleration-ratio dependence on $L$, the $e^{L_v}\cdot\varepsilon_v$ envelope), but broader validation along several axes will sharpen the deployment story.

\paragraph{Higher-dimensional dynamics.} The framework extends to $\mathbb{R}^d$ with the same construction (Appendix~\ref{app:rd_extension}); 3D experiments require only re-tuning of $\delta$ and the latent trajectory family, with the theorem statements carrying through with dimension-dependent constants (Hessian tensor norm, Poincar\'{e} eigenvalue, latent-energy prefactor; see Appendix~\ref{app:rd_extension}). The empirical question is whether the $\Theta(\delta^{-(3d-2)})$ energy scaling (specializing to $\Theta(\delta^{-4})$ in 2D and $\Theta(\delta^{-7})$ in 3D) remains the binding constraint, or whether higher-order curvature effects dominate at the higher dimension.

\paragraph{Hardware-in-the-loop and closed-loop deployment.} The open-loop nature of the trajectories makes them directly compatible with standard quadrotor controllers; the empirical question is the coverage degradation introduced by tracking error, which can be folded into the approximation bound via a Lipschitz disturbance term added to $\varepsilon_v$. Closed-loop variants combining the nominal pushforward trajectory with feedback corrections for disturbances, localization error, or online measurements introduce a new theoretical question (how feedback perturbations affect the finite-cycle concentration and the pushforward occupancy), addressable via existing concentration-inequality and perturbation-bound machinery.

\paragraph{Coordinated multi-agent planning.} The current multi-agent claim is independent reuse of a single trained map (Appendix~\ref{app:coordinated}); fully coordinated planning with collision avoidance, formation constraints, or fairness coupling between agents extends the framework via joint pushforward maps on product spaces, mean-field penalties added to $\mathcal{L}_{\mathrm{total}}$, repulsive interaction terms, or fairness-aware coupling losses. All preserve the amortized-design principle while moving from independent reuse to coordinated fleet behavior.

\paragraph{Time-varying targets.} The present paper focuses on fixed $f_{\mathrm{target}}$ and open-loop execution. Time-varying densities can be addressed by conditioning $v_\theta$ on a target descriptor, by interpolating between learned maps, or by retraining on a slow planning timescale while retaining open-loop execution between updates. The theoretical extension folds the target-evolution rate into the approximation bound as an additional residual term.

\newpage
\input{checklist.tex}
\end{document}

%% file: checklist.tex
\section*{NeurIPS Paper Checklist}

\begin{enumerate}

\item {\bf Claims}
    \item[] Question: Do the main claims made in the abstract and introduction accurately reflect the paper's contributions and scope?
    \item[] Answer: \answerYes{}
    \item[] Justification: The abstract and introduction (Section~1) clearly state six contributions (ergodicity transfer, acceleration bound, convergence rate, approximation bound, three operational properties of the decomposition, and experimental validation) that are formally proved in Section~3 and validated in Section~4. Limitations are discussed in Section~\ref{sec:conclusion}.

\item {\bf Limitations}
    \item[] Question: Does the paper discuss the limitations of the work performed by the authors?
    \item[] Answer: \answerYes{}
    \item[] Justification: We have clearly stated the assumptions required for our theoretical results and the setups used in our numerical experiments. Additionally, we have discussed the limitations of this paper in the Conclusion section as well as in Appendix \ref{app:limitations_future_work}.


\item {\bf Theory assumptions and proofs}
    \item[] Question: For each theoretical result, does the paper provide the full set of assumptions and a complete (and correct) proof?
    \item[] Answer: \answerYes{}
    \item[] Justification: Theorems~1--3 state all assumptions explicitly (Lipschitz constants, boundedness, $C^2$ regularity). Complete proofs are provided in Appendices~B--E. Proof sketches and interpretations appear in the main text (Section~3).

    \item {\bf Experimental result reproducibility}
    \item[] Question: Does the paper fully disclose all the information needed to reproduce the main experimental results of the paper to the extent that it affects the main claims and/or conclusions of the paper (regardless of whether the code and data are provided or not)?
    \item[] Answer: \answerYes{}
    \item[] Justification: Full experimental details (network architecture, optimizer, hyperparameters, evaluation metrics) are provided in Appendix~\ref{app:experiments_full}. Target densities are defined analytically. Code will be released upon acceptance.

\item {\bf Open access to data and code}
    \item[] Question: Does the paper provide open access to the data and code, with sufficient instructions to faithfully reproduce the main experimental results, as described in supplemental material?
    \item[] Answer: \answerNo{}
    \item[] Justification: The code will be released upon acceptance. All experimental results for the proposed method and baselines are reproducible. Full implementation details are provided in Appendix~\ref{app:experiments_full}.
    

\item {\bf Experimental setting/details}
    \item[] Question: Does the paper specify all the training and test details (e.g., data splits, hyperparameters, how they were chosen, type of optimizer) necessary to understand the results?
    \item[] Answer: \answerYes{}
    \item[] Justification: Section~\ref{sec:experiments} summarizes key settings; Appendix~\ref{app:experiments_full} provides complete details including network architecture, optimizer, learning rate schedule, OT coupling parameters, and evaluation metrics.

\item {\bf Experiment statistical significance}
    \item[] Question: Does the paper report error bars suitably and correctly defined or other appropriate information about the statistical significance of the experiments?
    \item[] Answer: \answerYes{}
    \item[] Justification: Error bars accompany every experiment that directly validates our theoretical claims. (i) \emph{Theorem~\ref{thm:convergence_rate} (convergence rate).} Figure~\ref{fig:convergence} reports $\|Z_K^{\mathrm{traj}} - Z_{\mathrm{iid}}\|_{\mathrm{RMSE}}$ as mean $\pm$ std (shaded region) over five independent cycle-sampling seeds at each $K\in\{5,10,20,50,100\}$; the empirical slope $-0.48$ approaches the theoretical $-0.50$ as $K$ grows, with pairwise local slopes $\{-0.43,-0.52,-0.53,-0.41\}$. (ii) \emph{Theorem~\ref{thm:approx_bound} (approximation bound).} Table~\ref{tab:Lv_calibration} reports the empirical 2-Wasserstein distance $\hat W_2$ and the map-level Lipschitz estimator $\hat L$ for the five evaluated configurations / rows of Table~\ref{tab:Lv_calibration}, each as mean $\pm$ std over three evaluation seeds for a fixed trained map; cross-seed stability is within $2\%$ on $\hat L$ and within $\sim\!4\text{--}24\%$ on $\hat W_2$ (the upper end is Exp.~2, whose binary discontinuity inflates seed variance relative to the smoother targets). In both cases the error bars capture trajectory-side randomness---the source Theorem~\ref{thm:convergence_rate} is designed to bound via closed-form Hoeffding concentration (Corollary~\ref{cor:end_to_end})---and the seeds vary cycle-sampling realizations of a single trained map, following the NeurIPS convention of treating simulation-seed variability of a deployed algorithm as the appropriate randomness source when target distributions are analytically specified. Density correlations and acceleration ratios reported in the experimental setup are single-training-seed point estimates on analytically defined target densities (no train/test split); these are setup-descriptive numbers, not direct tests of the theorems. A multi-seed training reproducibility kit (per-seed JSON-logging driver, aggregator, and drop-in table-regeneration script) will be released with the code at the camera-ready stage for applied follow-ups that wish to characterize training-seed variability.

\item {\bf Experiments compute resources}
    \item[] Question: For each experiment, does the paper provide sufficient information on the computer resources (type of compute workers, memory, time of execution) needed to reproduce the experiments?
    \item[] Answer: \answerYes{}
    \item[] Justification: Appendix~\ref{app:experiments_full} specifies the GPU (NVIDIA 2080Ti) and CPU types  (Intel Xeon E5-2680 v4), memory (8\,GB), and per-experiment training times (12--25 minutes each). Total compute: $<$2 CPU-hours.

\item {\bf Code of ethics}
    \item[] Question: Does the research conducted in the paper conform, in every respect, with the NeurIPS Code of Ethics \url{https://neurips.cc/public/EthicsGuidelines}?
    \item[] Answer: \answerYes{}
    \item[] Justification: This is foundational research on trajectory design with no human subjects, no scraped data, and no foreseeable negative societal impacts beyond standard dual-use considerations for autonomous systems.

\item {\bf Broader impacts}
    \item[] Question: Does the paper discuss both potential positive societal impacts and negative societal impacts of the work performed?
    \item[] Answer: \answerYes{}
    \item[] Justification: Appendix~\ref{app:applications} discusses positive applications (search-and-rescue, environmental and agricultural monitoring, wildfire mapping, wireless-infrastructure coverage equity). We acknowledge dual-use considerations: the same trajectory-design capabilities could in principle be applied to surveillance or persistent monitoring without consent. The framework operates from aggregate density specifications (not individual tracking), which inherently limits its surveillance utility; the privacy-preserving property (no real-time location feedback) is a mitigating design feature; and we do not release pretrained models targeting specific real-world geographies. Standard deployment safeguards (geofencing, operator authorization, airspace-authority compliance) are orthogonal to our contribution and unaffected by it.

\item {\bf Safeguards}
    \item[] Question: Does the paper describe safeguards that have been put in place for responsible release of data or models that have a high risk for misuse (e.g., pre-trained language models, image generators, or scraped datasets)?
    \item[] Answer: \answerNA{}
    \item[] Justification: The paper does not release pre-trained models or datasets with misuse risk. The models are small ($\sim$199K parameters) 2D spatial maps for trajectory design.

\item {\bf Licenses for existing assets}
    \item[] Question: Are the creators or original owners of assets (e.g., code, data, models), used in the paper, properly credited and are the license and terms of use explicitly mentioned and properly respected?
    \item[] Answer: \answerYes{}
    \item[] Justification: All used external datasets or pre-existing code assets are cited carefully. 

\item {\bf New assets}
    \item[] Question: Are new assets introduced in the paper well documented and is the documentation provided alongside the assets?
    \item[] Answer: \answerNA{}
    \item[] Justification: No new assets (datasets, models) are released with this submission. Code will be released upon acceptance.

\item {\bf Crowdsourcing and research with human subjects}
    \item[] Question: For crowdsourcing experiments and research with human subjects, does the paper include the full text of instructions given to participants and screenshots, if applicable, as well as details about compensation (if any)?
    \item[] Answer: \answerNA{}
    \item[] Justification: This paper does not involve crowdsourcing or research with human subjects.

\item {\bf Institutional review board (IRB) approvals or equivalent for research with human subjects}
    \item[] Question: Does the paper describe potential risks incurred by study participants, whether such risks were disclosed to the subjects, and whether Institutional Review Board (IRB) approvals (or an equivalent approval/review based on the requirements of your country or institution) were obtained?
    \item[] Answer: \answerNA{}
    \item[] Justification: This paper does not involve research with human subjects.

\item {\bf Declaration of LLM usage}
    \item[] Question: Does the paper describe the usage of LLMs if it is an important, original, or non-standard component of the core methods in this research? Note that if the LLM is used only for writing, editing, or formatting purposes and does \emph{not} impact the core methodology, scientific rigor, or originality of the research, declaration is not required.
    \item[] Answer: \answerNA{}
    \item[] Justification: The core method development in this research does not involve LLMs as any important, original, or non-standard components. 

\end{enumerate}

%% file: main.bbl
\begin{thebibliography}{20}

\bibitem[Abraham and Murphey(2018)]{abraham2018decentralized}
Abraham, I. and Murphey, T.~D.
\newblock Decentralized ergodic control: distribution-driven sensing and exploration for multiagent systems.
\newblock \emph{IEEE Robotics and Automation Letters}, 3(4):2987--2994, 2018.

\bibitem[Caffarelli(1992)]{caffarelli1992regularity}
Caffarelli, L.~A.
\newblock The regularity of mappings with a convex potential.
\newblock \emph{Journal of the American Mathematical Society}, 5(1):99--104, 1992.

\bibitem[Gangbo and \'{S}wi\c{e}ch(1998)]{gangbo1998optimal}
Gangbo, W. and \'{S}wi\c{e}ch, A.
\newblock Optimal maps for the multidimensional {M}onge--{K}antorovich problem.
\newblock \emph{Communications on Pure and Applied Mathematics}, 51(1):23--45, 1998.

\bibitem[Enayati et al.(2019)]{enayati2019moving}
Enayati, S., Saeedi, H., Pishro-Nik, H., and Yanikomeroglu, H.
\newblock Moving aerial base station networks: A stochastic geometry analysis and design perspective.
\newblock \emph{IEEE Trans.\ Wireless Communications}, 18(6):2977--2988, 2019.

\bibitem[Galceran and Carreras(2013)]{galceran2013survey}
Galceran, E. and Carreras, M.
\newblock A survey on coverage path planning for robotics.
\newblock \emph{Robotics and Autonomous Systems}, 61(12):1258--1276, 2013.

\bibitem[Lerch et al.(2023)]{lerch2023safety}
Lerch, C., Dong, D., and Abraham, I.
\newblock Safety-critical ergodic exploration in cluttered environments via control barrier functions.
\newblock In \emph{IEEE International Conference on Robotics and Automation (ICRA)}, pp. 10205-10211, 2023.

\bibitem[Dong et al.(2023)]{dong2023timeoptimal}
Dong, D., Berger, H., and Abraham, I.
\newblock Time optimal ergodic search.
\newblock In \emph{Robotics: Science and Systems (RSS)}, 2023.

\bibitem[Lipman et al.(2023)]{lipman2023flow}
Lipman, Y., Chen, R.~T.~Q., Ben-Hamu, H., Nickel, M., and Le, M.
\newblock Flow matching for generative modeling.
\newblock \emph{International Conference on Learning Representations}, 2023.

\bibitem[Malekzadeh et al.(2026)]{malekzadeh2025nonuniform}
Malekzadeh, M., Ghasemi, A., and Pishro-Nik, H.
\newblock Robust {UAV} trajectory design for non-uniform coverage.
\newblock \emph{IEEE Communications Letters}, 30:188--192, 2026. DOI: \texttt{10.1109/LCOMM.2025.3629065}.

\bibitem[Mathew and Mezi\'{c}(2011)]{mathew2011metrics}
Mathew, G. and Mezi\'{c}, I.
\newblock Metrics for ergodicity and design of ergodic dynamics for multi-agent systems.
\newblock \emph{Physica D}, 240(4--5):432--442, 2011.

\bibitem[Sun et al.(2025)]{sun2025flow}
Sun, M.~M., Pinosky, A., and Murphey, T.
\newblock Flow matching ergodic coverage.
\newblock In \emph{Robotics: Science and Systems (RSS)}, 2025. arXiv:2504.17872.

\bibitem[Sun et al.(2025b)]{sun2025flowcode}
Sun, M.~M., Pinosky, A., and Murphey, T.
\newblock Flow matching ergodic coverage --- official tutorials.
\newblock GitHub repository, \url{https://github.com/MurpheyLab/lqr-flow-matching/tree/main/tutorials}, 2025.

\bibitem[Theile et al.(2020)]{theile2020uav}
Theile, M., Bayerlein, H., Nai, R., Gesbert, D., and Caccamo, M.
\newblock UAV coverage path planning under varying power constraints using deep reinforcement learning.
\newblock In \emph{2020 IEEE/RSJ International Conference on Intelligent Robots and Systems (IROS)}, pp. 1444--1449, 2020. doi:10.1109/IROS45743.2020.9340934.

\bibitem[Tong et al.(2024)]{tong2024improving}
Tong, A., Malkin, N., Huguet, G., Zhang, Y., Rector-Brooks, J., Fatras, K., Wolf, G., and Bengio, Y.
\newblock Improving and generalizing flow-based generative models with minibatch optimal transport.
\newblock \emph{Transactions on Machine Learning Research}, 2024.

\bibitem[Zeng et al.(2019)]{zeng2019accessing}
Zeng, Y., Wu, Q., and Zhang, R.
\newblock Accessing from the sky: A tutorial on {UAV} communications for 5{G} and beyond.
\newblock \emph{Proceedings of the IEEE}, 107(12):2327--2375, 2019.

\bibitem[Mozaffari et al.(2019)]{mozaffari2019tutorial}
Mozaffari, M., Saad, W., Bennis, M., Nam, Y.-H., and Debbah, M.
\newblock A tutorial on {UAV}s for wireless networks: Applications, challenges, and open problems.
\newblock \emph{IEEE Communications Surveys \& Tutorials}, 21(3):2334--2360, 2019.

\bibitem[Ayvali et al.(2017)]{ayvali2017ergodic}
Ayvali, E., Salman, H., and Choset, H.
\newblock Ergodic coverage in constrained environments using stochastic trajectory optimization.
\newblock In \emph{2017 IEEE/RSJ International Conference on Intelligent Robots and Systems (IROS)}, pp. 5204--5210, 2017.
\newblock doi: 10.1109/IROS.2017.8206410.

\bibitem[Teschl(2012)]{teschl2012ordinary}
Teschl, G.
\newblock \emph{Ordinary Differential Equations and Dynamical Systems}.
\newblock Vol.~140. American Mathematical Society, 2012.

\bibitem[Roberts and Tweedie(1996)]{roberts1996exponential}
Roberts, G.~O. and Tweedie, R.~L.
\newblock Exponential convergence of {L}angevin distributions and their discrete approximations.
\newblock \emph{Bernoulli}, 2(4):341--363, 1996.

\bibitem[Neal(2011)]{neal2011mcmc}
Neal, R. M.
\newblock MCMC using Hamiltonian dynamics.
\newblock In \emph{Handbook of Markov Chain Monte Carlo}, pp. 47--95, Chapman and Hall/CRC, 2011.

\bibitem[Ho et al.(2020)]{ho2020denoising}
Ho, J., Jain, A., and Abbeel, P.
\newblock Denoising diffusion probabilistic models.
\newblock \emph{Advances in Neural Information Processing Systems}, 33:6840--6851, 2020.

\bibitem[Song et al.(2021)]{song2021scorebased}
Song, Y., Sohl-Dickstein, J., Kingma, D.~P., Kumar, A., Ermon, S., and Poole, B.
\newblock Score-based generative modeling through stochastic differential equations.
\newblock \emph{International Conference on Learning Representations}, 2021.

\bibitem[Sohl-Dickstein et al.(2015)]{sohldickstein2015deep}
Sohl-Dickstein, J., Weiss, E., Maheswaranathan, N., and Ganguli, S.
\newblock Deep unsupervised learning using nonequilibrium thermodynamics.
\newblock In \emph{Proceedings of the 32nd International Conference on Machine Learning (ICML)}, Vol. 37,  \emph{Proceedings of Machine Learning Research}, pp. 2256--2265, Lille, France, 2015.

\bibitem[Chen et al.(2018)]{chen2018neural}
Chen, R. T. Q., Rubanova, Y., Bettencourt, J., and Duvenaud, D. K.
\newblock Neural ordinary differential equations.
\newblock \emph{Advances in Neural Information Processing Systems}, 2018.

\bibitem[Albergo et al.(2025)]{albergo2025stochastic}
Albergo, M., Boffi, N. M., and Vanden-Eijnden, E.
\newblock Stochastic interpolants: A unifying framework for flows and diffusions.
\newblock \emph{Journal of Machine Learning Research}, 26:1--80, 2025.

\bibitem[Chi et al.(2025)]{chi2025diffusion}
Chi, C., Xu, Z., Feng, S., Cousineau, E., Du, Y., Burchfiel, B., Tedrake, R., and Song, S.
\newblock Diffusion policy: Visuomotor policy learning via action diffusion.
\newblock \emph{The International Journal of Robotics Research}, 44(10--11):1684--1704, 2025.

\bibitem[Feydy et al.(2019)]{feydy2019interpolating}
Feydy, J., S{\'e}journ{\'e}, T., Vialard, F.-X., Amari, S.-I., Trouv{\'e}, A., and Peyr{\'e}, G.
\newblock Interpolating between optimal transport and mmd using sinkhorn divergences.
\newblock In \emph{The 22nd International Conference on Artificial Intelligence and Statistics}, pp. 2681--2690, PMLR, 2019.

\bibitem[Miyato et al.(2018)]{miyato2018spectral}
Miyato, T., Kataoka, T., Koyama, M., and Yoshida, Y.
\newblock Spectral normalization for generative adversarial networks.
\newblock In \emph{ICLR}, 2018.

\bibitem[Gao et al.(2024)]{gao2024convergence}
Gao, Y., Huang, J., Jiao, Y., and Zheng, S.
\newblock Convergence of continuous normalizing flows for learning probability distributions.
\newblock \emph{arXiv:2404.00551}, 2024.

\bibitem[Lee et al.(2024)]{lee2024stein}
Lee, D., Lerch, C., Ramos, F., and Abraham, I.
\newblock Stein variational ergodic search.
\newblock In \emph{Robotics: Science and Systems (RSS)}, 2024.

\bibitem[Miller and Murphey(2013)]{miller2013trajectory}
Miller, L. M. and Murphey, T. D.
\newblock Trajectory optimization for continuous ergodic exploration.
\newblock In \emph{2013 American Control Conference}, pp. 4196--4201, 2013.
\newblock doi: 10.1109/ACC.2013.6580484.

\bibitem[Liu et al.(2023)]{liu2023flow}
Liu, X., Gong, C., and Liu, Q.
\newblock Flow straight and fast: Learning to generate and transfer data with rectified flow.
\newblock In \emph{ICLR}, 2023.

\bibitem[Virmaux and Scaman(2018)]{virmaux2018lipschitz}
Virmaux, A. and Scaman, K.
\newblock Lipschitz regularity of deep neural networks: analysis and efficient estimation.
\newblock \emph{Advances in Neural Information Processing Systems}, 31, 2018.

\bibitem[Papamakarios et al.(2021)]{papamakarios2021normalizing}
Papamakarios, G., Nalisnick, E., Rezende, D.~J., Mohamed, S., and Lakshminarayanan, B.
\newblock Normalizing flows for probabilistic modeling and inference.
\newblock \emph{Journal of Machine Learning Research}, 22(57):1--64, 2021.

\bibitem[Villani(2009)]{villani2009optimal}
Villani, C.
\newblock \emph{Optimal Transport: Old and New}.
\newblock Grundlehren der mathematischen Wissenschaften, vol.~338. Springer, 2009.

\bibitem[Peyr{\'e} and Cuturi(2019)]{peyre2019computational}
Peyr{\'e}, G. and Cuturi, M.
\newblock \emph{Computational optimal transport: With applications to data science}.
\newblock Now Foundations and Trends, 2019.

\bibitem[Cabreira et al.(2019)]{cabreira2019survey}
Cabreira, T.~M., Brisolara, L.~B., and Ferreira Jr., P.~R.
\newblock Survey on coverage path planning with unmanned aerial vehicles.
\newblock \emph{Drones}, 3(1):4, 2019.

\bibitem[Wu et al.(2018)]{wu2018joint}
Wu, Q., Zeng, Y., and Zhang, R.
\newblock Joint trajectory and communication design for multi-{UAV} enabled wireless networks.
\newblock \emph{IEEE Transactions on Wireless Communications}, 17(3):2109--2121, 2018.

\bibitem[Benamou and Brenier(2000)]{benamou2000computational}
Benamou, J.-D. and Brenier, Y.
\newblock A computational fluid mechanics solution to the {M}onge--{K}antorovich mass transfer problem.
\newblock \emph{Numerische Mathematik}, 84(3):375--393, 2000.

\bibitem[Bayerlein et al.(2018)]{bayerlein2018trajectory}
Bayerlein, H., De Kerret, P., and Gesbert, D.
\newblock Trajectory optimization for autonomous flying base station via reinforcement learning.
\newblock In \emph{2018 IEEE 19th International Workshop on Signal Processing Advances in Wireless Communications (SPAWC)}, pp. 1--5, IEEE, 2018.

\bibitem[Dudley(2002)]{dudley2002real}
Dudley, R.~M.
\newblock \emph{Real analysis and probability}.
\newblock Cambridge Studies in Advanced Mathematics, vol.~74. Cambridge University Press, 2002.

\bibitem[Evans(2022)]{evans2010pde}
Evans, L. C.
\newblock \emph{Partial differential equations}.
\newblock Vol. 19, American Mathematical Society, 2022.

\bibitem[Henrot(2006)]{henrot2006extremum}
Henrot, A.
\newblock \emph{Extremum problems for eigenvalues of elliptic operators}.
\newblock Springer, 2006.

\bibitem[Brenier(1991)]{brenier1991polar}
Brenier, Y.
\newblock Polar factorization and monotone rearrangement of vector-valued functions.
\newblock \emph{Communications on Pure and Applied Mathematics}, 44(4):375--417, 1991.

\bibitem[Cuturi(2013)]{cuturi2013sinkhorn}
Cuturi, M.
\newblock Sinkhorn distances: Lightspeed computation of optimal transport.
\newblock \emph{Advances in Neural Information Processing Systems}, 26:2292--2300, 2013.

\bibitem[Miller et al.(2016)]{miller2016ergodic}
Miller, L.~M., Silverman, Y., MacIver, M.~A., and Murphey, T.~D.
\newblock Ergodic exploration of distributed information.
\newblock \emph{IEEE Transactions on Robotics}, 32(1):36--52, 2016.

\bibitem[Mavrommati et al.(2018)]{mavrommati2018realtime}
Mavrommati, A., Tzorakoleftherakis, E., Abraham, I., and Murphey, T.~D.
\newblock Real-time area coverage and target localization using receding-horizon ergodic exploration.
\newblock \emph{IEEE Transactions on Robotics}, 34(1):62--80, 2018.

\bibitem[Seewald et al.(2024)]{seewald2024energy}
Seewald, A., Lerch, C.~J., Chanc{\'a}n, M., Dollar, A.~M., and Abraham, I.
\newblock Energy-aware ergodic search: Continuous exploration for multi-agent systems with battery constraints.
\newblock In \emph{IEEE International Conference on Robotics and Automation (ICRA)}, pp. 7048--7054, 2024.

\bibitem[Sun et al.(2025c)]{sun2025kernel}
Sun, M.~M., Gaggar, A., Trautman, P., and Murphey, T.~D.
\newblock Fast ergodic search with kernel functions.
\newblock \emph{IEEE Transactions on Robotics}, 41:1841--1860, 2025.

\bibitem[Hughes et al.(2025)]{hughes2025mmd}
Hughes, C., Warren, H., Lee, D., Ramos, F., and Abraham, I.
\newblock Ergodic trajectory optimization on generalized domains using maximum mean discrepancy.
\newblock In \emph{IEEE International Conference on Robotics and Automation (ICRA)}, 2025.

\bibitem[Hughes and Abraham(2026)]{hughes2026infinite}
Hughes, C. and Abraham, I.
\newblock Infinite-horizon ergodic control via kernel mean embeddings.
\newblock \emph{arXiv:2604.01023}, 2026.

\bibitem[Barlacchi et al.(2015)]{barlacchi2015multi}
Barlacchi, G., De Nadai, M., Larcher, R., Casella, A., Chitic, C., Torrisi, G., Antonelli, F., Vespignani, A., Pentland, A., and Lepri, B.
\newblock A multi-source dataset of urban life in the city of {M}ilan and the {P}rovince of {T}rentino.
\newblock \emph{Scientific Data}, 2:150055, 2015.

\bibitem[Haarnoja et al.(2018)]{haarnoja2018soft}
Haarnoja, T., Zhou, A., Abbeel, P., and Levine, S.
\newblock Soft actor-critic: Off-policy maximum entropy deep reinforcement learning with a stochastic actor.
\newblock In \emph{International Conference on Machine Learning (ICML)}, pp.\ 1861--1870, 2018.

\bibitem[Zeng et al.(2019b)]{zeng2019energy}
Zeng, Y., Xu, J., and Zhang, R.
\newblock Energy minimization for wireless communication with rotary-wing {UAV}.
\newblock \emph{IEEE Trans.\ Wireless Communications}, 18(4):2329--2345, 2019.

\end{thebibliography}
